\documentclass[10pt,letterpaper]{article}
\usepackage[utf8]{inputenc}
\usepackage{geometry}
\usepackage{times}

\usepackage{xcolor}
\usepackage{microtype,booktabs,graphicx,subcaption}
\usepackage{amsmath,amssymb,amsthm,amsfonts,mathtools}
\usepackage{physics}
\usepackage{paralist,array}
\usepackage{hyperref}
\usepackage{thmtools,thm-restate,amsthm}

\usepackage{verbatim}
\usepackage{dsfont}
\usepackage{fontenc}

\makeatletter
\theoremstyle{plain}
\newtheorem{thm}{\protect\theoremname}
\theoremstyle{plain}
\newtheorem{assumption}[thm]{\protect\assumptionname}
\theoremstyle{remark}
\newtheorem{rem}[thm]{\protect\remarkname}
\theoremstyle{plain}

\theoremstyle{plain}
\newtheorem{lem}[thm]{\protect\lemmaname}
\usepackage{color}
\usepackage{hyperref}

\usepackage[lined,boxed,ruled,norelsize,algo2e]{algorithm2e}
\@addtoreset{section}{part}

\usepackage{thumbpdf}

\definecolor{blue}{rgb}{0.2, 0.3, 0.7}

\usepackage{ragged2e}%

\usepackage{natbib}
\usepackage{graphicx}
\usepackage{subcaption}
\usepackage{times}
\usepackage{paralist}
\usepackage{mathtools}

\DeclarePairedDelimiter{\crl}{\{}{\}}
\DeclarePairedDelimiter{\prn}{(}{)}
\DeclarePairedDelimiter{\nrm}{\|}{\|}
\DeclarePairedDelimiter{\brk}{[}{]}

\newcommand{\remove}[1]{{}}

\newcommand{\ie}{\textit{i.e., }}
\newcommand{\eg}{\textit{e.g., }}

\newcommand{\sign}{\text{sign}}
\newcommand{\indicator}{\vec{1}}

\renewcommand{\vec}[1]{\boldsymbol{\mathbf{#1}}}
\renewcommand{\cal}[1]{\mathcal{#1}}

\newcommand{\x}{\vec{x}}
\newcommand{\w}{\vec{w}}
\renewcommand{\v}{\vec{v}}
\newcommand{\spf}{\vec{u}}
\newcommand{\z}{\vec{z}}
\newcommand{\K}{\vec{K}}
\newcommand{\vnu}{\vec{\nu}}
\newcommand{\X}{\vec{X}}
\newcommand{\bx}{\bar{\vec{x}}}
\newcommand{\noise}{\vec{\xi}}
\newcommand{\noisesmall}{\vec{\zeta}}
\newcommand{\noisesubscript}{{\xi}}
\newcommand{\noisesmallsubscript}{{\zeta}}
\newcommand{\noisepatchindex}{p^{\noisesubscript}}

\newcommand{\bR}{\mathbb{R}}

\newcommand{\cN}{\mathcal{N}}

\newcommand{\cP}{\mathcal{P}}
\newcommand{\cD}{\mathcal{D}}

\newcommand{\snote}[1]{{\color{red} Suriya: #1}}

\newtheorem{claim}{Claim}

\newtheorem{definition}{Definition}
\newtheorem{remark}{Remark}
\newtheorem{theorem}{Theorem}
\newtheorem{lemma}[theorem]{Lemma}

\newtheorem{corollary}[theorem]{Corollary}

\makeatother

\usepackage{babel}
\providecommand{\assumptionname}{Assumption}

\providecommand{\lemmaname}{Lemma}
\providecommand{\propositionname}{Proposition}
\providecommand{\theoremname}{Theorem}
\providecommand{\remarkname}{Remark}

\newcommand{\vtheta}{\vec{\theta}}
\newcommand{\vxi}{\vec{\xi}}

\newcommand{\pbpk}{\mathcal{P}_{bp,k}}%
\newcommand{\pbp}{\mathcal{P}_{bp}}%
\newcommand{\pbpi}{\mathcal{P}_{bp}^{(i)}}%
\newcommand{\pbpki}{\mathcal{P}_{bp,k}^{(i)}}%
\newcommand{\pbpkip}{\mathcal{P}_{bp,k'}^{(i)}}%
\newcommand{\pbpkj}{\mathcal{P}_{bp,k}^{(j)}}%
\newcommand{\api}{\alpha_{p,i}}%
\newcommand{\apj}{\alpha_{p,j}}%
\global\long\def\R{\mathbb{R}}%
\global\long\def\E{\mathbb{E}}%
\global\long\def\Ginit{\mathcal{G}_{\textup{init}}}%
\global\long\def\D{\mathcal{D}}%
\global\long\def\oT{\overline{T}}%
\global\long\def\oTaug{\overline{T}_{\text{aug}}}%
\global\long\def\T{\mathcal{T}}%
\global\long\def\Dtr{\mathcal{D}_{\text{train}}}%
\global\long\def\Dtraug{\mathcal{D}_{\text{train}}^{\text{(aug)}}}%
\global\long\def\N{\mathcal{N}}%
\global\long\def\K{\mathcal{K}}%
\global\long\def\tO{\widetilde{O}}%
\global\long\def\tom{\widetilde{\Omega}}%
\global\long\def\tthe{\widetilde{\Theta}}%
\global\long\def\poly{\textup{poly}}%
\global\long\def\sign{\textup{sign}}%
 
\begin{document}

\title{Data Augmentation as Feature Manipulation}
\author{Ruoqi Shen \thanks{University of Washington, {\tt shenr3@cs.washington.edu}. Part of this work was done as a intern at Microsoft Research.} \and S\'ebastien Bubeck \thanks{Microsoft Research, {\tt sebubeck@microsoft.com}}\and Suriya Gunasekar \thanks{Microsoft Research, {\tt suriyag@microsoft.com}}}
\date{}
\maketitle

\begin{abstract}
		Data augmentation is a cornerstone of the machine learning pipeline, yet its theoretical underpinnings remain unclear. Is it merely a way to artificially augment the data set size? Or is it about encouraging the model to satisfy certain invariance? In this work we consider another angle, and we study the effect of data augmentation on the dynamic of the learning process. We find that data augmentation can alter the relative importance of various features, effectively making certain informative but hard to learn features more likely to be captured in the learning process. Importantly, we show that this effect is more pronounced for non-linear models, such as neural networks. Our main contribution is a detailed analysis of data augmentation on the learning dynamic for a two layer convolutional neural network in the recently proposed multi-view data model by \citet{allen2020towards}. We complement this analysis with further experimental evidence that data augmentation can be viewed as feature manipulation.
\end{abstract}

\section{Introduction}
\label{sec:introduction}
Data augmentation is a powerful technique for inexpensively increasing the size and diversity of training data. Empirically, even minimal data augmentation can substantially increase the performance of neural networks. It is commonly argued that data augmentation is useful to impose domain specific symmetries on the model, which would be difficult to enforce directly in the architecture \citep{simard2000transformation,simard2003best,chapelle2001vicinal,yaeger1996effective,shorten2019survey}. For example, semantics of a natural image is invariant under translation and scaling, so it is reasonable to augment an image data set with translated and scaled variations of its inputs. Simple augmentation with random crop up to $4$ pixels can lead to gains in the range 5-10\% \citep{ciregan2012multi,krizhevsky2017imagenet}. Another explanation often proposed for the role of data augmentation is merely that it increases the sample size.
As an alternative to symmetry inducing or sample size increase, we consider in this work the possibility that data augmentation should in fact be viewed as a more subtle {\em feature manipulation} mechanism on the data.
\newline

Consider, for illustration, an image data set with the task of learning to detect whether there is a cow in the image. A simplified view would be that there are \textit{true cow features} that generate the cow images, and we hope to learn those \textit{true features}. At the same time, because most images of cows contain grass, it would not be surprising if a neural network would learn to detect the spurious \textit{grass feature} for the task, and perhaps simply overfit the rare images such as desert cows that are not explained by the \textit{grass feature} (and similarly overfit the perhaps few images with grass and no cows). Now consider a simple data augmentation technique such as Gaussian smoothing (let us assume black and white images or else use additional color space augmentations). The \textit{grass feature}, sans color, is essentially a high frequency texture information, so we can expect the smoothing operation to make this feature  significantly diminished. In this example, the {\em feature manipulation} that data augmentation performs is effectively to render the spurious feature harder to detect, or more precisely to make it harder to learn, which in turn boosts the \textit{true cow features} to become the dominant features.
\newline

Continuing the illustration above, let us explore further the idea of data augmentation as {\em feature manipulation}. First note that the \textit{true cow feature} need not be one ``single well-defined object", but rather we may have many different true cow features. For example, \textit{true cow features} could be different for left-facing and right-facing cows. 
An imbalance in the training data with respect to those different features could make the rarer features hard to learn compared to the more common features, similarly to how the spurious grass feature was occluding the true cow features. In the example above, it could happen that in most images in the training data, the cows are facing right, which in turn could mean that the neural network will learn a cow feature {\em with an orientation} (right-facing), and then simply memorize/overfit the cows facing left. Yet another commonly used data augmentation technique such as random horizontal flip would solve this by balancing the occurrence of cow features with right-orientation and those with left-orientation, hopefully leading to a neural network dynamic that would discover {\em both} of those types of cow features. 
Note that one might be tempted to interpret this as inducing a mirror symmetry invariance in the model, but we emphasize that the effect is more subtle: the learned invariance is {\em only} for the relevant features, rather than being an invariance for {\em all} images (\eg on non-cow images one might not be invariant to the orientation).
\newline

More generally, to understand feature learning with and without data augmentation in gradient descent trained neural networks, we can think of three types of features of interest: (a) The ``easy to learn and good" features, which are accurate features for the learning problem and are {easy to learn} in the sense that they have large relative contribution in the gradient descent updates of the network. (b) The ``hard to learn and good" features, which are more nuanced to detect but are essential to fit the harder samples in the population distribution (\eg examples with rare object orientations).
These are features that despite being accurate have small relative contribution in the gradient descent updates (perhaps due to lack of sufficient representation in the training data), which in turn makes them hard to learn. (c) Finally, there are the ``easy to learn and bad" features, which while inaccurate, nevertheless interfere with the learning process as they have a large contribution in the gradient updates. Such features often correspond to spurious correlations or dominating noise patterns (\eg the \textit{grass feature}) which arise due to limitations in training data size or data collection mechanisms.\footnote{We do not mention ``hard to learn and inaccurate" features as they are conceptually irrelevant for  the training dynamics or accuracy of the  model.}
In this paper, we study data augmentation as a technique for manipulating the ``easiness" and ``hardness" of features by essentially changing their relative contributions in the gradient updates for the neural network.
\newline

We believe that this view of data augmentation as a {\em feature manipulation} mechanism is more insightful (and closer to the truth) than the  complementary and more straightforward views of ``symmetry inducing" or ``it's just more data". For one, data augmentation with specific symmetries do not necessarily lead to models that are respectively invariant. For example, \citet{azulay2019deep} show that even models trained with extensive translation and scale augmentation can be sensitive to single pixel changes in translation and scaling on inputs far from the training distribution, suggesting the inductive bias from data augmentation is more subtle. Further, this view could form a basis for studying more recent data augmentation techniques like MixUp \citep{zhang2017mixup}, CutOut \citep{devries2017improved}, and variants, which in spite of being widely successful in image tasks do not fit the conventional narrative of data augmentation.

\paragraph{Contributions} Given the diversity of data augmentation techniques (\eg see \citet{shorten2019survey,feng2021survey} for a survey), it is a formidable challenge to understand and analyze the corresponding feature manipulation for each case, and this task is  beyond the scope of the present paper. Our more modest objective is to start this program by studying a simple mathematical model where data augmentation can be provably shown to perform feature manipulation along the lines described in the illustration above. Specifically, we consider a variant of the multi-view data setting introduced in the pioneering work of \citet{allen2020towards} on ensemble learning. 
In our data model, each data point is viewed as a set of patches, with each patch being represented by a high-dimensional vector in $\R^d$. Moreover there is a set of $K$ ``true/good" features $\v_1, \v_2, \hdots, \v_K \in \R^d$. For any data point, each patch is then some combination of noise and features. Specifically at least one patch contains a ``good" feature whose orientation indicates the  label, \ie for some $k \in [K]$ this patch is $y \v_k$ where $y \in \{-1,1\}$ is the binary class label to be predicted. In this case we say that the data point is of the $k^{th}$ type. The other patches contain different forms of noise. 
If the training data contains sufficiently many type-$k$ data points, then the corresponding feature $\v_k$ is ``easy to learn and good", while the features corresponding to rare types are ``hard to learn and good". To model the ``easy to learn and bad" features we assume that one patch per datapoint receives a large (Gaussian) noise, which we call the dominant  noise. 
See Section \ref{sec:prelim} for exact details of the model. 
Given such training data we show the following for a two layer patch-wise convolutional network (see \eqref{eq:model}) trained using gradient descent (there is a number of caveats, see below for a list):
\begin{enumerate}
	\item When one or more features are sufficiently rare, the network will only learn the frequent ``easy to learn and good" features, and will overfit  the remaining data using the ``easy to learn and bad" noise component.
	\item On the other hand, with any data augmentation technique that can permute or balance the features, the network will learn all $K$ features, and thus achieve better test loss (and, importantly learn a better representation of this data\footnote{As a consequence of learning all the $K$ features, the learned model will not only be more accurate on the data distribution of training samples, but will also be robust to distribution shifts that alter the proportion of data of the $K$ feature types.}). We show that this happens because the representation of the ``hard to learn and good" features in the gradient updates will be boosted, and simultaneously the relative contribution of the dominant noise or the ``easy to learn and bad" features will be diminished.
	\item We show that this phenomenon is more pronounced for gradient descent dynamics in non-linear models in the following sense: we prove that even at high signal-to-noise ratio (SNR) the non-linear models might memorize through the noise components, while gradient descent on linear models overfit to noise only at much lower SNR. 
	This shows that data augmentation is useful in a wider range of cases for non-linear models than for linear models. 
	
	Moreover, our non-linear model can learn the distribution even in the presence of feature noise (in the form of $- \alpha y \v_{k'}$ for some small $\alpha>0$, which points to wrong class). On the other hand, a linear model cannot have low test error with such feature noise, thus showing a further separation between linear and non-linear models. 

\end{enumerate}

Some of the caveats to our theoretical results include the following points (none seem essential, but for some of them going beyond would require significant technical work):
\begin{itemize}
	\item Neural network architecture: we study two layer neural network with a special activation function (the latter can be viewed as a smoothed ReLU with fixed bias). We also assume poly-logarithmic (in $d$) width.
    \item Training: we study gradient descent rather than stochastic gradient descent, and furthermore we assume a specific training time (the same one with and without data augmentation).
    \item Data model: the distribution can be generalized in many ways, including having data points with mixed types (\eg ``multi-view" as in \cite{allen2020towards}), heterogeneous noise components, or even correlated noise components (see below for more on this). We also assume a very high dimensional regime $d \gg n^2$ (where $n$ is the training set size), although we believe our results should hold for $d \gg n$.
\end{itemize}

Even though our theoretical results are in  a limited setting, the feature manipulation effect of data augmentation is conceptually broader. We complement our analysis with experiments on CIFAR-10 and synthetic datasets, where we study data augmentation in more generality. We  circle back to our motivating problem with spurious features (\textit{{\`a}la the cow grass features story}) in a classification task. Our experiments show that simply shifting the spurious feature position randomly up to 2 pixels in each epoch, can significantly improve the test performance by making the spurious feature hard to learn. This happens even when we do not change any non-spurious pixels/features (and hence control learning additional image priors). We further formulate experiments to evaluate the value of a single data augmented image compared to an fully independent sample, and see that on CIFAR10 dataset that once $~50\%$ independent samples are available, a data augmented sample is almost as effective as an independent sample for the learning task. Finally, we show on synthetic dataset that the problem arising from imbalance in views (as studied in our main result) also holds for deeper convolutional architectures, even when the views are merely translations of each other. 

\paragraph{Related Work}
Starting with \citep{bishop1995training} there is a long line of work casting data augmentation as an effective regularization technique, see \citep{dao2019kernel, rajput2019does, wu2020generalization, yang2022sample} for recent developments in that direction. Other theoretical analyses have studied and quantified the gains of data augmentation from an invariance perspective \citep{chen2020group,mei2021learning}. The viewpoint we take here, based on studying directly the effect of augmentation on the learning dynamic, is strongly influenced by the work of Zeyuan Allen-Zhu and Yuanzhi Li in the last few years. For example in \citet{allen2020feature} they develop this perspective for {\em adversarial training} (which in some ways can be thought as a form of data augmentation, where each data point is augmented to its adversarial version). There they show that adversarial training leads to a certain form of {\em feature purification}, which in essence means that the filters learned by a convolutional neural network become closer to some ``ground truth" features. In \citep{allen2020towards} they introduce the multi-view model that we study here, and they used it to study (among other things) ensemble learning. In a nutshell, in their version of the model each data point has several views that can be used for classification, and the idea is that each model might learn only one of those views, hence there is benefit to ensembling in that it will allow to uncover all the features, just like here we suggest that data augmentation is a way to uncover all the features. Other notable works which share the philosophy of studying the dynamic of learning (although focused on linear models) include \citep{hanin2021data} which investigates the impact of data augmentation on optimization, and \citep{wu2020generalization} which considers the overparametrized setting and show that data augmentation can improve generalization in this case.

\paragraph{Notation} We use tilde notation $\tO$, $\tthe$, $\tom$ to hide $\log$ factors in standard asymptotic notation. 
For an integer $K$, $[K]=\crl{1,2,\ldots,K}$. We interchangeably use $\vec{a}\cdot \vec{b}$, $\langle \vec{a},\vec{b}\rangle $, or $\vec{a}^\top\vec{b}$ for standard inner product between two vectors.

\section{A mathematical model for understanding feature manipulation} \label{sec:prelim}

Our data model defined below is a variation of the multi-view data distribution in \cite{allen2020towards} for a binary classification task.  We represent the inputs $\x$ as a collection of $P$ non-overlapping patches $\x=(\x_1,\x_2,\ldots, \x_P)\in\R^{d\times P}$, where each patch is a $d$ dimensional vector.
The task is associated with $K$ unknown ``good" features denoted as $\v_1,\v_2,\ldots, \v_K\in \R^d$, such that for labels  $y\in\crl{-1,1}$, their orientation as $\{y\v_{k}\}_{k\in[K]}$ constitutes the $K$ views or sub-types of the class $y$.\footnote{For $M$-class classification, our analysis can be adapted by using separate set of features $\crl{\v_{k,m}}_{k}$ for each class $m\in[M]$, rather than $\crl{\pm\v_k}_k$. For $M=2$, under our learning algorithm, using $(\v_{k,-1},\v_{k,1})$ as features for $y={-1,1}$ is equivalent to using $-\v_{k},\v_k$  with $\v_k=\v_{k,1}-\v_{k,-1}$.}  Each input $\x_p$ patches either contain one of the ``good" feature $\{y\v_k\}$ or a ``bad" feature in the form of random and/or feature noise. 
Formally, our distribution is defined below.

\begin{definition} \label{def:data} $\cD$ is parametrized by $\big(\vec{\rho},\sigma_{\noisesubscript},\sigma_{\noisesmallsubscript},\alpha\big)$, where  $\vec{\rho}=(\rho_1,\rho_2,\ldots,\rho_K)$ is a discrete distribution over the  features $\crl{\v_k}_{k\in[K]}$, and $\sigma_{\noisesubscript}$,$\sigma_{\noisesmallsubscript}$, and $\alpha$ are noise parameters. Without loss of generality, let $\rho_1\ge \ldots \ge \rho_K$.
	A sample $(\x,y)\sim\cD$ is generated as follows: 
	\begin{enumerate}[\quad (a)]
		\item Sample $y\in\crl{1,-1}$ uniformly.
		\item Given $y$, the input $\x=(\x_1,\x_2,\ldots, \x_P)\in\R^{d\times P}$ is sampled as below:

	Feature patch: Choose the main feature patch $p^*\in[P]$ arbitrarily and set $	\x_{p^*}=y\v_{k^*} \text{, where } k^*\sim\vec{\rho}.$.
 Dominant noise: Choose a dominant noise patch $\noisepatchindex\neq p^*$ and generate $\x_{\noisepatchindex}=\noise \text{,\quad where }\noise\overset{\text{i.i.d}}{\sim}\mathcal{N}\big(0,\frac{\sigma_{\noisesubscript}^2}{d} I_d\big). $

Background: {For the remaining background patches{\footnotemark} $p\in [P]\setminus\crl{p^*,\noisepatchindex}$, select $0\le\alpha_{p}\le\alpha$ and set} $
\x_p=-\alpha_{p}y\v_{k_p}+\noisesmall_p\text{, where }k_p\sim\vec{\rho}, \noisesmall_p\sim\cN(0,\sigma_{\noisesmallsubscript}^2I_d).  $ 
	\end{enumerate}

\end{definition}
\footnotetext{In our definition, the dominant noise $\noise$ and the main feature $\v_{k^*}$ appear in exactly one patch. But our results also hold (by virtue of parameter sharing in \eqref{eq:model}) when for any disjoint non-empty subsets $\cP_f,\cP_n\subset[P]$, we set $\forall\,{p\in\cP_f},\;\x_p=y\v_{k^*}$ and $\forall\,{p\in\cP_n},\;\x_p=\noise_p\underset{\text{i.i.d}}{\sim}\cN(0,\sigma^2_{\noisesubscript}I_d/d)$.}

\begin{assumption}\label{ass:orthonormal features}	We assume the features $\crl{\v_k}_{k\in[K]}$ are orthonormal, \ie $\forall_{k,k'\in[K]},\;\v_k\cdot \v_{k'}=\indicator_{k=k'}$. 
\end{assumption}

The training dataset  consists of $n$ i.i.d., samples from $\cD$, $\Dtr=\crl{(\x^{(i)},y^{(i)}):i\in[n]}\sim \cD^{\otimes n}.$ We are interested in the high dimensional regime where $n\ll d$. $n$, $P$ and $K$ can grow with $d$.
Note that, in Definition~\ref{def:data} $k^*$, $p^*$, $\noisepatchindex$, $\noise$, and $(\alpha_p,k_p,\noisesmall_p)_{p\notin\crl{p^*,p^{\noisesubscript}}}$ all depend on $\x$, but we have dropped this dependence in the notation to avoid clutter. In our analysis, for $i=1,2,\ldots,n$, we use  $k^*_i$, $p^*_i$, $\noisepatchindex_i$, $\noise^{(i)}$, and $(\alpha_{p,i},k_{p,i},\noisesmall_{p,i})_{p\notin\crl{p^*_i,p^{\noisesubscript}}_i}$ to denote the corresponding quantities for the  sample $(\x^{(i)},y^{(i)})$ in the training dataset.

\remove{
	\snote{May be we should move this discussion to appendix}
	\begin{remark}
	or $k\in[K]$, let $\mathcal{I}_k=\crl{i\in[n]:k_{i}^*=k}$ denote the indices of data points with view $y\v_k$ as the main feature and $\hat{\rho}_{k}=\frac{1}{n}|\mathcal{I}_k|$ denote its empirical fraction in the training data. Recall from Definition~\ref{def:data} that $k_i$ are sampled independently with $\text{Pr}(k_i^*=k)=\rho_k$.  Thus, with with high probability, $\rho_{k}$ and $\hat{\rho}_{k}$ differ at most by $\sqrt{\frac{\log(n)}{n}}$. In the rest of the paper,  for simplicity we assume $\rho_{k}=\hat{\rho}_{k}$.

	Similarly, let $\hat{\rho}_{k}^{\text{(noise)}}$ be the proportion of feature
	noise $-y\v_{k}$ in dataset $\Dtr$, i.e., $\hat{\rho}_{k}^{\text{(noise)}}=\frac{1}{n(P-2)}|\crl{i\in[n],p\in[P]\setminus\{p^*_i,p^\noisesubscript_i\}]|k_{p,i}=k}|$

	Again from standard concentration, we have $\rho_{k}$ and $\hat{\rho}_{k}^{\text{(noise)}}$ differ by negligible quantity with high probability, thus we also assume $\rho_{k}=\overline{\rho}_{k}^{\Ftext{(noise)}}$.

	\end{remark}
}

\paragraph{Data augmentation}

Let $\Dtraug$ denote the augmented dataset obtained by transforming the i.i.d. training dataset $\Dtr$.
Our model for data augmentation is such that $\Dtraug$ has  equal number of samples with main feature $y\v_{k}$ for each $k\in[K]$.
Concretely,  consider linear transformations $\T_{1},\ldots\T_{K-1}$, such that for all $k$, $\T_k:\R^d\to\R^d$ and satisfies
\begin{equation}
	\forall\,{k'\in[K]},\; \T_{k}(\v_{k'})=\v_{((k'+k-1)\!\!\!\!\mod K)+1)}.
	\label{eq:permutation}
\end{equation}

Such transformations are well defined for $K\le d$, and in essence permute the  feature vectors $\v_k$ on patches with true feature or feature noise. At the same time, the Gaussian noise patches before and after transformation are no longer i.i.d.
We slightly abuse notation and define $\T_k(\x)$ on $\x\in\R^{d\times P}$ as  $\T_k(\x)=(\T_{k}(\x_1),\T_{k}(\x_2),\ldots,\T_{k}(\x_p))\in\R^{d\times P}$, as well as  $\T_k(\Dtr)$ on the training dataset as $\T_k(\Dtr)=\{(\T_{k}(\x^{(i)}),y^{(i)}):i\in[n]\}$. \newline

\noindent Our augmented dataset $\Dtraug$ consists $\Dtr$ along with the $K-1$ transformations of
of $\Dtr$ as defined below:
\begin{equation}\label{eq:augment}
\Dtraug=\Dtr\,\cup\,\T_{1}(\Dtr)\ldots\cup\,\T_{K-1}(\Dtr).
\end{equation}
Note that in $\Dtraug$ all the views are equally represented, \ie for each $k\in[K]$, we will have exactly $n$ samples from the feature $y\v_k$, and further $\Dtraug$ has more samples compared to $\Dtr$ with $|\Dtraug|=nK$, but they are no longer i.i.d. \newline

Since the features $\{\v_k\}_{k}$ are orthonormal (Assumption~\ref{ass:orthonormal features}) and all the non-feature noise are spherically symmetric, without loss of generality, we can assume that $\crl{\v_k}_{k\in[K]}$ are simply the first $K$ standard basis vectors in $\R^d$, \ie  $\v_k=\vec{e}_k$. 
In this case,  we can choose $\T_k$ for $k\in[K-1]$ as a permutation of coordinates satisfying \eqref{eq:permutation} on the first $K$ coordinate. If we further assume that the the permutations $\T_k$ do not have any fixed points, \ie $\forall\, i\in[d]$, $\T_k(\z)[i]\neq \z[i]$, then at initialization and updates of gradient descent, the augmented samples in $\Dtraug$ satisfy the same properties as i.i.d. samples in $\Dtr$ (upto constants and log factors). In this rest of the proof, we thus assume that $\T_k$ are permutations of coordinates without any fixed points in the orthogonal basis extended from  $\{\v_k\}_k$, and satisfies \eqref{eq:permutation}.

\paragraph{Role of different noise components} Our main result shows that when the dominant noise parameter $\sigma_{\noisesubscript}$ is sufficiently large, a neural network can overfit to this noise rather than learn all the views. However, with the right data augmentation, we can show that all the views can be accurately learned using a non-linear network. Furthermore, in the presence of feature noise $\{-\alpha_py\v_{k_p}\}$ (pointing to wrong class), linear models are unable to fit our data distribution, thus establishing a gap from linear models. \newline

We choose the noise parameters $\sigma_{\noisesubscript},\sigma_{\noisesmallsubscript},\alpha$ such that the \textit{dominant noise} $\noise$ and the true features $\{y\v_{k^*}\}$ have the main contribution to the learning dynamic compared to the feature noise (\ie $-\alpha_py\v_{k_p}$) or the minor noise (\ie $\noisesmall_p$). Thus, our results do not necessarily require noise in the background patches beyond establishing gap with linear models.  Since the minor noise $\sigma_{\noisesmallsubscript}$ does not provide any additional insight, we assume $\sigma_{\noisesmallsubscript}=0$. Our analysis can handle small $\sigma_{\noisesmallsubscript}$ with more tedious bookkeeping.

\subsection{Learning algorithm}\label{sec:learning_algorithm}
We use the following patch-wise convolutional network architecture with $C$ channels: let  $\w=\{\w_1,\w_2,\ldots \w_C\}\in\R^{d\times C}$ denote the learnable parameters of the model,
\begin{equation}
	F(\w,\x) = \sum_{c \in [C]} \sum_{p \in [P]} \psi(\w_c \cdot \x_p) \,,
	\label{eq:model}
\end{equation}
where $\psi$ is a non-linear activation function defined below:

\begin{minipage}[c]{\linewidth}
\centering
	{\begin{minipage}[t][][b]{.35\linewidth}
		\smallskip\noindent
		\begin{flalign*}
			\psi(z)&=\left\{\!\!\begin{array}{ll}\sign(z)\cdot\frac{1}{q}|z|^q \!\!\!&\text{if }|z|\le 1\\
			z-\frac{q-1}{q} \!\!\!&\text{if }z\ge1\\
			z+\frac{q-1}{q} \!\!\!&\text{if }z\le1\\
			\end{array}\right.&
		\end{flalign*}
	\end{minipage}}\quad
	{\begin{minipage}[t][][b]{0.2\linewidth}
			\includegraphics[width=\linewidth]{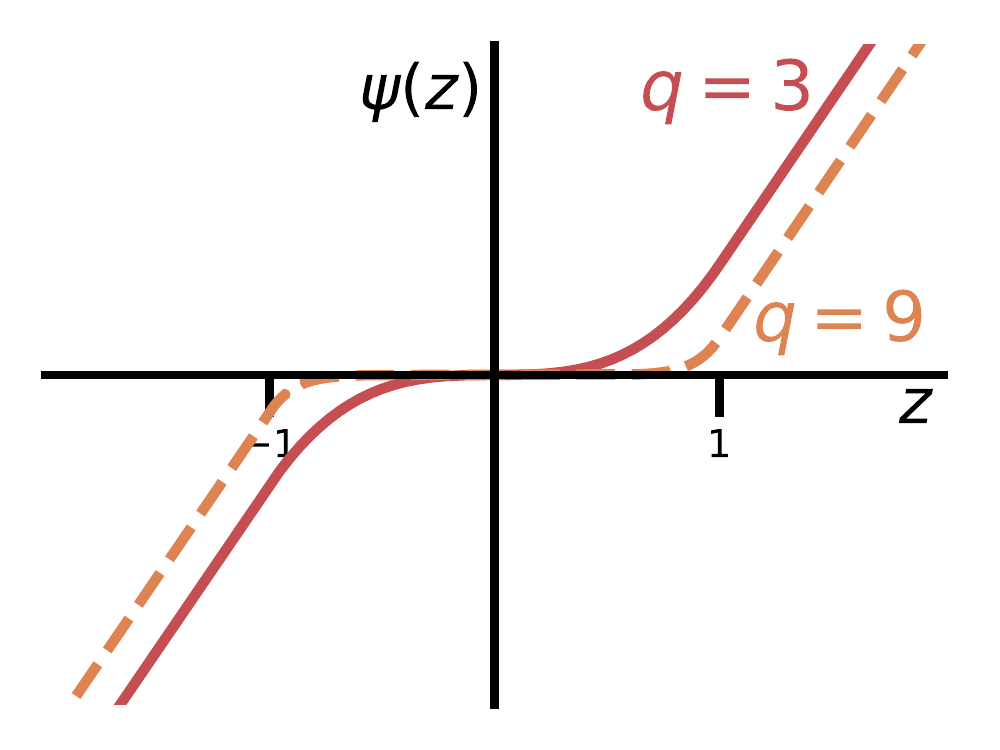}
	\end{minipage}}
\end{minipage}\newline

Our activation is a smoothed version of symmetrized ReLU with a fixed bias $\phi(z)=\text{ReLU}(z+1)-\text{ReLU}(-z-1)$. In fact, as $q\to \infty$, $\psi\to \phi$.
Note that since we do not train the second layer weights, we choose an odd-function as activation to ensure that the outputs can be negative.\newline

\noindent Consider the following logistic loss over the training dataset $\Dtr=\left\{ (\x^{(i)},y^{(i)}),i\in[n]\right\}$: 
	$L(\w)=\frac{1}{n}\sum_{i=1}^n\ell(y^{(i)}F(\w,{\x}^{(i)})),${ where } $\ell(z)=\log(1+\exp(-z)).$%
We learn the model using gradient descent on the above loss with step size $\eta$, \ie for $c\in[C]$, the weights $\w_c$ at  time step $t$ are given by
$\w_{c}(t)=\w_{c}(t-1)- \frac{\eta}{n} \sum_{i=1}^{n}y^{(i)}\ell'(y^{(i)}F(\w(t),\x^{(i)}))\nabla F(\w(t),\x^{(i)}).$\newline

\noindent The following lemma summarizes the conditions at Gaussian initialization $\w(0)=\{\w_c(0)\sim\cN(0,\sigma_0^2I_d):c\in[C]\}$.
\begin{restatable}{lemma}{ginitlemma}[$\Ginit$-conditions]\label{lem:init}
	Consider $n$ i.i.d. samples $\Dtr=\crl{(\x^{(i)},y^{(i)}):i\in[n]}$ from the distribution in Definition~\ref{def:data}. Let the parameters $\w$ of  the network in \eqref{eq:model} be initialized as $\w_c(0)\sim\cN(0,\sigma_0^2I_d)$ $\forall\,c\in[C]$.
	If the number of channels is $C=\Omega(\log{d})$, then with  probability greater than $1-O(\frac{n^2KC}{\poly(d)})$, the following conditions hold :
	\begin{enumerate}
		\item \textit{Feature-vs-parameter:} $\forall \, {k\in[K]}$,  $\max\limits_{c\in[C]} \w_{c}{(0)}\cdot\v_k \geq\Omega(\sigma_{0})$, and $\max\limits_{c\in[C]}\,|\w_{c}{(0)}\cdot\v_k|\leq\tO\left(\sigma_{0}\right).$
		\item \textit{Noise-vs-parameter:}$\forall \, {i\in[n]}$,  $\max\limits_{c\in[C]} \w_{c}{(0)}\cdot y^{(i)}\noise^{(i)} \geq\tom\left(\sigma_{0}\sigma_{\noisesubscript}\right)$, and $\max\limits_{c\in[C]}|\w_{c}{(0)}\cdot \noise^{(i)}|\leq\tO\left(\sigma_{0}\sigma_{\noisesubscript}\right).$
		\item \textit{Noise-vs-noise:} $\forall \, {i\in[n]}$, $\noise^{(i)}\cdot \noise^{(i)}=\Theta(\sigma^2_{\noisesubscript})$ and $\forall\,{i,j\in[n],i\neq j}$, $|\noise^{(i)}\cdot\noise^{(j)}|\leq\tO({\sigma_{\noisesubscript}^{2}}/{\sqrt{d}})$.
		\item \textit{Feature-vs-noise:} $\forall \, {i\in[n],k\in[K]}$, $|\noise^{(i)}\cdot\v_k|\leq\tO(\sigma_{\noisesubscript}/{\sqrt{d}})$.
		\item \textit{Parameter norm:} $\forall \, {c\in[C]}$, $\norm*{\w_{c}{(0)}} =\Theta(\sigma_{0}\sqrt{d})$.
	\end{enumerate}
\end{restatable}
\noindent The above lemma proved in the Appendix~\ref{app:initialization} follows from standard Gaussian concentration bounds. Further, we can show that $\Ginit$ also hold for the augmented dataset $\Dtraug$ even though the samples in $\Dtraug$ are not i.i.d.
\begin{restatable}{sublemma}{ginitsublemma}\label{lem:init-sub} 
	$\Ginit$ in Lemma~\ref{lem:init}  also holds for $\Dtraug$ defined in \eqref{eq:augment} with $n$ replaced by $nK$.
\end{restatable}

\subsection{Clarification on capacity in this model}
 We now informally discuss the size of our model class in the context of our data distribution. Consider the convolutional model \eqref{eq:model} with $C=1$ and say $\alpha=0$ for sake of simplicity in the data distribution. Using $\w_1=\w^{\texttt{gen}}=\gamma \sum_{k=1}^K \v_k$ for some large $\gamma>0$ will yield excellent training and test error. This is a model that would ``generalize". On the other hand for a fixed training set $\crl{(\x^{(i)}, y^{(i)})}_{i \in [n]}$,  one could also obtain almost perfect training error by using $\w_1=\w^{\texttt{overfit}}= \gamma \sum_{i=1}^n y^{(i)} \noise^{(i)}$, whenever $\sigma_\noisesubscript$ and $d\gg n$ (noise components $\crl{\noise^{(i)}}_{i \in [n]}$ are near orthonormal). Indeed with high probability, $\forall_{i\in[n]},\;y^{(i)} f(\w^{\texttt{overfit}},\x^{(i)})=y^{(i)}\sum_{p \in [P]} \psi(\w^{\texttt{overfit}} \cdot \x^{(i)}_p)$ is exactly
\begin{equation*}
\psi\prn[\big]{\gamma\sigma_{\noisesubscript}^2(1+\tO(\sqrt{n/d}))} + \psi\prn[\big]{\gamma \sigma_{\noisesubscript} O(\sqrt{n/d})}= \gamma\sigma_\noisesubscript^{2} (1+o(1)) \,.\newline
\end{equation*}

In other words the model with $\w^{\texttt{overfit}}$ will almost perfectly memorize the training set, while on the other hand it is clear that it will completely fail to generalize. This shows that the model class is large enough so that any classical measure of complexity, \eg Rademacher complexity, would fail to predict generalization (even data-dependent Rademacher complexity where the $\x^{(i)}$ follow our data distribution). In fact, our arguments below show that gradient descent could lead to a model of the form $\w^{\texttt{overfit}}$ in a Rademacher complexity setting (\ie with random label $y^{(i)}$ independent of the inputs $\x^{(i)}$). Thus, even restricting to models reached by gradient descent would still yield a high Rademacher complexity. This phenomenon has also been empirically observed in practical neural networks \cite{neyshabur2015search,zhang2021understanding}, and shown theoretically in simpler models in \cite{nagarajan2019uniform}. Thus, we are in a case where not only do we need to leverage the fact that we are using gradient descent to prove generalization, but we also  need to use the specific target function  (\ie the relation between $y$ and $\x$) that we are working with.

\subsection{Our argument in a nutshell} \label{sec:nutshell}
At a high level we show that there is a cutoff point in the features, denote it $K_{\texttt{cut}}$, such that running gradient descent on the above architecture and data distribution will lead to a model which is essentially a mixture of parts of $\w^{\texttt{gen}}$ and parts of $\w^{\texttt{overfit}}$ described above. Roughly it will be:
\begin{equation} \label{eq:mixture}
\sum_{k \leq K_{\texttt{cut}}} \v_k + \sum_{i : k^*_i > K_{\texttt{cut}}} y^{(i)} \noise^{(i)} \,.
\end{equation}
In words, the frequent enough features will be learned, and the data points that correspond to infrequent enough features will be memorized through their noise component. Quite naturally, this cutoff point will be decreasing with the magnitude of the noise $\sigma_{\noisesubscript}$, \ie the bigger the noise the fewer features will be learned. %
While this argument also holds for gradient descent dynamics on linear models, the cutoff point $K_{\texttt{cut}}$ of linear models can be higher than that of the non-linear models, which shows that non-linear models can memorize through the noise component at a higher SNR (see Section~\ref{sec:whatorder} for the exact cutoff point). \newline 

Where data augmentation will come in is that it can effectively change the frequency of the features, and in the extreme case we consider to make them all equal,\ie all with frequency $1/K$. We then show that there exists a setting of the parameters such that frequency $1/K$ is learned at noise magnitude $\sigma_\xi$, so that with data augmentation all the features are learned.

\subsection{Linear and tensor models}
Before diving into the dynamics of gradient descent for our neural network architecture and data distribution, let us expand briefly on linear models. In Appendix~\ref{app:linear} we study the max-$\ell_2$ margin linear classifier for our data, but for sake of simplicity we consider here an even more basic predictor that is specifically tailored to our data distribution:
$
\bar{\vtheta} := \frac{1}{n} \sum_{i=1}^n \sum_{p \in [P]} y^{(i)} \x_p^{(i)} \,.
$
Note that $\bar{\vtheta}$ is a linear function on $\R^d$, and we naturally extend it to the domain $\R^{d \times P}$ of our data points (with slight overloading of notation) as $\bar{\vtheta}(\x) = \sum_{p \in [P]} \bar{\vtheta} \cdot \x_p$. Compared to a gradient descent learned model, it is not clear whether this predictor is meaningful beyond our specific data distribution, and we emphasize that we study it merely as a shortest path to get quantitative estimates for the discussion in Section \ref{sec:nutshell} (\eg for the cutoff point and for the SNR of interest). In fact the (gradient descent learnable) max margin linear classifier has even better properties than the estimator $\bar{\w}$, see the Appendix~\ref{app:linear} for more details.

\paragraph{Derivation of a cutoff point.}
It is easy to check that with our data distribution we have $\bar{\vtheta} = \bar{\vtheta}_S + \bar{\vtheta}_N$ where $\bar{\vtheta}_S = \sum_{k=1}^K \rho_k \v_k$ (say the fraction of examples of type $k$ is exactly $\rho_k$) and $\bar{\vtheta}_N = \frac{1}{n} \sum_{i=1}^n y^{(i)} \noise^{(i)}$ (assume $\alpha =0$ for this discussion). In particular for $\x$ sampled from our distribution, we have with high probability $|\bar{\vtheta}_N(\x)|  \simeq \frac{\sigma_{\xi}^2}{\sqrt{n d}}$ and $\bar{\vtheta}_S(\x) \simeq \rho_k y$ if $\x$ is of type $k$. This means that the predictor $\bar{\vtheta}$ has successfully learned feature $\v_k$ iff $\rho_k > \frac{\sigma_{\xi}^2}{\sqrt{n d}}$. In other words for this linear model the cutoff frequency is at $\rho_{\mathrm{cut}} = \frac{\sigma_{\xi}^2}{\sqrt{n d}}$. With a small leap of faith (related to the fact that after data augmentation the noise terms are no longer i.i.d., which we show to be not a  in our proof of non-linear model) we can see that as long as this cutoff frequency is smaller than $\frac{1}{\sqrt{K}}$, data augmentation would enable full learning of all the views, since in that case the post-augmentation frequencies $\frac{1}{K}$ are larger than the cutoff frequency with $n$ replaced by $nK$, \ie $\frac{1}{K} \gg \frac{\sigma_{\xi}^2}{\sqrt{n K d}} = \frac{\rho_{\mathrm{cut}}}{\sqrt{K}}$.

\paragraph{Effect of simple non-linearity on SNR.} The simplest type of ``non-linearity" would be to consider a tensor method for this problem (note that this is nothing but a kernel method). Specifically, let
$T = \frac{1}{n} \sum_{i=1}^n \sum_{p\in [P]} y_i \left(\x_p^{(i)}\right)^{\otimes q} \,,$
be the natural empirical tensor for this problem, for some odd $q \in \mathbb{N}$, whose domain is extended from $\R^d$ to $\R^{d \times P}$ as before, i.e., $T(\x) = \sum_{p \in [P]} T(\x_p)$. Note that this function can be realized in our architecture with a pure polynomial activation function $\psi(z) = z^q$,  see \cite{BLN21} for more on neural network memorization with tensors. Similarly to the linear case one can decompose the tensor into a signal and noise components:
\[
T = S + N, \text{ where } S= \sum_{k=1}^K \rho_k \v_k^{\otimes q}, N=\frac{1}{n} \sum_{i=1}^n y_i \left(\noise^{(i)}\right)^{\otimes q}  \,.
\]
For $\x$ sampled from our distribution, we have with high probability, $|N(\x)| \simeq \frac{\sigma_{\xi}^{2 q}}{\sqrt{n d^q}}$ and $S(\x) \simeq \rho_k y$ if $\x$ has $\v_k$ as its main feature. Thus here the cutoff frequency is at $\rho_{\mathrm{cut}}^{(q)} = \frac{\sigma_{\xi}^{2 q}}{\sqrt{n d^q}}$. In particular we see that even at high SNR, say when $\sqrt{n d} \gg \sigma_{\xi}^2 \gg \sqrt d$ (in which case $\rho_{\mathrm{cut}}^{(1)} = o(1)$) we might have $\rho_{\mathrm{cut}}^{(q)} = \Omega(1)$ for $q>1$. To put it differently, the tensor methods will overfit to the noise at a different SNR from the pure linear model would, which in turns mean that there is a different range of SNR where data augmentation will be useful for non-linear models such as tensors. We will see this story repeating itself for the gradient descent on our neural network architecture.

\paragraph{Quantitative comparison with the neural network results.} We note that the thresholds derived here are {\em better} than those we obtain via our neural network analysis (note also that the tensor method can handle $\alpha >0$ similarly to what our non-linearity allows). However we emphasize again that, on the contrary to gradient descent on neural networks, the predictors here are artificial and specifically tailored to the data distribution at hand. Furthermore the complexity of the tensor method scales up with $q$, on the contrary to the neural network dynamic. %
 \section{Overview of gradient descent dynamics}
 \label{sec:overview}
Let us do some heuristic calculation in the simple case where $\alpha = 0$ (so that effectively there are only two relevant patches in inputs,  $\x_{p^*}=y\v_{k^*}$ and $\x_{p^{\noisesubscript}}=\noise$, respectively).  
Recall that $\w_c(0)\sim\cN(0,\sigma^2_0I_d)$ and $\noise\sim \cN(0,\sigma_{\noisesubscript}^2I_d/d)$. Thus, $\E[|\w_c(0) \cdot \x_{p^*}|^2] = \sigma_0^2$ and $\E[|\w_c(0) \cdot \x_{p^{\noisesubscript}}|^2] = \sigma_0^2 \sigma_{\noisesubscript}^2$ for all channels $c$. We will initialize so that these quantities are $o(1)$, and thus $f(\w(0),\x) = o(1)$ for $(\x,y)\sim\cD$. We study the gradient flow on minimizing $f$ in this section. 
\subsection{When you really learn...}
For $f$ to correctly classify a datapoint $\x$ with feature $\v_k$, it is morally sufficient that $|\w_c \cdot \v_k|$ is of order $1$ for some channel $c$. Let us look at the dynamics starting close to initialization (when $f(\w(0),\x) = o(1)$),  %
\begin{align} 
	 \frac{d}{dt} \w_c \cdot \v_k\,\notag 
	& = -\frac{1}{n} \sum_{i \in [n]}y^{(i)}\,\ell'\big(y^{(i)} F(\w_c,\x^{(i)})\big)\left[\nabla_{\w_c} F(\w,\x^{(i)})\cdot \v_k\right] \notag \\
	& \overset{(a)}= \frac{1+o(1)}{2n} \sum_{i \in [n]} \sum_{p \in [P]} \psi'(|\w_c \cdot \x_p^{(i)}|)\, y^{(i)}\x^{(i)}_p \cdot \v_k \notag \\
	& = \frac{1+o(1)}{2n} \sum_{i \in [n]} \psi'(|\w_c \cdot \v_{k^*_i}|) \v_{k^*_i} \cdot \v_k+ \underbrace{\frac{1+o(1)}{2n} \sum_{i \in [n]}  \psi'(|\w_c \cdot \noise^{(i)}|) y^{(i)}\noise^{(i)} \cdot \v_k}_{{:=\vartheta}} \notag \\\vspace{-5pt}
	& \overset{(b)}= \frac{1+o(1)}{2} \,\rho_k\,\psi'(|{\w_c \cdot \v_k}|) + \vartheta \,, \label{eq:approx1}
\end{align}
where in $(a)$, we use $-\ell'(o(1))=1/2+o(1)$ for logistic loss $\ell$, $\psi'(z) =\psi(|z|)$ since $\psi$ is odd, and $(b)$ follows from $\{\v_{k}\}$ being orthogonal. 
\newline

\noindent\textit{If we can ignore $\vartheta$}, resulting dynamic reduces to an ODE of the form $g'(t) = \rho_k \psi'(g(t))$  (ignoring constants) with $g(0) \approx \sigma_0 = o(1)$. As long as $g(t)=\w_c(t)\cdot\v_k$ is smaller than $1$ this can be rewritten as $g'(t) = \rho_k g(t)^{q-1}$ (because of the form of $\psi$ we chose), or equivalently $(g(t)^{2-q})' = -\rho_k$ up to constants. In particular, we see that after time $t=g(0)^{2-q}/\rho_k$, we will have  $g(t)=\Theta(1)$. This suggests that by time of order $1 / (\sigma_0^{q-2} \rho_k)$ at least one channel should have learned $\v_k$\footnote{We assume $q\geq3$. For the case $q=1$ or $q=2$, the time needed is $1/(\sigma_0^{q-1}\rho_k)$. }.
\newline

\noindent\textit{When can we indeed ignore (morally) the noise term $\vartheta$?} At initialization this term is of order
$\frac{\sigma^{q-1}_0 \sigma_{\noisesubscript}^{q}}{\sqrt{n d}}$. On the other hand the ``main" term $\w_c\cdot\v_k$ in \eqref{eq:approx1} is of order $ \rho_k \sigma_0^{q-1}$. Thus we see that we need $\frac{\sigma_{\noisesubscript}^q}{\sqrt{n d}} \ll \rho_k$. In fact we will need a slightly more stringent condition, because the cancellation in $\vartheta$ leading to a scaling of $1/\sqrt{n}$ becomes more complicated to analyze after initialization due to the dependencies getting introduced. So we will use the more brutal bound $|\vartheta| \lesssim \frac{ \sigma_0^{q-1} \sigma^q_{\noisesubscript}}{\sqrt{d}}$ which in turn means we need  $\frac{\sigma^q_{\noisesubscript}}{\sqrt{d}} \ll \rho_k$.
\newline

Summarizing the above, we expect that if ${\sigma^q_{\noisesubscript}}/{\sqrt{d}} \ll \rho_k$, then by time $1/(\sigma_0^{q-2} \rho_k)$ we will have one channel that has learned the feature $\v_k$.

\subsection{... and when you overfit ...}
Another sufficient condition to correctly classify a datapoint $(\x^{(j)},y^{(j)})$ would be to overfit to its dominant noise part $\noise^{(j)}$, \ie $|\w_c \cdot \noise^{(j)}|$ is of order $1$ for some channel $c$. Here we have at initialization:
\begin{align}
	&\frac{d}{dt} \w_c \cdot \noise^{(j)} \notag  \\
	& = \frac{1+o(1)}{2n} \sum_{i \in [n]} \sum_{p \in [P]}  \psi'(|\w_c \cdot \x_p^{(i)}|) y^{(i)}\x_p^{(i)} \cdot \noise^{(j)} \notag \\
	& = \frac{1+o(1)}{2n} \Bigg(y^{(j)} \psi'(|\w_c \cdot \noise^{(j)}|) \|\noise^{(j)}\|^2 \Bigg.
+\psi'(|\w_c \cdot \v_{k_j^*}|) \v_{k^*_j} \cdot \noise^{(j)}
+ \sum_{i \neq j, p \in [P]}\psi'(|\w_c \cdot \x_p^{(i)}|)\, y^{(i)}\x_p^{(i)} \cdot \noise^{(j)}\Bigg) \notag \\
	& = \frac{(1+o(1)) \sigma^2_\noisesubscript}{2n} y^{(j)} \psi'(|\w_c \cdot \noise^{(j)}|)+ \Gamma \, \label{eq:approx2}
\end{align}
where $\Gamma$ is the last two term from the penultimate step. \newline 

\noindent \textit{Assuming $\Gamma$ can be ignored,} we can mimic the reasoning above (for $\w_c\cdot\v_k$) with $h(t)=y^{(j)}\w_c\cdot \noise^{(j)}$ and $h(0)=O(\sigma_0 \sigma_{\noisesubscript})$.  We thus expect to correctly classify a datapoint by {\em overfitting to its noise} after time $O(n / (\sigma_0^{q-2} \sigma_{\noisesubscript}^q))$.
\newline

\noindent\textit{When can we ignore the noise term $\Gamma$?} The order of $\Gamma$ is $\sigma_{\noisesubscript}^{q+1} \sigma_0^{q-1} / \sqrt{d}$ (at initialization it is in fact this times $1/\sqrt{n}$ but we ignore this improvement due to the dependencies arising through learning). On the other hand the main term in \eqref{eq:approx2} is of order $\sigma^{q+1}_\noisesubscript \sigma_0^{q-1} / n$ at initialization, so we obtain the condition $\sqrt{d} \gg n$ (which could possibly be improved to $d \gg n$ if cancellation remained correct throughout learning). \newline 

Summarizing again, if $d\gg n^2$, by time in the order of $n / (\sigma_0^{q-2} \sigma_{\noisesubscript}^q)$, we can expect the data points that were not fit before this time to be overfit using noise parameters.\newline

\subsection{... and in what order}
\label{sec:whatorder}
Let us assume $d \gg n^2$ and $\frac{\sigma^{q}_{\noisesubscript}}{\sqrt{d}} \ll \rho_k$. Then the above discussion reveals that if  $n / (\sigma^{q-2}_0 \sigma^{q}_{\noisesubscript}) \ll 1/ (\sigma_0 \rho_k) \Leftrightarrow \rho_k \ll \sigma^q_{\noisesubscript} / n$, we will not be able to learn $\v_k$ because we will overfit before learning (In fact, in this case, we do not need the condition $\frac{\sigma^{q}_{\noisesubscript}}{\sqrt{d}} \ll \rho_k$). 
This essentially gives rise to a channel filter (or a combination thereof) of the form \eqref{eq:mixture}, with the cutoff point $K_{out} =\{k: \rho_k \ll \sigma^q_{\noisesubscript} / n\}$ being now specified. 
\newline

Data augmentation can fix the order by effectively permuting the features. After data augmentation, we get the proportion of any feature to be $1/K$ and the training set size to be $nK$. Note that our data augmentation only permutes the coordinates so that the inner product between $\noise$ and $\T_k(\noise)$ should be at the same order as two independent noise. The learning process only depend on the inner product between the samples so our previous analysis still holds. Then, after data augmentation, for every view $k\in[K]$, we have $\rho_{k}^{\text{(aug)}} =1/K$. Then, as long as $ \sigma_{\xi}^q/n=o(1)$, we have $\rho_{k}^{\text{(aug)}} \gg\sigma_{\xi}^q/(nK)$ and are able to learn $\v_k$ before overfitting. 

\subsection{What about spurious features?}
In addition to overfitting noise, the model can also overfit spurious features. The spurious features can be viewed as noise vectors that appear in more than one sample. We did not prove this case formally in our main theorems for simplicity, but we will give the proof intuition here. Let $\spf \in \R^d$, $\norm{\spf} = 1$, be some spurious feature. Now assume that in addition to the dominant noise patch and the feature patch, $\spf$ appears in $1> \rho_{\spf}^{(-1)} > 0 $ fraction of the datapoints with label $y = 1$ and $\rho_{\spf}^{(-1)} < \rho_{\spf}^{(1)} $ fraction of the datapoints with label $y = -1$. We assume $\spf$ is orthogonal to the main features $\v_1, ..., \v_K$. Let $\mathcal{I}_{\spf}$ be the set of samples with $\spf$. We have at initialization:
\begin{align}
	\frac{d}{dt} \w_c \cdot \spf \notag 
	& = \frac{1+o(1)}{2n} \sum_{i \in [n]} \sum_{p \in [P]}  \psi'(|\w_c \cdot \x_p^{(i)}|) y^{(i)}\x_p^{(i)} \cdot  \spf \notag \\	
	& = \frac{1+o(1)}{2n}  \sum_{i \in \mathcal{I}_{\spf}} y^{(i)} \psi'(|\w_c \cdot \spf |) \norm{\spf}^2  + \underbrace{\frac{1+o(1)}{2n} \sum_{i \in [n]}  \psi'(|\w_c \cdot \noise^{(i)}|) y^{(i)}\noise^{(i)} \cdot\spf}_{{:=\Upsilon}}  \notag \\
	& = \frac{1+o(1) }{2n} (\rho_{\spf}^{(1)} - \rho_{\spf}^{(-1)} )\psi'(|\w_c \cdot \spf |) \norm{\spf}^2  + \Upsilon.  \, \label{eq:approx3}
\end{align}
\noindent \textit{Assuming $ \Upsilon$ can be ignored,} we can mimic the reasoning for $\w_c\cdot\v_k$ with $h(t)=\w_c\cdot \spf$ and $h(0)=O(\sigma_0)$.  We thus expect to correctly classify a datapoint in class $y =1$ with spurious feature $\spf$ by {\em overfitting to $\spf$} after time $O(n / (\sigma_0^{q-2}(\rho_{\spf}^{(1)} - \rho_{\spf}^{(-1)} )))$.
\newline

\noindent\textit{When can we ignore the noise term $ \Upsilon$?} Similar to the term $\vartheta$ in \eqref{eq:approx1}, the order of $ \Upsilon$ is $\sigma_{\noisesubscript}^{q} \sigma_0^{q-1} / \sqrt{nd}$. On the other hand the main term in \eqref{eq:approx3} is of order $ (\rho_{\spf}^{(1)} - \rho_{\spf}^{(-1)} ) \sigma_0^{q-1} $. Thus we need $\frac{\sigma^q_\xi}{\sqrt{nd}} \ll \rho_{\spf}^{(1)} - \rho_{\spf}^{(-1)} $. \newline

Summarizing above, since $\spf$ can appear in both class $y = 1$ and class $y =-1$, it should not be used as an indicator of the label $y$. However, when $\spf$ appears predominantly in one class (e.g., when $\frac{\sigma^q_\xi}{\sqrt{nd}} \ll \rho_{\spf}^{(1)} - \rho_{\spf}^{(-1)} $), the model can overfit $\spf$ and use $\spf$ to classify the datapoints.  %

\section{Main Results}

We learn the model $F(\w,\x)$ in (\ref{eq:model}) using
gradient descent with step size $\eta$ on loss $L(\w)$ in (\ref{eq:loss}). The weight $\w_{c}$, $c\in[C]$, at time step $t$ is denoted as $\w_c(t)$. 
\remove{The weight $\w_{c}$, $c\in[C]$, at time step $t$ for training on
$\Dtr=\left\{ (\x^{(i)},y^{(i)}),i\in[n]\right\} $ is given by 
\[
\w_{c}(t)=\w_{c}(t-1)-\eta\sum_{i=1}^{n}\ell'(y^{(i)}F(\w(t),\x^{(i)}))\nabla F(\w(t),\x^{(i)}).
\]}
The weight $\w_{c}(t)$ for training on $\Dtraug$ is obtained similarly,
with the samples replaced by $\Dtraug=\left\{ (\x^{(i)},y^{(i)}),i\in[Kn]\right\} $.
In addition to the assumptions we have discussed in Section \ref{sec:overview},
we make some additional assumptions for controlling the omitted quantities arising through training and testing.

\begin{assumption}
\label{assu:assume_main}We assume the following holds. For some constant $q\geq 3$,
\begin{enumerate}
\item \label{assu:views}The first view is dominant, $1\geq \rho_{1}\geq \Omega(1)$. The other views
$k\in[K]\backslash\left\{ 1\right\} $ are minor views, $n\rho_{k}\leq o\left(\sigma_{\xi}^{q}\right)$.

\item \label{assu:noise}The standard deviation of the dominant noise satisfies $\omega (1) \leq \sigma_{\xi}^{q} \leq o(n)$.
\item \label{assu:weight}The standard deviation of the weights at initialization is bounded,
$\sigma_{0}\leq o(1/\sigma_\xi)$.
\item \label{assu:nk}The number of samples and views are bounded, $ nK\leq o\left(\sigma_{0}^{q-1}\sigma_{\xi}^{q-1}d^{1/2} \right)$. 

\item \label{assu:fnoise}The feature noise satisfies, for $T=\tthe\left(\max\left\{ n\eta^{-1}\sigma_{\xi}^{-q}\sigma_{0}^{-q+2},K\eta^{-1}\sigma_{0}^{-q+2}\right\} \right)$,
\[
\omega(P^{-1})\leq\alpha\leq o\left(\eta^{-1}T^{-1}P^{-\frac{1}{q}}\sigma_{\xi}\min\{d^{-\frac{1}{2}}, \sigma_0\}\right).
\]
\end{enumerate}
\end{assumption}

Condition 1-3 in Assumption~\ref{assu:assume_main} have been explained in Section~\ref{sec:overview}. $\sigma_0\leq o(1)$ and $\sigma_0 \sigma_{\xi}\leq o(1)$ guarantee that at initialization, the main features and the dominant noise have $o(1)$ correlation with the weight. We choose $\sigma_{\xi}\geq \omega(1)$ so that without properly learning the main feature, the inner product between random initialized weights and the dominant patch can dominate the model output. Condition 4 is a more stringent version of the condition $n \ll d^{1/2}$ in Section~\ref{sec:overview} to control all the terms during training. 
  In Condition 5, we assume an upper bound on the feature noise $\alpha$. We assume the existence of feature noise only for establishing gap with linear models, so we did not optimize the upper bound on $\alpha$. It is possible the proof can go through with milder constraints on $\alpha$.\newline

An example of a set of parameters that satisfy the above assumption is 
\begin{equation*}
    \begin{split}
        &q=3,\sigma_{0}=d^{-0.15},\sigma_{\xi}=d^{0.1},n=d^{0.33},\\ & K=d^{0.06},
        \rho_{1}=\frac{1}{2},\rho_{2}=\rho_{3}=...=\rho_{K}=\frac{1}{2(K-1)},\\ &\alpha=d^{-0.95},P=d.        
    \end{split}
\end{equation*}

In Theorem $\ref{thm:withoutaug}$, we show that under the above conditions, without data augmentation,
gradient descent can find a classifier with perfect training accuracy without learning the
minor views. On the other hand, Theorem $\ref{thm:withaug}$ shows that
with data augmentation, all $k$ views can be learned without 
overfitting to noise.
\begin{restatable}[Training without data augmentation]{thm}{withoutaug}
\label{thm:withoutaug}Suppose that Assumption $\ref{assu:assume_main}$
holds. Let $\oT$ be the first time step such that $\w(\oT)$
can classify all $(\x^{(i)},y^{(i)})\in\Dtr$ with constant margin,
\ie, 
\[
y^{(i)}F(\w(\oT),\x^{(i)})\geq\tom(1), \text{ \;\;for all (\ensuremath{\x^{(i)}},\ensuremath{y^{(i)}})\ensuremath{\in\Dtr}}.
\]
For hidden channel number $C=\Theta(\log d)$, and small step size
$\eta$, with
probability at least $1-O(\frac{n^2K}{\poly(d)}),$ $\text{\ensuremath{\oT}}=\tthe\left(n\eta^{-1}\sigma_{\xi}^{-q}\sigma_{0}^{-q+2}\right)$.
Moreover, at time step $\oT$, views $\v_{2},\ldots,\v_{K}$ have never
been learned, so that $\forall_{0\leq t\leq\oT}$,
\[
\Pr_{(\x,y)\sim\D}\left[yF(\w(t),\x)<0\right]\geq\left(\frac{1}{2}-O\left(\frac 1 {\sqrt{C}}\right)\right)\sum_{k=2}^{K}\rho_{k}.
\]
\end{restatable}

\begin{restatable}[Training with data augmentation]{thm}{withaug}
\label{thm:withaug}Suppose assumption $\ref{assu:assume_main}$
holds. Let $\oTaug$ be the first time step such that $\w(\oTaug)$
can classify all $(\x^{(i)},y^{(i)})\in\Dtraug$ with constant margin, \ie
\[
y^{(i)}F(\w(\oTaug),\x^{(i)})\geq\tom(1),\text{ \;\;for all (\ensuremath{\x^{(i)}},\ensuremath{y^{(i)}})\ensuremath{\in\Dtraug}}.
\]
For hidden channels number $C=\Theta(\log d)$, and small step size
$\eta$, with
probability at least $1-O(\frac{n^2K^3}{\poly(d)})$, 
$\oTaug=\tthe\left(K\eta^{-1}\sigma_{0}^{-q+2}\right)$, and at $\oTaug$,
\[
\Pr_{(\x,y)\sim\D}\left[yF(\w(\oTaug),\x)<0\right]\leq\frac{nK}{\poly (d)}.
\]
\end{restatable}

\begin{rem}
In Theorem $\ref{thm:withoutaug}$ and Theorem $\ref{thm:withaug}$,
we evaluate the testing accuracy at the earliest time step $T$ when the trained neural network with weights $\w(T)$ can classify all samples in the training set $\Dtr$ with a constant margin. Our result does not rule
out the possibility that if trained longer than $\bar{T}$, the
network can learn the minor views as well. However, we should expect
the gradients on the training set stay small after the network can
classify all sample correctly. The main reason we assume an upper
bound on $\eta T$ is when training too long, the norm of the weights
$\w$ can blow up. One possible strategy to avoid
such upper bound on $\eta T$ is to add weight decay to the gradient
descent algorithm in training. 
\end{rem}

\begin{rem}
For simplicity of the proof, we only keep track of the channel with
the maximum correlation with the main feature or the noise, $\arg\max_{c\in[C]}\w_{c}(t)\cdot\v_{k}$
and $\arg\max_{c\in[C]}y\w_{c}(t)\cdot\vxi$. For the other channels,
we only give a rough bound on their correlation. For this reason,
we assume the number of channels is $C=\Theta(\log d)$ so that the
output is dominated by the channel with the maximum correlation. To
extend the result to higher number of channels, such as polynomial
in $d$, we need to keep track of all channels and scale the output
layer by $\frac{1}{C}$ .
\end{rem}

\begin{rem}
In our model, we show that when there exists some large dominant noise, the neural network overfits to
the noise instead of learning the minor features. In practice, the model
can overfit to any vector that contributes significantly to the gradient
of the loss. For example, our proof can be extended to the case where there exists some spurious feature that
appears in sufficiently many samples. In such case, even when the magnitude of the
spurious feature is smaller than the dominant noise in our distribution, the network can still overfit it.
\end{rem}

\captionsetup[figure]{font=small}

\section{Experiments}
\label{sec:exp}
Our theoretical results showed that data augmentation can make it harder to overfit to the noise components (the ``easy to learn and bad" feature in our model) by manipulating the relative gradient contribution of noise vs true features. To simplify our analysis, we assumed independent dominant noise in each sample. We hypothesize that the feature manipulation effect of data augmentation is broader in practice. In particular, our high level argument suggests that a model can also overfit to spurious features, like the \textit{grass feature} in our story of cows in the introduction, which have strong class dependent correlations. In Section~\ref{subsec:sf}, we show experiments to this effect that complement our theory. 
We further conduct two additional experiments that support this paper's thesis. 
In Section~\ref{subsec:davsindep}, we show an experiment with a modified data augmentation pipeline that demonstrates that the benefits of data augmentation cannot be fully explained by the learning of right invariance by the model. Finally, in Section~\ref{subsec:unbalanced} we elaborate on the problem with unbalanced views, where we show that adding extra samples from one dominant view to balanced dataset can hurt the performance of the learned models.
\subsection{Spurious Feature}
\label{subsec:sf}
We use images of the dog class and the cat class from CIFAR-10 dataset, which are of size $32\times 32$ pixels and  $3$ channels. We generate a row of random pixels $\vec{u}\sim\cN(0,\sigma^2I_d)$, where $d=32$ and $\sigma=25$, which is added as a synthetic spurious feature to a class dependent position in an image. The spurious feature $\vec{u}$ is added to the first channel in the row $r_{\textup{cat}}$ for cat images, and in row $r_{\textup{dog}}$ for dog images. For each  image $\x$ in the dataset, with probability $p<1$  we introduce a spurious feature, and with probability $(1-p)$ we leave it unperturbed.  We always select $r_{\textup{cat}}\in\{0,1,\ldots, 15\}$ in the upper half of the image, and $r_{\textup{dog}}\in\{16,17,\ldots,31\}$ in the lower half. In this way, the spurious feature position has a weak correlation to the class label. See Figure~\ref{fig:catdogvisualization} for sample images with spurious features. 
We consider three types training sets with varying degrees of data augmentation as shown Figure~\ref{fig:catdogvisualization}-(b,c,d). 
\begin{compactenum}
	\item \textit{No augmentation:} As a baseline without augmentation, we  center-crop the image to size $[3, 28, 28]$. 
	\item \textit{Random crop:} In each epoch, we  randomly crop a $[3, 28, 28]$ from the original $[3,32,32]$ image---a standard technique used in practice. This  would in essence disperse the position of spurious feature $\vec{u}$. For example, cat images with $\vec{u}$ in row $r_{\textup{cat}} = 9$, will now contain $\vec{u}$ in a row uniformly chosen from $r_{\textup{cat}}^\text{aug}\sim\mathcal{U}(\crl{5, 6, 7, 8, 9})$.   
	\item \textit{Randomized noise position:} Random crop, although standard, has a confounding effect that in addition to perturbing the position of $\vec{u}$, it might also incorporate other useful inductive biases about images. For a more direct comparison to the baseline, we also look at a special augmentation, wherein we perturb just the spurious feature row position by a uniform random number in $[-2, 2]$ in each epoch and then use a simple center crop. As in the case of random crop, this would again disperse the spurious feature from $r_{\textup{cat}} = 9$ to  $r_{\textup{cat}}^\text{aug}\in\mathcal{U}(\crl{5, 6, 7, 8, 9})$. But the non-spurious features/pixels remain the same as baseline.\newline

\end{compactenum}

\begin{figure*}[h]
	\centering
	\begin{subfigure}{.20\textwidth}
		\centering
		\hspace{-1mm}\includegraphics[width=.40\linewidth, height=0.4\linewidth]{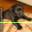}\\
		\vspace{0.25mm}
		\includegraphics[width=.40\linewidth, height=0.4\linewidth]{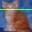}
		\label{fig:sample-orig}
		\caption{Original images}
	\end{subfigure}%
	\hspace{4mm}
	\begin{subfigure}{.20\textwidth}
		\centering
		\hspace{-1mm}\includegraphics[width=.35\linewidth,height=0.35\linewidth]{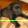}\\
		\vspace{0.5mm}
		\includegraphics[width=.35\linewidth,height=0.35\linewidth]{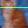}
		\label{fig:sample-noaug}
		\caption{No augmentation}
	\end{subfigure}%
	\hspace{4mm}
	\begin{subfigure}{.20\textwidth}
		\centering
		\hspace{-1mm}\includegraphics[width=.35\linewidth,height=0.35\linewidth]{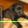}
		\includegraphics[width=.35\linewidth,height=0.35\linewidth]{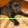}\\
		\vspace{0.5mm}
		\includegraphics[width=.35\linewidth,height=0.35\linewidth]{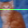}
		\includegraphics[width=.35\linewidth,height=0.35\linewidth]{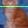}
		\label{fig:sample-random-crop}
		\caption{Random crop}
	\end{subfigure}%
	\hspace{4mm}
	\begin{subfigure}{.20\textwidth}
		\centering
		\hspace{-1mm}\includegraphics[width=.35\linewidth,height=0.35\linewidth]{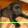}
		\includegraphics[width=.35\linewidth,height=0.35\linewidth]{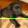}\\
		\vspace{0.5mm}
		\includegraphics[width=.35\linewidth,height=0.35\linewidth]{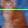}
		\includegraphics[width=.35\linewidth,height=0.35\linewidth]{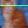}
		\label{fig:sample-random-feature}
		\caption{Random noise position}
	\end{subfigure}%
	\caption{Examples of training images in the spurious features experiment (Section~\ref{subsec:sf}). 
		For ease of visualization, we use a green line rather than random row vector $\vec{u}$ to indicate the spurious feature.  
		In the original $[3,32,32]$ images shown in (a), the  spurious feature is added to the first channel of row $r_{\textup{cat}} = 9$ for the cat class (lower images),  and of row $r_{\textup{dog}} = 22$ for the dog class (upper images).  
		Sub-figures (b,c,d) correspond to samples from different data augmentation methods described in the experiment.
		\label{fig:catdogvisualization}}
\end{figure*}

We compare the testing accuracy of training on these three types of training set in Figure~\ref{fig:catdog} for different values of $r_{\textup{cat}}$ and $r_{\textup{dog}}$.  
When $(r_{\textup{cat}} , r_{\textup{dog}}) = (15,16)$ (Figure~\ref{fig:catdog}, right), after data augmentation with either 
\textit{random noise position}  or \textit{random crop}, the position of $\vec{u}$ in the perturbed imaged has a large overlap across classes. So it is not surprising that the test accuracy with augmentation remains about the same for almost all values of $p$ (fraction of images with spurious features). 
On the other hand, for positions $(9,22)$ and $(12, 19)$ (Figure~\ref{fig:catdog}, left \& center)), although the two data augmentation techniques disperse the positions of spurious feature, its location in the two classes still stays separated. The cat images always have $\vec{u}$ in the upper half of the image while the dog images always have $\vec{u}$ in the lower half of the image. Interestingly, even so, the data augmentation, specially even the simple \textit{random feature position}, can improve the test accuracy. In this case, while augmentation does not eliminate the existence of spurious features, it still diminishes them  by making the spurious features harder to be learned and overfitted. In addition to shifting the spurious features, random crop can shift other important features as well to boost the minor views, so the testing accuracy when training with random crop can be even higher than only shifting the spurious feature position.

\begin{figure*}[h!]
	\centering
	\includegraphics[scale=0.32]{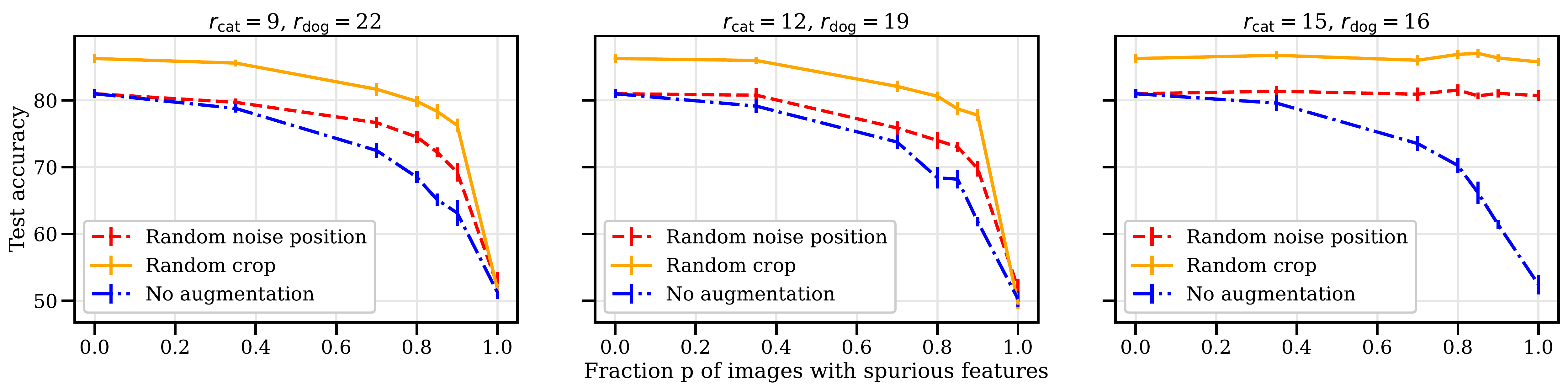}
	\caption{Comparison of different data augmentation strategies for the CIFAR-10 cat-vs-dog classification task with a synthetic spurious feature. The plots show results for different sets of positions of spurious feature $(r_{\textup{cat}},r_{\textup{dog}})$ as we vary the fraction $p$ of all the images that have the spurious feature. The plots are averaged over five runs with error bars of one standard deviation. The test datapoints are always center-cropped images of size $[3,28,28]$ with no spurious feature. In all configurations, we train a ResNet20 network using SGD for $120$ epochs with momentum $0.9$, weight decay $0.005$, and learning rate starting at $0.1$ and annealed to $(0.01,0.001)$ at epochs $(40,80)$. }\label{fig:catdog}
\end{figure*}

\subsection{Augmented samples vs. independent samples}
\label{subsec:davsindep}
\remove{When using data augmentation in practice, a new random transformation (\eg  randomly flip the image or crop the image at random position) is used in each epoch of training. Motivation usually is to encourage the model to learn task specific invariance. For example, if a model sees many variations of an image with few pixel translation shifts (across epochs), then we expect it will learn to be invariant to translations, which in turn boosts performance. Our paper provides a complementary hypothesis by which data augmentation can be helpful even without learning invariance. In this experiment we support our hypothesis by showing that data augmentation could be beneficial even when the model is unlikely to learn good invariance properties. \newline
	
	To demonstrate this, we conduct an experiment where we fix the augmented dataset prior to training and use the same fixed dataset for all epochs. 
	
	Another common motivation is to encourage  the model to learn task specific invariance, but training with many variations of the same input image across epochs. 
}

When using data augmentation, typically a new random transformation (\eg  random flip or crop at a random position of an image) is used in each epoch of training. This procedure effectively increases the training dataset size (albeit with non i.i.d correlated samples). In this experiment, we control for the number of unique samples seen by the training algorithm and ask the question: \textit{how effective is a single data augmented sample compared to an independent sample?}\newline 

For this experiment, we work with the full CIFAR-10 dataset which has $50000$ training examples for $10$ classes. Given a ratio $p$ of independent samples to total sample size, we generate a training set of size $n=50000$ as follows: We first select $pn$ independent samples for the task. We then cyclically generate a data augmented variant these $pn$ independent samples until we obtain the remaining $(1-p)n$ datapoint. 
For example, in the CIFAR-10 dataset with $n=50000$, if $p=0.6$, the training set consists of $30000$ independent samples, of which $20000$ have one additional augmented sample. If $p = 0.2$, the training set has $10000$ independent samples and four data augmented versions of each of the $10000$ independent samples. Thus, for $p=1$, there is no augmentation, and for smaller $p$, there are more augmented samples, but less independent samples. The dataset thus generated is then fixed for all epochs. In this way, the number of unique samples seen by training algorithm is always $n=50000$ for all $p$.\newline 

\begin{figure}[!h]
	\centering
	\includegraphics[scale=0.5]{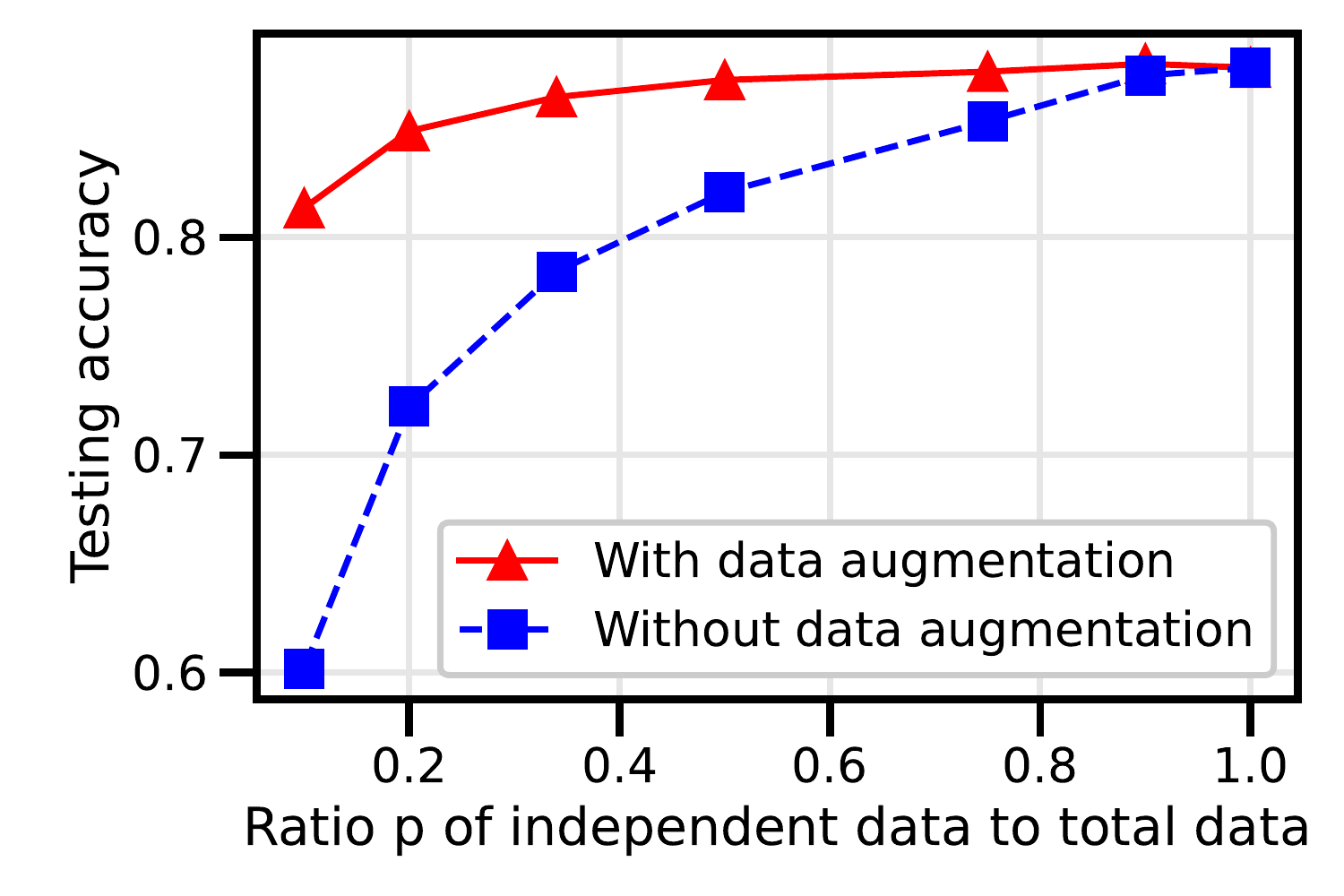}
	\caption{Augmented vs independent samples: for each $p$ on the x-axis, the data augmented training (red-solid curve) uses $50000p$ independent images from CIFAR-10, along with  $50000(1-p)$ data augmented samples. The  augmented dataset is fixed across epochs. For the baseline without data augmentation (blue-dashed curve) we simply use the $50000p$ independent samples. We use the standard CIFAR-10 test dataset and the results are averaged over $3$ runs. In each instance, we train a ResNet20 for $160$ epochs using SGD with momentum $0.9$, weight decay $0.005$, and learning rate starting at $0.1$ and annealed to $(0.01,0.001)$ at epochs $(80,120)$. }\label{fig:indepvsaug}
\end{figure}

In Figure~\ref{fig:indepvsaug}, we compare the accuracies of a ResNet20 model trained on such partially data augmented samples to the baseline of training with just the $pn$ independent samples without any augmentation. 
Our experiment shows that even this partial data augmentation can significantly improve the testing accuracy.  In this experiment,  since each example has only a small number of augmented variations (\eg for $p\ge 0.5$ at most one augmented variant of the an example is seen throughout training), it is unlikely that they lead to learning any kind of task specific invariance, which is the usual motivation. However, by having the important feature appearing at a slightly different location, data augmentation can still facilitate the learning of the important features via the feature manipulation view described in our paper. 
Furthermore, comparing the accuracy of un-augmented full dataset with $p=1.0$ on blue-dashed curve to that of data augmented training for $p\ge 0.5$ on the red curve, we see that a fixed data augmented image can improve the test accuracy nearly as much as an independent sample does. This shows that if we have an important feature in an image, \eg a cat ear, shifting it two pixels can help nearly as effectively as a completely new cat ear.

\subsection{Unbalanced Dataset}
\label{subsec:unbalanced}
In this experiment, we train a simple convolutional neural network on a synthetic dataset with unbalanced views. We show that when one view is much more prevalent in the dataset than the other views, having more samples of the dominant view can hurt learning. 
Our data consist of samples $(\x,y)$ from two classes $y \in\{-1,1\}$. The input $\x\in\R^{3\times 15}$ has $3$ channels, each with $15$ pixels. After sampling $y$ uniformly, we generate $\x$ by setting one of the $15$ pixels to the main feature $[y,y,y]$. The other pixels are set to a Gaussian noise $\mathcal{N}(0, \sigma_{\noisesubscript}^2I_3)$. For different choices of $\sigma_{\noisesubscript}$, we first construct a balanced dataset $\D_{\textup{bal}}$ of size $n_{\textup{bal}}$ such that roughly equal number of samples that have the good feature $[y,y,y]$ present at each pixel. Our full dataset $\D_{\textup{full}}$ with $n_{\textup{full}}$ samples consists of $\D_{\textup{bal}}$ along with additional $n_{\textup{full}} - n_{\textup{bal}}$ samples with the main feature only at pixel $3$. We use a balanced  testing dataset.  

\begin{figure*}[h!]
	\centering
	\includegraphics[scale=0.45]{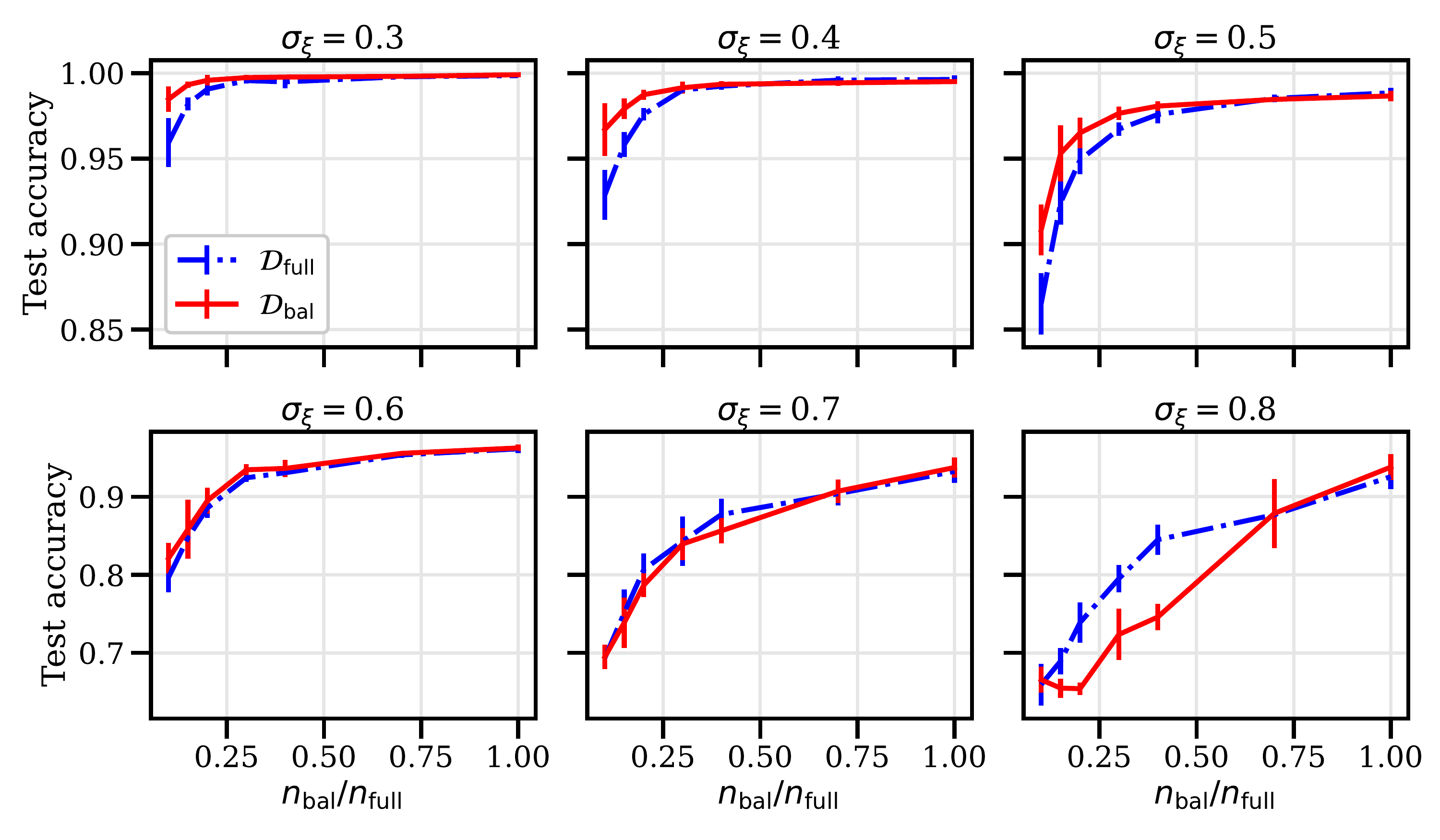}
	\caption{Comparison of training on $\D_{\textup{bal}}$ to $\D_{\textup{full}}$ as we vary the ratio of balanced examples $n_{\textup{bal}}/n_{\textup{full}}$ for different values of noise magnitude $\sigma_{\noisesubscript}$. 
		We learn the data using a simple convolutional neural network with two convolutional layers with ReLU activation, a maxpool layer and a linear layer. The two convolutional layers and the max pool layer have kernel size 4, and strides 2,1 and 2, respectively. The models are trained for $200$ epochs  using SGD with momentum $0.9$, weight decay $0.05$, and learning rate starting at $0.1$ and annealed to $0.01$ at epoch $100$. For all training sets, the training accuracy at the end of training is at least $0.99$. 
	} \label{fig:unbalanced}
\end{figure*}

In Figure~\ref{fig:unbalanced}, we see that compared to the balanced dataset $\D_{\textup{bal}}$, although the full dataset $\D_{\textup{full}}$ has strictly more samples with the accurate kind of features, when $\sigma_\xi$ is not too large, the test accuracy is consistently on par or even lower than training on just the balanced subset. In this case, the views are simply features positioned at different pixels. For very large $\sigma_\xi$, the test accuracy of the balanced subset can be low because in such case, the full dataset can learn the dominant view well, but the unbalanced dataset has too few samples to learn any view.  The experiment shows that even for architectures such as convolutional networks, which are believed to have some translation invariance, we should not expect samples from one view to help the learning of other views.

\section*{Acknowledgements}
We would like to thank Yi Zhang for valuable feedback.
 \newpage{}
\bibliographystyle{plainnat}
\bibliography{ref.bib}

\begin{thebibliography}{27}
\providecommand{\natexlab}[1]{#1}
\providecommand{\url}[1]{\texttt{#1}}
\expandafter\ifx\csname urlstyle\endcsname\relax
  \providecommand{\doi}[1]{doi: #1}\else
  \providecommand{\doi}{doi: \begingroup \urlstyle{rm}\Url}\fi

\bibitem[Allen-Zhu and Li(2020{\natexlab{a}})]{allen2020feature}
Zeyuan Allen-Zhu and Yuanzhi Li.
\newblock Feature purification: How adversarial training performs robust deep
  learning.
\newblock \emph{arXiv preprint arXiv:2005.10190}, 2020{\natexlab{a}}.

\bibitem[Allen-Zhu and Li(2020{\natexlab{b}})]{allen2020towards}
Zeyuan Allen-Zhu and Yuanzhi Li.
\newblock Towards understanding ensemble, knowledge distillation and
  self-distillation in deep learning.
\newblock \emph{arXiv preprint arXiv:2012.09816}, 2020{\natexlab{b}}.

\bibitem[Azulay and Weiss(2019)]{azulay2019deep}
Aharon Azulay and Yair Weiss.
\newblock Why do deep convolutional networks generalize so poorly to small
  image transformations?
\newblock \emph{Journal of Machine Learning Research}, 20:\penalty0 1--25,
  2019.

\bibitem[Berry(1941)]{berry1941accuracy}
Andrew~C Berry.
\newblock The accuracy of the gaussian approximation to the sum of independent
  variates.
\newblock \emph{Transactions of the american mathematical society}, 49\penalty0
  (1):\penalty0 122--136, 1941.

\bibitem[Bishop(1995)]{bishop1995training}
Chris~M Bishop.
\newblock Training with noise is equivalent to tikhonov regularization.
\newblock \emph{Neural computation}, 7\penalty0 (1):\penalty0 108--116, 1995.

\bibitem[Bubeck et~al.(2021)Bubeck, Li, and Nagaraj]{BLN21}
S{\'e}bastien Bubeck, Yuanzhi Li, and Dheeraj Nagaraj.
\newblock A law of robustness for two-layers neural networks.
\newblock \emph{Conference on Learning Theory (COLT)}, 2021.

\bibitem[Chapelle et~al.(2001)Chapelle, Weston, Bottou, and
  Vapnik]{chapelle2001vicinal}
Olivier Chapelle, Jason Weston, L{\'e}on Bottou, and Vladimir Vapnik.
\newblock Vicinal risk minimization.
\newblock \emph{Advances in neural information processing systems}, pages
  416--422, 2001.

\bibitem[Chen et~al.(2020)Chen, Dobriban, and Lee]{chen2020group}
Shuxiao Chen, Edgar Dobriban, and Jane~H Lee.
\newblock A group-theoretic framework for data augmentation.
\newblock \emph{Journal of Machine Learning Research}, 21\penalty0
  (245):\penalty0 1--71, 2020.

\bibitem[Ciregan et~al.(2012)Ciregan, Meier, and Schmidhuber]{ciregan2012multi}
Dan Ciregan, Ueli Meier, and J{\"u}rgen Schmidhuber.
\newblock Multi-column deep neural networks for image classification.
\newblock In \emph{2012 IEEE conference on computer vision and pattern
  recognition}, pages 3642--3649. IEEE, 2012.

\bibitem[Dao et~al.(2019)Dao, Gu, Ratner, Smith, De~Sa, and
  R{\'e}]{dao2019kernel}
Tri Dao, Albert Gu, Alexander Ratner, Virginia Smith, Chris De~Sa, and
  Christopher R{\'e}.
\newblock A kernel theory of modern data augmentation.
\newblock In \emph{International Conference on Machine Learning}, pages
  1528--1537. PMLR, 2019.

\bibitem[DeVries and Taylor(2017)]{devries2017improved}
Terrance DeVries and Graham~W Taylor.
\newblock Improved regularization of convolutional neural networks with cutout.
\newblock \emph{arXiv preprint arXiv:1708.04552}, 2017.

\bibitem[Elandt(1961)]{elandt1961folded}
Regina~C Elandt.
\newblock The folded normal distribution: Two methods of estimating parameters
  from moments.
\newblock \emph{Technometrics}, 3\penalty0 (4):\penalty0 551--562, 1961.

\bibitem[Feng et~al.(2021)Feng, Gangal, Wei, Chandar, Vosoughi, Mitamura, and
  Hovy]{feng2021survey}
Steven~Y Feng, Varun Gangal, Jason Wei, Sarath Chandar, Soroush Vosoughi,
  Teruko Mitamura, and Eduard Hovy.
\newblock A survey of data augmentation approaches for nlp.
\newblock \emph{arXiv preprint arXiv:2105.03075}, 2021.

\bibitem[Hanin and Sun(2021)]{hanin2021data}
Boris Hanin and Yi~Sun.
\newblock How data augmentation affects optimization for linear regression.
\newblock \emph{Advances in Neural Information Processing Systems}, 34, 2021.

\bibitem[Krizhevsky et~al.(2017)Krizhevsky, Sutskever, and
  Hinton]{krizhevsky2017imagenet}
Alex Krizhevsky, Ilya Sutskever, and Geoffrey~E Hinton.
\newblock Imagenet classification with deep convolutional neural networks.
\newblock \emph{Communications of the ACM}, 60\penalty0 (6):\penalty0 84--90,
  2017.

\bibitem[Mei et~al.(2021)Mei, Misiakiewicz, and Montanari]{mei2021learning}
Song Mei, Theodor Misiakiewicz, and Andrea Montanari.
\newblock Learning with invariances in random features and kernel models.
\newblock \emph{arXiv preprint arXiv:2102.13219}, 2021.

\bibitem[Nagarajan and Kolter(2019)]{nagarajan2019uniform}
Vaishnavh Nagarajan and J.~Zico Kolter.
\newblock Uniform convergence may be unable to explain generalization in deep
  learning.
\newblock In H.~Wallach, H.~Larochelle, A.~Beygelzimer, F.~d\textquotesingle
  Alch\'{e}-Buc, E.~Fox, and R.~Garnett, editors, \emph{Advances in Neural
  Information Processing Systems}, volume~32. Curran Associates, Inc., 2019.
\newblock URL
  \url{https://proceedings.neurips.cc/paper/2019/file/05e97c207235d63ceb1db43c60db7bbb-Paper.pdf}.

\bibitem[Neyshabur et~al.(2015)Neyshabur, Tomioka, and
  Srebro]{neyshabur2015search}
Behnam Neyshabur, Ryota Tomioka, and Nathan Srebro.
\newblock In search of the real inductive bias: On the role of implicit
  regularization in deep learning.
\newblock In \emph{ICLR (Workshop)}, 2015.

\bibitem[Rajput et~al.(2019)Rajput, Feng, Charles, Loh, and
  Papailiopoulos]{rajput2019does}
Shashank Rajput, Zhili Feng, Zachary Charles, Po-Ling Loh, and Dimitris
  Papailiopoulos.
\newblock Does data augmentation lead to positive margin?
\newblock In \emph{International Conference on Machine Learning}, pages
  5321--5330. PMLR, 2019.

\bibitem[Shorten and Khoshgoftaar(2019)]{shorten2019survey}
Connor Shorten and Taghi~M Khoshgoftaar.
\newblock A survey on image data augmentation for deep learning.
\newblock \emph{Journal of Big Data}, 6\penalty0 (1):\penalty0 1--48, 2019.

\bibitem[Simard et~al.(2000)Simard, Le~Cun, Denker, and
  Victorri]{simard2000transformation}
Patrice~Y Simard, Yann~A Le~Cun, John~S Denker, and Bernard Victorri.
\newblock Transformation invariance in pattern recognition: Tangent distance
  and propagation.
\newblock \emph{International Journal of Imaging Systems and Technology},
  11\penalty0 (3):\penalty0 181--197, 2000.

\bibitem[Simard et~al.(2003)Simard, Steinkraus, and Platt]{simard2003best}
Patrice~Y Simard, Dave Steinkraus, and John~C Platt.
\newblock Best practices for convolutional neural networks applied to visual
  document analysis.
\newblock In \emph{Seventh International Conference on Document Analysis and
  Recognition, 2003. Proceedings.}, volume~3, pages 958--958. IEEE Computer
  Society, 2003.

\bibitem[Wu et~al.(2020)Wu, Zhang, Valiant, and R{\'e}]{wu2020generalization}
Sen Wu, Hongyang Zhang, Gregory Valiant, and Christopher R{\'e}.
\newblock On the generalization effects of linear transformations in data
  augmentation.
\newblock In \emph{International Conference on Machine Learning}, pages
  10410--10420. PMLR, 2020.

\bibitem[Yaeger et~al.(1996)Yaeger, Lyon, and Webb]{yaeger1996effective}
Larry Yaeger, Richard Lyon, and Brandyn Webb.
\newblock Effective training of a neural network character classifier for word
  recognition.
\newblock \emph{Advances in neural information processing systems}, 9:\penalty0
  807--816, 1996.

\bibitem[Yang et~al.(2022)Yang, Dong, Ward, Dhillon, Sanghavi, and
  Lei]{yang2022sample}
Shuo Yang, Yijun Dong, Rachel Ward, Inderjit~S Dhillon, Sujay Sanghavi, and
  Qi~Lei.
\newblock Sample efficiency of data augmentation consistency regularization.
\newblock \emph{arXiv preprint arXiv:2202.12230}, 2022.

\bibitem[Zhang et~al.(2021)Zhang, Bengio, Hardt, Recht, and
  Vinyals]{zhang2021understanding}
Chiyuan Zhang, Samy Bengio, Moritz Hardt, Benjamin Recht, and Oriol Vinyals.
\newblock Understanding deep learning (still) requires rethinking
  generalization.
\newblock \emph{Communications of the ACM}, 64\penalty0 (3):\penalty0 107--115,
  2021.

\bibitem[Zhang et~al.(2017)Zhang, Cisse, Dauphin, and
  Lopez-Paz]{zhang2017mixup}
Hongyi Zhang, Moustapha Cisse, Yann~N Dauphin, and David Lopez-Paz.
\newblock mixup: Beyond empirical risk minimization.
\newblock \emph{arXiv preprint arXiv:1710.09412}, 2017.

\end{thebibliography}

\pagebreak{}

\appendix

\section*{Appendix}
We clarify that throughout the appendix $c_1,c_2,\ldots$ denote constants, while $C$ denotes the number of channels in our model \eqref{eq:model} and is not a constant, but is a function of $d$. 
Throughout the appendix, for any sample $(\x^{(i)}, y^{(i)})$, we let $\pbpi$ be the background patches of $(\x^{(i)}, y^{(i)})$ and for $k\in[K]$, $\pbpki$ be the  background patches with feature noise $-\api y\v_k$.

\section{Useful concentration lemmas}\label{app:concentration}
We first state the following standard results on Gaussian samples. These will be used in our proof frequently. .

\begin{lemma}[Laurent-Massart $\chi^2$ tail bound] Consider a standard Gaussian vector $\z\sim\cN(0,I_d)$. For any positive vector $\vec{a}\in\R^d_{\ge0}$, and any $t>0$, the following concentration holds
	\begin{flalign*}
		\text{Pr}\brk[\Big]{\sum_{i=1}^d \vec{a}_i\z_i^2\ge \nrm{\vec{a}}_1+2\nrm{\vec{a}}_2\sqrt{t}+2\nrm{\vec{a}}_\infty t}&\le \exp(-t),\\
		\text{Pr}\brk[\Big]{\sum_{i=1}^d \vec{a}_i\z_i^2\le \|\vec{a}\|_1-2\|\vec{a}\|_2\sqrt{t}}&\le \exp(-t).
	\end{flalign*}
	\label{lem:lm}
\end{lemma}
The following corollary immediately follows from using $t=\log{(2/\delta)}$ and $\vec{a}_i=1$ in the above lemma
\begin{corollary}[$\ell_2$ norm of Gaussian vector]\label{cor:gl2norm}Consider $\z\sim\cN(0,\sigma^2I_d)$, for any $\delta\in(0,1)$ and large enough $d$, we have with probability greater than $1-\delta$,  %
	\begin{equation*}
		\sigma^2d\prn[\Bigg]{1-2\sqrt{\frac{\log(2/\delta)}{d}}}\le \nrm{\z}^2_2\le \sigma^2d\prn[\Bigg]{1+4\sqrt{\frac{\log(2/\delta)}{d}}}.
	\end{equation*}
\end{corollary}

\begin{lemma}[Gaussian correlation]\label{lem:gausscorr}Consider independently sampled Gaussian vectors $\z_1\sim\cN(0,\sigma_1^2I_d)$ and $\z_2\in\N(0,\sigma^2_2I_d)$. For any $\delta\in(0,1)$ and large enough $d$, there exists a constant $c_1,c_2$  such that
	\begin{equation*}
		\begin{split}
			|\z_1\cdot \z_2| &\le  c_1\sigma_1\sigma_2\sqrt{d\log(2/\delta)} \quad\text{w.p }\ge 1-\delta,\\
			\z_1\cdot \z_2 &\ge c_2\sigma_1\sigma_2\sqrt{d} \quad\text{w.p }\ge 1/4.
		\end{split}
	\end{equation*}
\end{lemma}
{
\begin{proof}
	Let $u=\|\z_2\|_2$ and $v=\z_1\cdot\frac{\z_2}{\|\z_2\|_2}$. Since $\z_1$ is spherically symmetric, we have $v\sim\N(0,\sigma_1^2)$ and is independent of $u$. We first show the upper bound. 
	\begin{align}
		\Pr(|uv|\ge t)&=\Pr(|v|\ge t/u,u\ge c)+\Pr(|v|\ge t/u,u\le c)\tag{holds $\forall\,c>0$}\\
			&\le \min_{c>0} \Pr(u\ge c) + \Pr(|v|\ge t/c)\nonumber\\
			&\le \Pr(u \ge 2\sigma_2\sqrt{d}) + \Pr(|v|\ge \frac{t}{2\sigma_2\sqrt{d}})\tag{using $c=2\sigma_2\sqrt{d}$}\\
			&{\le} \exp(-\frac{d}{4})+\exp(-\frac{t^2}{8\sigma_1^2\sigma_2^2d})\tag{using Lemma~\ref{lem:lm} on $u=\|\z_2\|_2$ and $v\sim\N(0,\sigma_1^2)$}
	\end{align}
	Using $t=2\sigma_1\sigma_2\sqrt{2d\log(2/\delta)}$, we get the first inequality for all {$d\ge 4\log(2/\delta)$}. \newline 

	For the lower bound, using a similar argument as above we have
	\begin{align}
		\Pr(uv\le t)&\le \min_{c>0}  \Pr(v\le t/c)+\Pr(u\le c)\nonumber\\
			&\le \Pr(v\le \frac{\sigma_1}{4})+\Pr(u \le \frac{1}{2}\sigma_2\sqrt{d}) \tag{using $c=\frac{1}{2}\sigma_2\sqrt{d}$ and $t=\frac{1}{8}\sigma_1\sigma_2\sqrt{d}$}\\
			&{\le}\, \frac{5}{8}+\exp(-\frac{d}{16})\le \frac{3}{4}\tag{using Lemma~\ref{lem:lm} on $u$ and cdf bound on $v$}
	\end{align}
	The lower bound thus holds for $d\ge 16\log(8)$ using $t=\frac{1}{8}\sigma_1\sigma_2\sqrt{d}$. This concludes the proof of the lemma.
\end{proof}
}

\begin{lemma}[Gaussian tail concentration]\label{lem:maxgauss} Consider i.i.d samples $\crl{z_c\sim\cN(0,\sigma^2):c\in[C]}$. We have the following:
	\begin{flalign*}
		\max_{c\in[C]} |z_c|&\le \sigma \sqrt{2\log{\frac{2C}{\delta}}}, \quad \text{w.p} \ge 1-\delta,\\
		\max_{c\in[C]} z_c&\ge \frac{\sigma}{2}, \quad \text{w.p} \ge 1-\exp(-C/4).
	\end{flalign*}
\end{lemma}
\begin{proof}
	These are standard Gaussian tail bounds, which we prove here for completeness.
	We have:
	\begin{equation*}
		\Pr(\max_{c\in[C]}z_c\ge t)\le \sum_{c\in[C]}\Pr(z_c\ge t)\le C\exp(\frac{-t^2}{2\sigma^2}).%
	\end{equation*}
	Using the same argument for over $2C$ variables $\{z_c\sim\N(0,\sigma^2),-z_c\sim\N(0,\sigma^2)\}_{c\in[C]}$ along with $t=\sigma\sqrt{2\log(2C/\delta)}$, we have the first inequality that $\max_{c\in[C]} |z_c|\le \sigma \sqrt{2\log{\frac{2C}{\delta}}}, \quad \text{w.p} \ge 1-\delta$.\newline

	\noindent Furthermore, $\forall_{c\in[C]}$, we have $\Pr(z_c\ge \sigma/2)\ge 1/4$, hence
	\[\Pr(\max_{c\in[C]}z_c\ge \sigma/2)\ge 1-\big(1-1/4\big)^C\ge 1-\exp(-C/4)\]
	This concludes the proof of the lemma. \end{proof}

\begin{lem}[Berry--Esseen theorem \citep{berry1941accuracy}]
	\label{lem:bethm}Consider i.i.d samples $\{u_{i}:i\in[n]\}$ with
	$\E u_{i}=0$, $\E u_{i}^{2}=\sigma^{2}>0$ and $\E\left|u_{i}\right|^{3}=\rho<\infty$.
	Let $F_{n}$ be the cumulative distribution function of $\frac{1}{\sigma\sqrt{n}}\sum_{i=1}^{n}u_{i}$,
	and $\Phi$ be the cumulative distribution function of the standard
	normal distribution. For all $t$, there exists a constant $c_{1}$
	such that
	\[
	\left|F_{n}(t)-\Phi(t)\right|\leq\frac{c_{1}\rho}{\sigma^{3}\sqrt{n}}.
	\]
\end{lem}

\begin{lem}[Anti-concentration of $q$-th power of Gaussian random variables]
	\label{lem:guassianq} Consider i.i.d samples $\{z_{c}\sim\mathcal{N}(0,1):c\in[C]\}$.
	\textup{For constant integer $q\geq 1$, there exist constants $c_{1},c_{2}>0$}
	such that for any $t\leq o(1)$,
	\[
	\Pr\left[\sum_{c\in[C]}z_{c}^{q}\geq c_{1}t\sqrt{C}\right]\geq\frac{1}{2}-o(1)-\frac{c_{2}}{\sqrt{C}}.
	\]
\end{lem}

\begin{proof}
	For constant $q$, $\E z_{c}^{2q}\leq O(1)$ and $\E\left|z_{c}\right|^{3q}\leq O(1)$
	\citep{elandt1961folded}. Then, by Lemma \ref{lem:bethm}, for any
	$t$, there exist $c_{1}$ and $c_{2}$ such that
	\begin{align*}
		\Pr\left[\frac{1}{c_{1}\sqrt{C}}\sum_{c\in[C]}z_{c}^{q}\geq t\right] & \geq\Pr\left[z_{1}\geq t\right]-\frac{c_{2}}{\sqrt{C}}.
	\end{align*}

	Choosing $t=o(1)$ proves the lemma.
\end{proof}
\section{Additional notation}\label{app:additional_notation}

Recall the data distribution $\D$ from Definition~\ref{def:data}.  Further recall that, for $i\in[n]$,  we use $k^*_i$, $p^*_i$, $\noisepatchindex_i$, $\noise^{(i)}$, and $(\alpha_{p,i},k_{p,i})_{p\notin\crl{p^*_i,p^{\noisesubscript}_i}}$ to denote the respective quantities $k^*$, $p^*$, $\noisepatchindex$, $\noise$, and $(\alpha_{p},k_{p})_{p\notin\crl{p^*,p^{\noisesubscript}}}$ in Definition~\ref{def:data} for the $i^\text{th}$ training sample $(\x^{(i)}, y^{(i)})\in\Dtr\sim\D$. In addition to these notation in Section~\ref{sec:prelim}, we introduce the following additional notation for the proofs.
\begin{enumerate}
	\item $\forall\,k\in[K]$, let $\mathcal{I}_k=\crl{i\in[n]:k_i^*=k}$ denote the set of indices of the training data $(x,y)$ with $y\v_k$ as the main feature. Further, let $n_k=|\mathcal{I}_k|$
	\item $\forall\, i\in[n]$ and $\forall\,k\in[K]$, let $\pbpki$ be the  background patches  of the $i^{\text{th}}$ sample with $k^\text{th}$-type feature noise, \ie \[\pbpki=\crl{p\in[P]\setminus \crl{p_i^*,p_i^{\noisesubscript}}:\x^{(i)}_p=-\api y\v_k};\] and let  $\pbpi=\bigcup_{k\in[K]}\pbpki=[P]\setminus\crl{p_i^*,p_i^{\noisesubscript}}$ denote the set of all background patches of the $i^{\text{th}}$ sample.\newline 
\end{enumerate}

\begin{remark}
	For $k\in[K]$, let $\hat{\rho}_{k}=\frac{1}{n}|\mathcal{I}_k|$ denote the empirical fraction in the training data of $k^{\text{th}}$. Recall that $k_i$ are sampled independently with $\text{Pr}(k_i^*=k)=\rho_k$.  Thus, with with high probability, $\rho_{k}$ and $\hat{\rho}_{k}$ differ at most by $\sqrt{\frac{\log(n)}{n}}$. In the rest of the paper,  for simplicity we assume $\rho_{k}=\hat{\rho}_{k}$.\newline

	Similarly, let $\hat{\rho}_{k}^{\text{(noise)}}$ be the proportion of feature
	noise $-y\v_{k}$ in dataset $\Dtr$, i.e., $\hat{\rho}_{k}^{\text{(noise)}}=\frac{1}{n(P-2)}|\crl{i\in[n],p\in[P]\setminus\{p^*_i,p^\noisesubscript_i\}]|k_{p,i}=k}|$
	Again from standard concentration, we have $\rho_{k}$ and $\hat{\rho}_{k}^{\text{(noise)}}$ differ by negligible quantity with high probability, thus we also assume $\rho_{k}=\hat{\rho}_{k}^{\text{(noise)}}$.

\end{remark}

\section{Proof of initialization conditions in Lemma~\ref{lem:init}}\label{app:initialization}

\ginitlemma*

\begin{proof}
	Recall the setting of the lemma: $\forall\,{k\in[K]}$, $\|\v_k\|_2=1$, $\forall\, {i\in[n]}$, $y^{(i)}\noise^{(i)}\overset{i.i.d}\sim\N(0,\frac{\sigma^2_{\noisesubscript}}{d}I_d)$, and $\forall\,{c\in[C]}$, $\w_c(0)\overset{i.i.d}{\sim}\N(0,\sigma_0^2I_d)$. We have the following arguments that prove the lemma, where we  use  $\delta=\frac{1}{\poly(d)}$.  %

	\begin{enumerate}
		\item \textit{Feature parameter correlations:} $\forall\,{k\in[K]}$, we have $\{\w_c(0)\cdot \v_{k} \sim\N(0,\sigma_0^2)\}_{c\in[C]}$ are $C$ i.i.d Gaussian. Thus, using union bound on the Gaussian tail concentration in Lemma~\ref{lem:maxgauss} we have condition (1)  holds w.p. $\ge 1-K\delta-K\exp(-C/4)$.
		\item \textit{Noise-parameter correlation:} $\forall\,{i\in[n]}$ and $\forall\,{c\in[C]}$ using Gaussian correlation bound from Lemma~\ref{lem:gausscorr}, we have $|\w_c(0)\cdot \noise^{(i)}|\ge \tO(\sigma_0\sigma_{\noisesubscript})$ w.p. $\ge 1-nC\delta$.

		Furthermore, using the second inequality in Lemma~\ref{lem:gausscorr}, we have $\w_c(0)\cdot y^{(i)}\noise^{(i)}\ge c_2\sqrt{\sigma_0\sigma_{\noisesubscript}}$ w.p. $\ge 1/4$. Hence, $\max_{c\in[C]}\w_c(0)\cdot y^{(i)}\noise^{(i)}\ge c_2\sqrt{\sigma_0\sigma_{\noisesubscript}}$ w.p. $\ge 1-(1-1/4)^C\ge 1-\exp(C)$.

		Thus, summing over failure probabilities, we have that condition (2) holds w.p. $\ge 1-nC \delta- n\exp(-C/4)$
		\item \textit{Noise-noise correlations:} Using the $\ell_2$ norm bound from Corollary~\ref{cor:gl2norm} on $\norm*{\noise^{(i)}}_2^2$, and the  correlation tail bound on $|\noise^{(i)}\cdot\noise^{(j)}|$ for $i\neq j$ from Lemma~\ref{lem:gausscorr}, we have condition (3) holds w.p. $\ge 1-2n^2\delta$
		\item \textit{Feature noise correlation:} $\forall\,k\in[K]$, we have $\{\noise^{(i)}\cdot \v_{k} \sim\N(0,\sigma_{\noisesubscript}^2/d)\}_{i\in[n]}$ are $n$ i.i.d Gaussians. Thus, again using union bound on the Gaussian tail concentration in Lemma~\ref{lem:maxgauss} condition (4)  holds w.p. $\ge 1-n\delta$.
		\item \textit{Parameter norm:} From concentration of $\ell_2$ norm of Gaussian vector in Corollary~\ref{cor:gl2norm},   condition (5) holds w.p. $\ge1-2C\delta$  .
	\end{enumerate}

	\noindent %
	The lemma follows from using  $\delta=\frac{1}{\poly(d)}$  and $C=\Omega(\log{d})\Rightarrow\exp(-C)=O\prn[Big]{\frac{1}{\poly(d)}}$.
\end{proof}

\ginitsublemma*
\begin{proof}
	Recall that since the features $\{\v_k\}_{k}$ are orthonormal (Assumption~\ref{ass:orthonormal features}) and all the non-feature noise are spherically symmetric, without loss of generality, we assume that $\crl{\v_k}_{k\in[K]}$ are simply the first $K$ standard basis vectors in $\R^d$, \ie  $\v_k=\vec{e}_k$. In this case,  we choose $\T_k$ for $k\in[K-1]$ as a permutation of coordinates of $\bR^d$ without any fixed points, \ie $\forall\, i\in[d]$, $\T_k(\z)[i]\neq \z[i]$ that satisfies \eqref{eq:permutation} on the first $K$ coordinate. 
	\newline

	We now show that the $\Ginit$ conditions in Lemma~\ref{lem:init} holds for $\Dtraug=\Dtr\,\cup\,\T_{1}(\Dtr)\,\cup\,\T_{2}(\Dtr)\,\cup\ldots\cup\,\T_{K-1}(\Dtr)$ defined with transformations $\{\T_{k}\}_{k\in[K-1]}$ described above.
	\begin{itemize}
	\item First, among the $\Ginit$ conditions, (1) and (5) are independent of the samples and hence immediately hold.
	\item Secondly,  $\forall\,{i\in[n]}$ and $\forall\,{k\in[K]}$, $\T_{k}(\noise^{(i)})$ is simply some permutation of the coordinates of $\noise^{(i)}\sim \N(0,\sigma^2_{\noisesubscript}I_d/d)$, and hence $\T_{k}(\noise^{(i)})\sim \N(0,\sigma^2_{\noisesubscript}I_d/d)$ has the same marginal distribution as $\noise^{(i)}$. This implies that conditions (2) and (4), as well the norm condition in (3) of Lemma~\ref{lem:init} also holds for $\Dtraug$.
	\item Finally, note that $\forall\,{i\neq j}$, $\forall\,k,k'$, $\T_{k}(\noise^{(i)})$ and $\T_{k'}(\noise^{(j)})$ are independent Gaussians. Thus, the correlation bounds in (3) of the form $|\T_{k}(\noise^{(i)})\cdot \T_{k'}(\noise^{(j)})|=\tO({\sigma_{\noisesubscript}^2}/{\sqrt{d}})$ for all $i\neq j$ also follow from the proof of Lemma~\ref{lem:init}.
	\end{itemize}
	The only non-trivial condition we want to show is the following bound on the noise correlations of distinct transformations of the same sample, \ie we only need to show that $|\noise^{(i)}\cdot \T_k(\noise^{(i)})|\le \tO({\sigma_{\noisesubscript}^2}/{\sqrt{d}})$ with high probability for all $k\in[K-1]$. Note that for any $1\le k<k'\le K-1$, $\T_{k}(\noise^{(i)})\cdot \T_{k'}(\noise^{(i)})$ is equivalent in distribution to $\noise^{(i)}\cdot \T_{k'-k}(\noise^{(i)})$.

	\begin{claim} If $\noise\sim \N(0,\sigma_{\noisesubscript}^2I_d/d)$ then $\forall\,k\in[K-1]$, $|\noise\cdot \T_{k}(\noise)|\le O\left(\sigma_{\noisesubscript}^2\sqrt{\frac{\log{(1/\delta)}}{d}}\right)$ w.p. $\ge 1-\delta$.
	\end{claim}

	\begin{proof}
		At a high level, we create a non-overlapping partition of the entries of $\noise$ into three vectors $\noise^{\prime}$, $\noise^{\prime\prime}$, and $\noise^{\prime\prime\prime}$, each of which of length at least $d/6$. The partition is chosen such that same partitioning of entries of $\T_k(\noise)$ denoted as $\tilde{\noise}^{\prime}$, $\tilde{\noise}^{\prime\prime}$, and $\tilde{\noise}^{\prime\prime\prime}$ are independent of $\noise^{\prime}$, $\noise^{\prime\prime}$, and $\noise^{\prime\prime\prime}$, respectively.
		We then have $\noise\cdot\T_k(\noise)=\noise^{\prime}\cdot \tilde{\noise}^{\prime}+\noise^{\prime\prime}\cdot\tilde{\noise}^{\prime\prime}+\noise^{\prime\prime\prime}\cdot\tilde{\noise}^{\prime\prime\prime}$, where each term is a dot product of two independent spherical Gaussians of length at least $d/6$ and entrywise variance of $\sigma_{\noisesubscript}^2/d$. The claim then follows from bounding each term using  Lemma~\ref{lem:gausscorr}. \newline

		We divide the coordinates of $\vxi$ into disjoint and ordered lists $L_{1},L_{2},\ldots$, constructed
		as follows. The first list is
		\[
			L_{1}=\brk[\big]{\vxi[1],\,\T_k(\vxi)[1],\,\T_k^2(\vxi)[1],\ldots,\T_k^{s_1}(\vxi)[1]},
		\]
		where $\T^m_k$ denotes composition of $\T_k$ for $m$ times, and  we stop the list at the first $s_1\le d-1$ such that $\T^{s_1+1}_{k}(e_1)=e_1$ (when $\T^{s_1+1}_{k}(\vxi)[1]=\vxi[1]$). We claim that this stopping criteria ensures that $L_{1}$ has $s_1$ unique coordinate of $\vxi$ without any duplicates. If not, there exists some $0\le s'<s{''}\le s_1$ such that  $\T^{s{''}}_k(e_1) =\T^{s{'}}_k(e_1)$. Since $\T_k$ is a permutation (hence invertible), this would imply that $\T^{s{''}-s'}_k(e_1)=e_1$ for $s{''}-s'\le s_1$, which contradicts the stopping criteria.\newline  %

		Note that if $s_1=d-1$, we have included all the coordinates of $\vxi$ in $L_1$, and we stop our stop our construction here. If $L_{1}$ does not contain all coordinates of $\vxi$, let $1<j_{2}\le d$ be the first coordinate such that $\vxi[j_{2}]\notin L_{1}$. Let, %
		\[
			L_{2}=\brk[\big]{\vxi[j_{2}],\,\T_k(\vxi)[j_2],\,\T_k^2(\vxi)[j_2],\ldots,\T_k^{s_2}(\vxi)[j_2]},
		\]
		where we stop either when all the entries of $\vxi$ have been included in $L_1^{(m)}$ or $L_2^{(m)}$, or at the first integer $s_2$ such that $\T^{s_2+1}_{k}(e_{j_2})=e_{j_2}$ (when $\T^{s_2+1}_{k}(\vxi)[j_2]=\vxi[j_2]$). With a similar argument as with $L_1$, there are no duplicate coordinates in $L_{2}$. Furthermore, we either have have $L_2$ and $L_1$ containing disjoint coordinates of $\noise$, or have $L_1\subset L_2$. 
		To see this, suppose for $0\le s'\le s_1$ and $0\le s{''}\le s_2$, we have $\T_k^{s'}(e_{1})=\T_k^{s{''}}(e_{j_2})$. If $s'\ge s''$, again from invertibility of $\T_k$, we would have $\T_k^{s'-s{''}}(e_{1})=e_{j_2}$ for $s'-s{''}\le s_1$, which is contradiction for $\vxi[j_2]\notin L_1$. On the other hand, if $s'<s{''}$, then $\T_k^{s{''}-s{'}}(e_{j_2})=e_1$, and the entire construction of $L_1$ would also be contained in $L_2$. This would imply that all the coordinates of $L_1$ are contained in $L_2$ exactly once (since $L_2$ does not have duplicates). 		Without loss of generality, we assume the former condition that $L_2$ and $L_1$ are disjoint holds as otherwise, $L_1\subset L_2$ and we can simply drop the first list $L_1$ from our construction, and our proof follows exactly. 
		 \newline

		We construct $L_{3},L_{4},\ldots,L_{\ell}$ similarly
		until all coordinates of $\vxi$ belong to exactly one list. %
		We also define $\T L_{1},\T L_{2},\ldots,\T L_{\ell}$ as lists obtained by circularly shifting the coordinates of $L_{1}, L_{2},\ldots, L_{\ell}$, respectively, by one index. For example,
		$
			\T L_{1}=\brk[\big]{\T_k(\vxi)[1],\T_k^2(\vxi)[1],\ldots,\T_k^{s_1}(\vxi)[1],\,\vxi[1]}.\newline
		$

		By construction, for $l=1,2\ldots \ell$, for every coordinate of $\vxi$ that is included in $L_l$, has the same coordinate of $\T_k(\noise)$ is included in $\T L_{l}$ at the same position, \ie for all $i\le s_l,j\le d$,  $L_l[i]=\vxi[j]\implies\T L_l[i]= \T(\vxi)[j]$.
		We now construct $\noise^{\prime}$, $\noise^{\prime\prime}$, and $\noise^{\prime\prime\prime}$. For $l=1,2\ldots, \ell$, do the following:
		\begin{itemize}
			\item Sequentially distribute all the elements \textit{except the last element}  of $L_{l}$ to $\noise^{\prime}$, $\noise^{\prime\prime}$, $\noise^{\prime\prime\prime}$, \eg the $1^{\text{st}}$ element of $L_{l}$ goes to $\noise^{\prime}$, $2^\text{nd}$ to $\noise^{\prime\prime}$, $3^\text{rd}$ to $\noise^{\prime\prime\prime}$, $4^\text{th}$ to $\noise^{\prime}$ and so on.
			This assignment ensures that $\noise^{\prime}$, $\noise^{\prime\prime}$, $\noise^{\prime\prime\prime}$ do not contain any adjacent entries of $L_{l}$, \ie if $L_l[i]$ is in $\noise^{\prime}$, then $L_l[i+1]$ is not in $\noise^{\prime}$, and same is true for $\noise^{\prime\prime}$, and $\noise^{\prime\prime}$.
			\item Include the last element of $L_l$ to a list among $\noise^{\prime},\noise^{\prime\prime},\noise^{\prime\prime\prime}$ that \textit{does not} contain the first  or the second last element of $L_l$. %
			Thus the last element of $L_l$ is not in the same list as its circularly adjacent neighbors $\vxi[j_l]$ and $\T_k^{s_l-1}(\vxi)[j_l]$.%
			\item Repeat the exact assignment as above to distribute the elements of $\T L_{l}$ to $\tilde{\noise}^{\prime}$, $\tilde{\noise}^{(\prime\prime}$, $\tilde{\noise}^{\prime\prime\prime}$.\newline
		\end{itemize}

		\remove{\color{red}
		We consider each subpatch of $K$ consecutive coordinates of $\vxi$ separately. Let $d= MK$ and for $m\in[M]$, let $\vxi_{m}\in\bR^K$ denote the $m$-th subpatch of $\vxi$, \ie $\vxi_{m}=\vxi[Km+1:K(m+1)]$.
		We divide the coordinates of $\vxi_{m}$ into disjoint list $L_{1}^{(m)},L_{2}^{(m)},\ldots$, constructed
		as follows. The first list is
		\[
			L_{1}^{(m)}=\brk[\Big]{\vxi_{m}[1],\,\vxi_{m}[k\bmod K+1],\,\vxi_m[2k\bmod K +1],\ldots,\vxi_{m}[sk\bmod K+1]},
		\]
		where we stop the list at the first $s\le K-1$ such that $(s+1)k\bmod K=0$. If $s=K-1$, then $L^{(m)}_1$ consists of all the $K$ coordinates of $\vxi_m$.  Note that $L_{1}^{(m)}$ does not have any duplicate coordinates of $\noise_m$. If not, there exists some $0\le s'<s{''}\le s$ such that  $s'k\bmod K +1 =s{''}k\bmod K+1$. This would imply $(s{''}-s')k\bmod K =0$ for $s{''}-s'\le s$, which is a contradiction for the stopping criteria. \newline

		If $L_{1}^{(m)}$ does not contain all coordinates of $\vxi_{m}$, let $j_{2}\le K$ be the first coordinate such that $\vxi_{m}[j_{2}]\notin L_{1}^{(m)}$. Let, %
		\[
			L_{2}^{(m)}=\brk[\Big]{\vxi_{m}[j_{2}],\,\vxi_{m}[(j_{2}+k-1)\bmod K+1],\,\vxi_{m}[(j_{2}+2k-1)\bmod K+1],\ldots},
		\]
		where we stop either when all the entries of $\vxi_m$ have been included in $L_1^{(m)}$ or $L_2^{(m)}$, or after $s+1$ entries at $\vxi_{m}[(j_{2}+sk-1)\bmod K+1]$, where recall that $s$ is the first integer such that $(s+1)k\bmod K=0$ or equivalently,  $(j_{2}+
		(s+1)k-1)\bmod K+1=j_2$. With a similar argument as with $L^{(m)}_1$, there are no duplicate coordinates in $L_{2}^{(m)}$. Furthermore, $L_2^{(m)}$ and $L_1^{(m)}$ have disjoint coordinates of $\noise$.
		 \newline

		We construct $L_{3}^{(m)},L_{4}^{(m)},\ldots,L_{\ell}^{(m)}$ similarly
		until all coordinates belong to exactly one list. %
		We also define $\T L_{1}^{(m)},\T L_{2}^{(m)},\ldots,\T L_{\ell}^{(m)}$ as lists obtained by circularly shifting the coordinates of $L_{1}^{(m)}, L_{2}^{(m)},\ldots, L_{\ell}^{(m)}$, respectively, by one index. For example,
		$
			\T L_{1}^{(m)}=\brk[\big]{\vxi_{m}[k\bmod K+1],\vxi_m[2k\bmod K +1],\ldots,\vxi_{m}[sk\bmod K+1],\,\vxi_{m}[1]}.\newline
		$

		Recall the $\T_k(\noise)$ circularly shifts the coordinates of $\vxi_m$ by $k$ coordinates for all $m$. Thus, for $l=1,2,\ldots,\ell$ and $m\in[M]$, for every  coordinate of $\vxi_m$ that is included in $L_l^{(m)}$, the  same coordinate of $(\T_k(\noise))_m$ is included in $\T L_{l}^{(m)}$ at the same position.
		We now construct $\noise^{\prime}$, $\noise^{\prime\prime}$, and $\noise^{\prime\prime\prime}$. For $l=1,2\ldots, \ell$, do the following:
		\begin{itemize}
			\item For all $m$, sequencially distribute all the elements \textit{except the last element}  of $L_{l}^{(m)}$ to $\noise^{\prime}$, $\noise^{\prime\prime}$, $\noise^{\prime\prime\prime}$. For example, the $1^{\text{st}}$ element of $L_{l}^{(m)}$ for all $m$ goes to $\noise^{\prime}$, $2^\text{nd}$ to $\noise^{\prime\prime}$, $3^\text{rd}$ to $\noise^{\prime\prime\prime}$, $4^\text{th}$ to $\noise^{\prime}$, $5^\text{th}$ to $\noise^{\prime\prime}$ and so on.
			This assigment ensures that $\noise^{\prime}$, $\noise^{\prime\prime}$, $\noise^{\prime\prime\prime}$ do not contain any adjacent entries of $L_{l}^{(m)}$, \ie for all $m$, if the $i^\text{th}$ index element of $L_l^{(m)}$ is in $\noise^{\prime}$, then $(i+1)^\text{th}$  index element is not in $\noise^{\prime}$, and same is true for $\noise^{\prime\prime}$, $\noise^{\prime\prime}$.
			\item Include the last element of $L_l^{(m)}$ to a list that \textit{does not} contain the first element of $L_l^{(m)}$ or the second last element of $L_l^{(m)}$. Recall that for all $m$, the first element of $L_l^{(m)}$ is always included in $\noise^{\prime}$. If in the assignement above, the second last element of $L_l^{(m)}$ was also included in $\noise^{\prime}$, then we simply pick either one of $\noise^{\prime\prime}$ or $\noise^{\prime\prime\prime}$ to include the last element of $L_l^{(m)}$. On the other hand, if the second last element was included in $\noise^{\prime\prime}$ (or $\noise^{\prime\prime\prime}$), then the last element is assigned to $\noise^{\prime\prime\prime}$ (or $\noise^{\prime\prime}$). This ensures that for all $m$ last element of $L_l^{(m)}$ is not in the same list as its cicularly adjacent neighbours, namely the first and the second last elements of $L_l^{(m)}$. 
			\item Repeat the exact assigmenet as above to distribute the elements of $\T L_{l}^{(m)}$ for all $m$ to $\tilde{\noise}^{(0)}$, $\tilde{\noise}^{(1)}$, $\tilde{\noise}^{(2)}$.\newline
		\end{itemize}
		}

		By construction, $\{\noise^{\prime},\noise^{\prime\prime},\noise^{\prime\prime\prime}\}$ and $\{\tilde{\noise}^{\prime},\noise^{\prime\prime},\noise^{\prime\prime\prime}\}$ satisfy the  following properties:  
		\begin{inparaenum}[(a)] 
			\item $\noise\cdot\T_k(\noise)=\noise^{\prime}\cdot \tilde{\noise}^{\prime}+\noise^{\prime\prime}\cdot\tilde{\noise}^{\prime\prime}+\noise^{\prime\prime\prime}\cdot\tilde{\noise}^{\prime\prime\prime}$. 
			\item $\noise^{\prime},\noise^{\prime\prime}$, and $\noise^{\prime\prime\prime}$ are independent of $\tilde{\noise}^{\prime},\tilde{\noise}^{\prime\prime}$, and $\tilde{\noise}^{\prime\prime\prime}$, respectively.  Furthermore, each of these vectors is a spherical Gaussian with entrywise variance of $\sigma_{\noisesubscript}^2/d$.
			\item  we have included at least $d/3-1=\Theta(d)$ entries of $\vxi$ in each of $\noise^{\prime},\noise^{\prime\prime}$, and $\noise^{\prime\prime\prime}$.
		\end{inparaenum}
		The claim now follows from using Lemma~\ref{lem:gausscorr} on $\noise^{\prime}\cdot \tilde{\noise}^{\prime}$, $\noise^{\prime\prime}\cdot\tilde{\noise}^{\prime\prime}$, and $\noise^{\prime\prime\prime}\cdot\tilde{\noise}^{\prime\prime\prime}$. %
	
	\end{proof}

	\noindent The above claim completes the proof of Lemma~\ref{lem:init-sub}.
\end{proof}

\section{Linear models}
\label{app:linear}
In this section we discuss the behavior of linear models for data from our distribution $\D$ in Definition~\ref{def:data}. We consider the same patchwise convolutional model in \eqref{eq:model}, \textit{but without non-linearity}. Without loss of generality, assume $C=1$. Thus, for $\vtheta\in\R^d$, the model effectively becomes $f^{\text{linear}}(\vtheta,\x)=\vtheta\cdot\bx$, where $\bx=\sum_p\x_p$.

\paragraph{Linear models without feature noise.} In the first result stated and proved below, we assume no feature noise $\alpha_p=0$. In this case, $\bar{\x}^{(i)}=y^{(i)}\v_{k^*_i}+\noise^{(i)}$. %
Recall the notation that for $k\in[K]$, $\cal{I}_k=\crl{i\in[n]:k_i^*=k}$ and $n_k=|\cal{I}_k|$.

\begin{theorem}
	\label{thm:linear1}
	With high probability, the max $\ell_2$ margin linear model over $\Dtr=\crl{(\bar{\x}^{(i)},y^{(i)}):i\in[n]}$ is given by
	\begin{equation}\label{eq:linear_maxmgin}
		\widehat{\vtheta}_{\ell_2}= \sum_{k\in[K]}\frac{1}{1+(1+o(1))\sigma_{\noisesubscript}^2/n_k}\left(\v_k+\frac{1}{n_k}\sum_{i\in I_{k}}y^{(i)}\noise^{(i)}\right)
	\end{equation}
\end{theorem}
\begin{proof}
	Without loss of generality, assume the data points are grouped by the feature type $k_i^*$, such that $\cal{I}_1=\crl{1,2,\ldots,n_1}$, $\cal{I}_2=\crl{n_1+1,n_1+2,\ldots n_1+n_2}$, and so on. Also let $\X\in\R^{n\times d}$ denote a matrix containing $y^{(i)}\bar{\x}^{(i)}$ as rows and let $\K=\X\X^\top\in\R^{n\times n}$ denote the corresponding kernel matrix. \newline

	The $\ell_2$ max margin classifier is given by $\widehat{\vtheta}_{\ell_2}=\min_{\vtheta}\|\vtheta\|_2^2 \text{\; s.t. }\X\vtheta\ge 1$. From the optimality conditions of the max-margin problem, we know that there exists a dual variable $\vnu\in\R^n_+$, s.t. $\widehat{\vtheta}_{\ell_2}=\X^\top \vnu$. We use notation $\vnu_k\in\R^{n_k}_+$ such that $\vnu=[\vnu_1^\top,\vnu_2^\top,\ldots \vnu_K^\top]^\top$.
	We can now write the objective and constraints of the max margin problem in terms of dual variables as follows: $\|\vtheta\|_2^2 = \vnu^\top \K\vnu$ and the margin condition is $\K\vnu\ge1$.\newline

	\noindent Let us first look at structure of $\K$. Recall that $\bar{\x}^{(i)}=y^{(i)}\v_{k^*_i}+\noise^{(i)}$, where $\{\v_{k}\}_k$ are orthonormal and $\noise^{(i)}\sim\N(0,\sigma_\noisesubscript^2I_d/d)$. Using the standard concentration inequalities in Appendix~\ref{app:concentration}, the following holds with high probability. 
	\[
	\K_{ij}=y^{(i)}\bar{\x}^{(i)}\cdot y^{(j)}\bar{\x}^{(j)}=\left\{\begin{array}{ll}
		1+\sigma_{\noisesubscript}^2+\tO\prn[bigg]{\frac{\sigma^2_{\noisesubscript}}{\sqrt{d}}}&\text{if }i=j\\
		1+\tO\prn[bigg]{\frac{\sigma^2_{\noisesubscript}}{\sqrt{d}}}&\text{if }i\neq j,\,k_{i}^*=k_j^*\\
		\tO\prn[bigg]{\frac{\sigma^2_{\noisesubscript}}{\sqrt{d}}}&\text{if }i\neq j,\,k_{i}^*\neq k_j^*
	\end{array}\right..
	\]
	We can combine all the $\tO(\frac{\sigma_\xi^2}{\sqrt{d}})$ terms in $\Delta$, and write $\K=\bar{\K}+\Delta$ where $\bar{\K}_{ij}=\indicator_{k_i^*=k_j^*}+\sigma_{\noisesubscript}^2 \indicator_{i=j}$. Thus, $\bar{\K}$ is a block diagonal matrix which is dominant compared to lower order terms in $\Delta$. \newline

	\noindent Based on this block dominant structure of $\K$, for $\w=\vec{X}^\top\vec{\nu}$ and $\vnu\ge 0$, the margin on data points is given by,
	\begin{equation}
		\forall\,{i\in\cal{I}_k},\;(\K\vnu)_i=\|\vnu_k\|_1 +\sigma_{\noisesubscript}^2 \vnu_{k,i} + (\Delta \vec{\nu})_{i},
		\label{eq:linear-margin}
	\end{equation}
	and the $\ell_2$ norm is given by,
	\begin{equation}\label{eq:linear-norm}
		\|\vtheta\|_2^2 = \vec{\nu}^\top \K\vec{\nu}=\left(\sum_{k\in[K]}\|\vec{\nu}_k\|_1^2+\sigma^2_{\noisesubscript}\|\vec{\nu}_k\|_2^2\right) +\vec{\nu}^\top\Delta\nu.
	\end{equation}
	Recall that $\Delta_{ij}=\tO(\sigma^2_{\noisesubscript}/\sqrt{d})$, we have the following two possibilities of $\nu$:
	\begin{enumerate}[\text{Case} 1.]
		\item $\|\vnu\|_\infty=O(1)$: In this case $(\Delta \vec{\nu})_{i}=o(\sigma^2_\noisesubscript)$ and we have  $(\K\vnu)_i=\|\vnu_k\|_1 +\sigma_{\noisesubscript}^2 \vnu_{k,i}+o(\sigma_\xi^2)$. Thus the margin constraint requires that $\min_{k} \min_{i\in\cal{I}_k}\|\vnu_k\|_1+\sigma_{\noisesubscript}^2\vnu_{k,i}+o(\sigma_\xi^2)\ge1.$
		Furthermore, for large enough $d$, $\|\vtheta\|_2^2$ is monotonic in $\vnu_{k,i}$ (for positive $\vnu_{k,i}$). Thus the optimal $\vnu$ is given by
		\begin{equation}
			\forall\,{k\in[K]},\;\forall\,{i\in\cal{I}_{k}},\;\vnu_{k,i}=\frac{1}{n_k+(1+o(1))\sigma_{\noisesubscript}^2}. %
			\label{eq:linear-optimal}
		\end{equation}
		In this case, $\|\vtheta\|_2^2=\frac{1}{1+\sigma_{\noisesubscript}^2/n_k}(1+o(1))=O(1)$.
		\item If $\vnu=\omega(1)$, then $\|\vtheta\|_2^2=\omega(1)$  which is suboptimal compared to the above case. \newline
	\end{enumerate}

	\noindent In conclusion, we have the optimal $\vnu$ for the max-margin problem given by \eqref{eq:linear-optimal}. Thus,
	\begin{equation*}
		\begin{split}
			\widehat{\vtheta}_{\ell_2}= \X^\top \vnu&=\sum_{k\in[K]}\sum_{i\in I_{k}}\nu_{k,i}y^{(i)}\bar{\x}^{(i)}\\
			&=\sum_{k\in[K]}\sum_{i\in I_{k}}\frac{\v_{k}+y^{(i)}\noise^{(i)}}{n_k+(1+o(1))\sigma_{\noisesubscript}^2}\\
			&=\sum_{k\in[K]}\frac{1}{1+{(1+o(1))}\sigma_{\noisesubscript}^2/n_k}\left(\v_k+\frac{1}{n_k}\sum_{i\in I_{k}}y^{(i)}\noise^{(i)}\right).
		\end{split}
	\end{equation*}
	This concludes the proof of the theorem.
\end{proof}

For the above classifier, for simplicity, we look at the case when there are only two views, $k=2$. Corollary~\ref{cor:cutoff} follows from direct calculation on $\widehat{\vtheta}_{\ell_2}^\top \x$ for a sample $\x$ from our distribution. The thresholds given in Corollary~\ref{cor:cutoff} are better than the threshold we derive for our neural network.
\begin{corollary}
	\label{cor:cutoff}
	Suppose $k=2$, $\omega(1) \leq \sigma_\xi^2 \leq \sqrt{nd}$ and $n\leq d$. The $\ell_2$ max-margin linear model in \eqref{eq:linear_maxmgin} can successfully learn feature $\v_1$. To successfully learn feature $\v_2$, we need $\rho_2 \gg \frac{\sigma_\xi^2}{\sqrt {nd}}$  if $n\leq o(\sigma_\xi ^2)$ and $\rho_2 \gg \frac{\sigma_\xi^3}{n\sqrt {d}}$ otherwise.
\end{corollary}

\paragraph{Linear models with feature noise.} In the second result,  we study linear models in the presence of feature noise. We show linear models are not able to learn samples from our data distribution $\D$ while the non-linear model we study can learn $\D$.
To facilitate the proof of linear models, we make some additional
simplifications. These simplifications are not necessary for our main results.
For linear model results alone, we consider the case when the dominant noise $\vxi$
is zero, \ie  $\sigma_{\xi}=0$. Note that letting $\sigma_{\xi}>0$
can only make the classification harder. Let $\Lambda(\x)$ be the
sum of the coefficients of the feature noise if $\x$, \ie  $\Lambda(\x)=\sum_{k\in[K]}\sum_{p\in\pbp}\alpha_{p}$.
Let $\mu_{\Lambda}$ be the probability that $\Lambda(\x)>1$ for each
$(\x,y)$. We assume that the patch with the main feature
is chosen uniform randomly from $[P]$. Let $\D'$ be the distribution
satisfies the above assumptions.
\begin{thm}
	\label{thm:linear2}
	For any linear classifier $\vtheta\in\R^{d\times P}$ , we have
	\[
	\Pr_{(\x,y)\sim\D'}\left[\sign\left\langle \x,\vtheta\right\rangle \neq\sign\;y\right]>\frac{1}{P}\min\left\{ \mu_{\Lambda},1-\mu_{\Lambda}\right\} \min_{k\in[K]}\rho_{k}.
	\]
	Moreover, there exists a non-linear model $F$ in \eqref{eq:model} with weights
	$\w$, such that
	\[
	\Pr_{(\x,y)\sim\D'}\left[\sign\;F(\w,\x)\neq\sign\;y\right]=0.
	\]
\end{thm}

\begin{proof}
	Let $\Delta=\min_{p\in[P],k\in[K]}\vtheta_{(p-1)d+k-1}$ and $p^{*},k^{*}=\arg\min_{p\in[P],k\in[K]}\vtheta_{(p-1)d+k-1}$.
	If $\Delta\leq0$, for any sample with main feature $y\v_{k^{*}}$
	in patch $p^{*},$ and $\Lambda(\x)\leq1$,
	\begin{align*}
		y\left\langle \x,\vtheta\right\rangle  & \leq-\Delta+\Lambda(\x)\Delta<0.
	\end{align*}

	If $\Delta>0$, then for any sample with main feature $y\v_{k^{*}}$
	in patch $p^{*},$ with $\Lambda(\x)>1$,
	\[
	y\left\langle \x,\vtheta\right\rangle \leq\Delta-\Lambda(\x)\Delta\leq0.
	\]
	Then, for both the case that $\Delta>0$ and the case that $\Delta\leq0$,
	with probability at least $\min\left\{ \mu_{\Lambda},1-\mu_{\Lambda}\right\} \min_{k\in[K]}\rho_{k}/P$,
	$\sign\;\left\langle \x,\vtheta\right\rangle \neq\sign\;y$.

	Now, consider the non-linear model given by weights $\w_{1}=\sum_{k\in[K]}\v_{k}$
	and $\w_{c}=0$ for all $c\in[C]\backslash\left\{ 1\right\} $. For
	any datapoint $(\x,y)$ with main feature $y\v_{k^{*}}$,
	\begin{align*}
		yF(\w,\x) & =y\sum_{c\in[C]}\sum_{p\in[P]}\psi\left(\left\langle \w_{c},\x_{p}\right\rangle \right)\\
		& =\psi\left(\left\langle \w_{1},\v_{k^{*}}\right\rangle \right)-\sum_{k\in[K]}\sum_{\pbpk}\psi\left(\left\langle \w_{1},\alpha\v_{k}\right\rangle \right)\\
		& \geq\frac{1}{q}-\frac{1}{q}\alpha^{q}P\\
		& >0.
	\end{align*}

	Thus, we have $\text{sign}\;F(\w,\x)=\text{sign}\;y$ for all samples
	$(\x,y)$.
\end{proof} %

\section{Proof of the Main Results}\label{app:main_result}

\subsection{Dynamics of network weights: learning features and noise}\label{app:dynamics}
We first present a few lemmas useful for the proof of the main results. We derive the training trajectories for the dataset
without data augmentation $\Dtr$. All lemmas in this section also
hold for the dataset with data augmentation $\Dtraug$ with $n$ replaced $Kn$
and $\rho_{k}$ replaced by $\rho_{k}^{\text{(aug)}}=\frac{1}{K}$. We defer the proof of the lemmas to Appendix~\ref{app:lem_proof}.

Lemma~\ref{lem:v_bd_all} and Lemma~\ref{lem:xi_bd_all} give some rough bounds on $\left\langle \w_c(t),\v_k\right\rangle $ and $\langle \w_c(t),\vxi^{(i)}\rangle $, which are used repeatedly in the proof. 
\begin{restatable}[Rough upper and lower bound on $\left\langle \w_c(t),\v_k\right\rangle $]{lem}{vbdall}
	\label{lem:v_bd_all}Suppose $\Ginit$ holds and
	\[
	\alpha\leq o\left(\sigma_{\xi}^{\frac{1}{q}}d^{-\frac{1}{2q}}P^{-\frac{1}{q}}\left(\sigma_{0}+\eta T\rho_{k}+\eta T\sigma_{\xi}d^{-1/2}\right)^{-\frac{q-1}{q}}\right).
	\]
	For all $0\leq t'\leq t\leq T$ and $k\in[K]$, we have
	\begin{align*}
		\max_{c\in[C]}\left\langle \w_{c}(t),\v_{k}\right\rangle  & \leq\max_{c\in[C]}\left\langle \w_{c}(t'),\v_{k}\right\rangle +\eta(t-t')\tO\left(\rho_{k}+\sigma_{\xi}d^{-1/2}\right)\\
		& \leq\tO\left(\sigma_{0}+\eta T\left(\rho_{k}+\sigma_{\xi}d^{-1/2}\right)\right),
	\end{align*}
	and
	\begin{align*}
		\min_{c\in[C]}\left\langle \text{\ensuremath{\w}}_{c}(t),\v_{k}\right\rangle  & \geq\min_{c\in[C]}\left\langle \w_{c}(t'),\v_{k}\right\rangle -\eta(t-t')\tO\left(\sigma_{\xi}d^{-1/2}\right)\\
		& \geq-\tO\left(\sigma_{0}+\eta T\sigma_{\xi}d^{-1/2}\right).
	\end{align*}
\end{restatable}

\begin{restatable}[Rough lower bound on $\langle \w_c(t),\vxi^{(i)}\rangle $]{lem}{xibdall}
	\label{lem:xi_bd_all}Suppose $\Ginit$ holds and
	\[
	\alpha\leq\tO\left(\min\left\{ 1,\sigma_{\xi}^{\frac{1}{q}}d^{-\frac{1}{2q}}\right\} P^{-1/q}\left(\sigma_{0}+\eta T\left(\max_{k\in[K]}\rho_{k}+\sigma_{\xi}d^{-1/2}\right)\right)^{-(q-1)/q}\right).
	\]
	For all $0\leq t\leq t'\leq T$ and $i\in[n]$, we have
	\[
	\min_{c\in[C]}y^{(i)}\left\langle \w_{c}(t),\vxi^{(i)}\right\rangle \geq\min_{c\in[C]}y^{(i)}\left\langle \w_{c}(t'),\vxi^{(i)}\right\rangle -\eta(t-t')\tO\left(\sigma_{\xi}^{2}d^{-1/2}+\sigma_{\xi}d^{-1/2}\right).
	\]

\end{restatable}

Combining Lemma~\ref{lem:v_bd_all} and Lemma~\ref{lem:xi_bd_all}, we can show that when the time step $T$ is bounded, $\left\langle \w_c(t),\v_k\right\rangle $ and $\langle \w_c(t),\vxi^{(i)}\rangle $ are lower bounded.

\begin{restatable}[Lower bound on $\left\langle \w_c(t),\v_k\right\rangle $ and $\langle \w_c(t),\vxi^{(i)}\rangle $]{lem}{wxidecrease}
	\label{lem:wxidecrease}Suppose $\Ginit$
	holds, \\
	$n\leq o\left(\min\left\{ \sigma_{0}^{q-1}\sigma_{\xi}^{q}d^{1/2},\sigma_{0}^{q-1}\sigma_{\xi}^{q-1}d^{1/2}\right\} \right)$,
	$K\leq o\left(\min\left\{ \sigma_{0}^{q-1}\sigma_{\xi}^{-1}d^{1/2},\sigma_{0}^{q-1}d^{1/2}\right\} \right)$,
	and
	\[
\alpha\leq\tO\left(\min\left\{ 1,\sigma_{\xi}^{\frac{1}{q}}d^{-\frac{1}{2q}}\right\} P^{-1/q}\left(\sigma_{0}+\eta T\left(\max_{k\in[K]}\rho_{k}+\sigma_{\xi}d^{-1/2}\right)\right)^{-(q-1)/q}\right).
\]
	for some $T=\tthe\left(\max\left\{ n\eta^{-1}\sigma_{\xi}^{-q}\sigma_{0}^{-q+2},K\eta^{-1}\sigma_{0}^{-q+2}\right\} \right)$.
	For all $0\leq t'\leq t\leq T$, and $c\in[C]$,

	\[
	\left\langle \w_{c}(t),\v_{k}\right\rangle \geq\left\langle \w_{c}(t'),\v_{k}\right\rangle -o(\sigma_{0}),
	\]
	and for all $i\in[n]$,
	\[
	y^{(i)}\left\langle \w_{c}(t),\vxi^{(i)}\right\rangle \geq y^{(i)}\left\langle \w_{c}(t'),\vxi^{(i)}\right\rangle -o\left(\sigma_{0}\sigma_{\xi}\right).
	\]
\end{restatable}
Next, Lemma~\ref{lem:vphase1} and Lemma~\ref{lem:xi_phase2} compute  the time it takes for the model to learn the main feature $\v_k$, $k\in[K]$ and overfit the noise $\vxi^{(i)}$, $i\in[n]$.  Lemma~\ref{lem:vk_bd} and Lemma~\ref{lem:xii_bd} upper bound $\left\langle \w_c(t),\v_k\right\rangle $ and $\langle \w_c(t),\vxi^{(i)}\rangle $ for $t$ smaller than the time identified in Lemma~\ref{lem:vphase1} and Lemma~\ref{lem:xi_phase2}.
\begin{restatable}[Learning the main feature]{lem}{vphase}
	\label{lem:vphase1}Suppose $\Ginit$
	holds, $C=\Theta(\log d)$, $\sigma_{0}\sigma_{\xi}\leq o(1)$, $\sigma_{\xi}^{q}d^{-1/2}\leq o(\rho_{k})$ and 	\[
	\alpha\leq o\left(P^{-\frac{1}{q}}\min\left\{ 1,\sigma_{\xi}^{\frac{1}{q}}d^{-\frac{1}{2q}}\left(\sigma_{0}+\eta T\max_{k\in[K]}\rho_{k}+\eta T\sigma_{\xi}d^{-1/2}\right)^{-\frac{q-1}{q}},\left(\sigma_{0}+\eta T\max_{k\in[K]}\rho_{k}+\eta T\sigma_{\xi}d^{-1/2}\right)^{-1}\right\} \right),
	\]
	for some $T\geq\tom\left(\left(\eta\rho_{k}\sigma_{0}^{q-2}\right)^{-1}\right)$.
	 For any $k\in[K]$ and $0\leq t\leq T$, if
	\[
	\max_{c\in[C]}\left\langle \w_{c}(t),\v_{k}\right\rangle \leq O(C^{-1/q}),\text{\; and \; }\max_{i\in[n],c\in[C]}y^{(i)}\left\langle \w_{c}(t),\vxi^{(i)}\right\rangle \leq\tO(\sigma_{0}\sigma_{\xi}),
	\]
	then
	\[
	\max_{c\in[C]}\left\langle \w_{c}(t+1),\v_{k}\right\rangle =\max_{c\in[C]}\left\langle \w_{c}(t),\v_{k}\right\rangle +\Theta\left(\eta\rho_{k}\psi'\left(\max_{c\in[C]}\left\langle \w_{c}(t),\v_{k}\right\rangle \right)\right).
	\]
	Moreover, if $\max_{i\in[n],c\in[C]}\left\langle \w_{c}(t),\vxi^{(i)}\right\rangle \leq\tO(\sigma_{0}\sigma_{\xi})$
	for all $t\leq\tO\left(\frac{1}{\eta\rho_{k}\sigma_{0}^{q-2}}\right)$,
	there exists $T'\leq\tO\left(\frac{1}{\eta\rho_{k}\sigma_{0}^{q-2}}\right)$
	such that $\max_{c\in[C]}\left\langle \w_{c}(T'),\v_{k}\right\rangle \geq\Omega\left(C^{-1/q}\right)$.
\end{restatable}

\begin{restatable}[Overfitting the dominant noise]{lem}{xiphase}
	\label{lem:xi_phase2}Suppose $\Ginit$
	holds, $C=\Theta(\log d)$, $n\leq o\left(\min\left\{ \sigma_{0}^{q-1}\sigma_{\xi}^{q}d^{1/2},\sigma_{0}^{q-1}\sigma_{\xi}^{q-1}d^{1/2}\right\} \right)$
	and
	\[
	\alpha\leq o\left(P^{-\frac{1}{q}}\min\left\{ 1,\sigma_{\xi}^{\frac{1}{q}}d^{-\frac{1}{2q}}\left(\sigma_{0}+\eta T\max_{k\in[K]}\rho_{k}+\eta T\sigma_{\xi}d^{-1/2}\right)^{-\frac{q-1}{q}},\left(\sigma_{0}+\eta T\max_{k\in[K]}\rho_{k}+\eta T\sigma_{\xi}d^{-1/2}\right)^{-1}\right\} \right),
	\]
	for some $T\geq\tom\left(n\eta^{-1}\sigma_{\xi}^{-q}\sigma_{0}^{-q+2}\right)$.

	Let $i\in[n]$ be some sample such that for all $0\leq t\leq T$,
	$\max_{c\in[C]}\left\langle \w_{c}(t),\v_{k_{i}^{*}}\right\rangle \leq O(C^{-1/q})$.
	For any time step $0\leq t\leq T$ , if
	\[
	\max_{c\in[C]}y^{(i)}\left\langle \w_{c}(t),\vxi^{(i)}\right\rangle \leq O(C^{-1/q}),
	\]
	we have
	\[
	\max_{c\in[C]}y^{(i)}\left\langle \w_{c}(t+1),\vxi^{(i)}\right\rangle =\max_{c\in[C]}y^{(i)}\left\langle \w_{c}(t),\vxi^{(i)}\right\rangle +\frac{\eta}{n}\tthe\left(\sigma_{\xi}^{2}\psi'\left(\max_{c\in[C]}y^{(i)}\left\langle \w_{c}(t),\vxi^{(i)}\right\rangle \right)\right).
	\]
	Moreover, there exists times step $T'\leq\tO\left(n\eta^{-1}\sigma_{\xi}^{-q}\sigma_{0}^{-q+2}\right)$
	such that $\max_{c\in[C]}y^{(i)}\left\langle \w_{c}(T'),\vxi^{(i)}\right\rangle \geq\Omega\left(C^{-1/q}\right)$.
\end{restatable}

 \begin{restatable}[Upper bound on $\left\langle \w_c(t),\v_k\right\rangle $]{lem}{vkbd}
	\label{lem:vk_bd}If $\Ginit$ holds, for all $k\in[K]$ and $t\leq o\left(\frac{\sigma_{0}}{\eta\rho_{k}\sigma_{0}^{q-1}+\eta\sigma_{\xi}d^{-1/2}}\right)$,
	$$\max_{c\in[C]}\left\langle \w_{c}(t),\v_{k}\right\rangle \leq\tO\left(\sigma_{0}\right).$$
\end{restatable}

\begin{restatable}[Upper bound on $\langle \w_c(t),\vxi^{(i)}\rangle $]{lem}{xiibd}
	\label{lem:xii_bd}Suppose $\Ginit$
	holds,
	$n\leq o\left(\min\left\{ \sigma_{0}^{q-1}\sigma_{\xi}^{q}d^{1/2},\sigma_{0}^{q-1}\sigma_{\xi}^{q-1}d^{1/2}\right\} \right)$
	and
	\[
	\alpha\leq o\left(P^{-\frac{1}{q}}\min\left\{ 1,\sigma_{\xi}^{\frac{1}{q}}d^{-\frac{1}{2q}}\left(\sigma_{0}+\eta T\max_{k\in[K]}\rho_{k}+\eta T\sigma_{\xi}d^{-1/2}\right)^{-\frac{q-1}{q}}\right\} \right),
	\]
	for some $T\geq\tom\left(n\eta^{-1}\sigma_{\xi}^{-q}\sigma_{0}^{-q+2}\right)$.
	For all $t\leq o(n\eta^{-1}\sigma_{\xi}^{-q}\sigma_{0}^{-q+2})$ and
	$i\in[n]$, $\max_{c\in[C]}y^{(i)}\left\langle \w_{c}(t),\vxi^{(i)}\right\rangle \leq\tO(\sigma_{0}\sigma_{\xi}).$
\end{restatable}
Finally, Lemma~\ref{lem:w_randxi} bounds $\left\langle \w_c(t), \vxi \right\rangle$ for some noise patch $\vxi$ from the testing set. Lemma~\ref{lem:w_randxi}  is useful in proving the test accuracy.
\begin{restatable}[Bound on $\left\langle \w_c(t), \vxi \right\rangle$ for $\vxi$ from the testing set]{lem}{wrandxi}
	\label{lem:w_randxi}Let $\vxi\sim\N(0,\sigma_{\xi}^{2}I_{d})$ be
	a random noise vector independent of the dataset. Suppose $\Ginit$ holds, $C=\Theta(\log d)$,
	$n\leq o\left(\min\left\{ \sigma_{0}^{q-1}\sigma_{\xi}^{q}d^{1/2},\sigma_{0}^{q-1}\sigma_{\xi}^{q-1}d^{1/2}\right\} \right),$
	$K\leq o\left(\min\left\{ \sigma_{0}^{q-1}\sigma_{\xi}^{-1}d^{1/2},\sigma_{0}^{q-1}d^{1/2}\right\} \right)$,
	and
	\[
	\alpha\leq o\left(P^{-\frac{1}{q}}\min\left\{ 1,\sigma_{\xi}^{\frac{1}{q}}d^{-\frac{1}{2q}}\left(\sigma_{0}+\eta T\max_{k\in[K]}\rho_{k}+\eta T\sigma_{\xi}d^{-1/2}\right)^{-\frac{q-1}{q}},\left(\sigma_{0}+\eta T\max_{k\in[K]}\rho_{k}+\eta T\sigma_{\xi}d^{-1/2}\right)^{-1}\right\} \right),
	\]
	for some $T=\tthe\left(\max\left\{ n\eta^{-1}\sigma_{\xi}^{-q}\sigma_{0}^{-q+2},K\eta^{-1}\sigma_{0}^{-q+2}\right\} \right)$.
	With probability at least $1-\frac{nK}{\poly d}$, for all $c\in[C]$
	and $0\leq t\leq T$,
	\[
	\left|\left\langle \w_{c}(t),\vxi\right\rangle -\left\langle \w_{c}(0),\vxi\right\rangle \right|\leq o(\sigma_{0}\sigma_{\xi}).
	\]
\end{restatable}

\subsection{Proof of main results from Lemmas in Appendix~\ref{app:dynamics}}

We first derive some implications of Assumption $\text{\ref{assu:assume_main}}$
that we use as conditions in the lemmas in \ref{app:dynamics}.
\begin{enumerate}
	\item $nK\leq o\left(\min\left\{ \sigma_{0}^{q-1}\sigma_{\xi}^{q}d^{1/2},\sigma_{0}^{q-1}\sigma_{\xi}^{q-1}d^{1/2}\right\} \right)$
	follows from $nK\leq o(\sigma_{0}^{q-1}\sigma_{\xi}^{q-1}d^{1/2})$
	and $\sigma_{\xi}\geq\omega(1)$.
	\item $K\leq o\left(\min\left\{ \sigma_{0}^{q-1}\sigma_{\xi}^{-1}d^{1/2},\sigma_{0}^{q-1}d^{1/2}\right\} \right)$
	follows from $nK\leq o(\sigma_{0}^{q-1}\sigma_{\xi}^{q-1}d^{1/2})$,
	$\sigma_{\xi}\geq\omega(1)$ and $n\geq\omega(\sigma_{\xi}^{q})$.
	\item $\sigma_{\xi}d^{-1/2}\leq o(1)$ follows from $nK\leq o\left(\sigma_{0}^{q-1}\sigma_{\xi}^{q-1}d^{1/2}\right)$, $\sigma_{0}\sigma_{\xi}\leq o(1)$ and $n\geq \omega(\sigma_\xi^q)$.
	\item $\sigma_{\xi}^{q}K\leq o(d^{1/2})$ follows from $nK\leq o(\sigma_{0}^{q-1}\sigma_{\xi}^{q-1}d^{1/2})$,
	$\sigma_{\xi}\sigma_{0}\leq o(1)$ and $o(n)\geq\sigma_{\xi}^{q} \geq \omega(1)$.
	\item $\alpha\leq o\left(P^{-\frac{1}{q}}\sigma_{\xi}\min\left\{ d^{-1/2},\sigma_{0}\right\} \left(\sigma_{0}+\eta T\max_{k\in[K]}\rho_{k}+\eta T\sigma_{\xi}d^{-1/2}\right)^{-1}\right)$
	follows from $\sigma_{\xi}d^{-1/2}\leq o(1)$, $\sigma_{0}\leq o(1)$
	and $\eta T\geq\omega(1)$.
\end{enumerate}

Now, using Lemma~\ref{lem:v_bd_all} - \ref{lem:w_randxi}, we prove the main theorems. 

\withoutaug*
\begin{proof}%
By Lemma $\text{\ref{lem:init}}$, with probability at least $1-O\left(\frac{n^2K\log d}{\poly d}\right)$,
$\Ginit$ holds.
We first show that all $(\x^{(i)},y^{(i)})\in\Dtr$ can be classified
correctly with constant margin at some $T=\tthe\left(n\eta^{-1}\sigma_{\xi}^{-q}\sigma_{0}^{-q+2}\right)$.
We first consider the samples $i\in[n]$ such that $k_{i}^{*}=1$.
If Assumption $\text{\ref{assu:assume_main}}$ holds, $\omega(\sigma_{\xi}^{q})\leq n$,
so $\eta^{-1}\rho_{1}^{-1}\sigma_{0}^{-q+2}\leq o\left(n\eta^{-1}\sigma_{\xi}^{-q}\sigma_{0}^{-q+2}\right)$.
By Lemma $\text{\ref{lem:xii_bd}}$, $\max_{c\in[C]}y^{(i)}\left\langle \w_{c}(t),\vxi^{(i)}\right\rangle \leq\tO(\sigma_{0}\sigma_{\xi})$
for all $t\leq\tO\left(\eta^{-1}\rho_{1}\sigma_{0}^{-q+2}\right)$.
Then, by Lemma $\text{\ref{lem:vphase1}}$, there exists some $t^{*}\leq\tO\left(\eta^{-1}\rho_{1}^{-1}\sigma_{0}^{-q+2}\right)$
such that $\max_{c\in[C]}\left\langle \w_{c}(t^{*}),\v_{1}\right\rangle =\Theta\left(C^{-1/q}\right)$.
Moreover, by Lemma $\text{\ref{lem:wxidecrease}}$, at any time step
$t^{*}\leq t'\leq\tO\left(n\eta^{-1}\sigma_{\xi}^{-q}\sigma_{0}^{-q+2}\right)$,
the feature $\v_{1}$ satisfies,
\[
\max_{c\in[C]}\left\langle \w_{c}(t'),\v_{1}\right\rangle \geq\max_{c\in[C]}\left\langle \w_{c}(t^{*}),\v_{1}\right\rangle -o\left(\sigma_{0}\right)\geq\Omega\left(C^{-1/q}\right).
\]
We can further show for all $c\in[C]$ and $t'\leq\tO\left(n\eta^{-1}\sigma_{\xi}^{-q}\sigma_{0}^{-q+2}\right)$,
$\left\langle \w_{c}(t'),\v_{1}\right\rangle $ and $y^{(i)}\left\langle \w_{c}(t'),\vxi^{(i)}\right\rangle $
are lower bounded. By Lemma $\text{\ref{lem:wxidecrease}}$, when
$\Ginit$ holds,
\[
\left\langle \w_{c}(t'),\v_{1}\right\rangle \geq\left\langle \w_{c}(0),\v_{1}\right\rangle -o\left(\sigma_{0}\right)\geq-\tO(\sigma_{0}),
\]
and for all $i\in[n]$,
\[
y^{(i)}\left\langle \w_{c}(t'),\vxi^{(i)}\right\rangle \geq y^{(i)}\left\langle \w_{c}(0),\vxi^{(i)}\right\rangle -o\left(\sigma_{0}\sigma_{\xi}\right)\geq-\tO(\sigma_{0}\sigma_{\xi}).
\]
Then, there exists some $T=\tthe\left(n\eta^{-1}\sigma_{\xi}^{-q}\sigma_{0}^{-q+2}\right)$such
that for $i$ with $k_{i}^{*}=1$,
\begin{align}
y^{(i)}F(\w(T),\x^{(i)})= & ~ y^{(i)}\sum_{c\in[C]}\sum_{p\in[P]}\psi\left(\left\langle \w_{c}(T),\x_{p}^{(i)}\right\rangle \right)\nonumber \\
= &~ y^{(i)}\sum_{c\in[C]}\psi\left(\left\langle \w_{c}(T),y^{(i)}\v_{k_{i}^{*}}\right\rangle \right)+y^{(i)}\sum_{c\in[C]}\sum_{k\in[K]}\sum_{p\in\pbpki}\psi\left(\left\langle \w_{c}(T),-\api y^{(i)}\v_{k}\right\rangle \right)\nonumber \\
 & +y^{(i)}\sum_{c\in[C]}\psi\left(\left\langle \w_{c}(T),\vxi^{(i)}\right\rangle \right)\label{eq:thm1-eq1}\\
\geq & ~ \Omega\left(\frac{1}{C}\right)-C\tO(\sigma_{0}^{q})-CP\alpha^{q}\tO\left(\left(\sigma_{0}+\eta T\left(\max_{k\in[K]}\rho_{k}+\sigma_{\xi}d^{-1/2}\right)\right)^{q}\right)-C\tO(\sigma_{0}^{q}\sigma_{\xi}^{q})\nonumber \\
\geq &~ \tom\left(1\right).\nonumber
\end{align}
The third step follows from $\max_{c\in[C]}\left\langle \w_{c}(T),\v_{1}\right\rangle \geq\Omega\left(C^{-1/q}\right)$,
$\min_{c\in[C]}\left\langle \w_{c}(T),\v_{1}\right\rangle \geq-\tO(\sigma_{0})$,\\
$\min_{c\in[C]}y^{(i)}\left\langle \w_{c}(T),\vxi^{(i)}\right\rangle \geq-\tO(\sigma_{0}\sigma_{\xi})$
and Lemma $\text{\ref{lem:v_bd_all}}$. The last step follows from
the the upper bound assumption on $\alpha$, $\sigma_{0}\leq o(1)$
and $\sigma_{0}\sigma_{\xi}\leq o(1)$.

We next show that the training accuracy is perfect for all $i\in[n]$
such that $k_{i}^{*}\neq1$. By Lemma $\text{\ref{lem:vk_bd}}$ and
Assumption $\text{\ref{assu:assume_main}}$ that $\rho_{k}\leq o\left(n^{-1}\sigma_{\xi}^{q}\right)$
and $n\leq o(\sigma_{\xi}^{q-1}\sigma_{0}^{q-1}d^{1/2})$, we have
$\frac{\sigma_{0}}{\eta\rho_{k}\sigma_{0}^{q-1}+\eta\sigma_{\xi}d^{-1/2}}\geq\omega\left(n\eta^{-1}\sigma_{\xi}^{-q}\sigma_{0}^{-q+2}\right)$,
and therefore $\max_{c\in[C]}\left\langle \w_{c}(t),\v_{k}\right\rangle \leq\tO\left(\sigma_{0}\right)$
for all $0\leq t\leq\tO\left(n\eta^{-1}\sigma_{\xi}^{-q}\sigma_{0}^{-q+2}\right)$
and $k\neq1$. Then, for any $i\in[n]$ such that $k_{i}^{*}\neq1$,
by Lemma $\text{\ref{lem:xi_phase2}}$, there exists some time step
$t^{(i)}$ such that $\max_{c\in[C]}y^{(i)}\left\langle \w_{c}(t^{(i)}),\vxi^{(i)}\right\rangle \geq\Omega\left(C^{-1/q}\right)$.
Moreover, by Lemma $\text{\ref{lem:wxidecrease}}$, for all $t^{(i)}\leq t'\leq\tO\left(n\eta^{-1}\sigma_{\xi}^{-q}\sigma_{0}^{-q+2}\right)$,
$\max_{c\in[C]}y^{(i)}\left\langle \w_{c}(t'),\vxi^{(i)}\right\rangle \geq\Omega\left(C^{-1/q}\right)$.

Then, there exists some $T=\tthe\left(n\eta^{-1}\sigma_{\xi}^{-q}\sigma_{0}^{-q+2}\right)$
such that for all $(\x^{(i)},y^{(i)})\in\Dtr$ such that $k_{i}^{*}\neq1$,
\begin{align*}
y^{(i)}F(\w(T),\x^{(i)}) & \geq\Omega\left(\frac{1}{C}\right)-C\tO(\sigma_{0}^{q})-CP\alpha^{q}\tO\left(\left(\sigma_{0}+\eta T\left(\max_{k\in[K]}\rho_{k}+\sigma_{\xi}d^{-1/2}\right)\right)^{q}\right)-C\tO(\sigma_{0}^{q}\sigma_{\xi}^{q})\\
 & \geq\Omega\left(\frac{1}{C}\right).
\end{align*}
The first step follows from $\text{\eqref{eq:thm1-eq1}}$, and Lemma
$\text{\ref{lem:v_bd_all}}$. The second step follows from the upper
bound assumption on $\alpha$, $\sigma_{0}\leq o(1)$ and $\sigma_{0}\sigma_{\xi}\leq o(1)$.

Thus, at some $T=\tthe\left(n\eta^{-1}\sigma_{\xi}^{-q}\sigma_{0}^{-q+2}\right)$,
for all $i\in[n]$, we have $y^{(i)}F(\w(t),\x^{(i)})\geq\Omega\left(\frac{1}{C}\right)\geq\tom\left(1\right)$.

Next, we show that the margin is $o(1)$ at $t\leq o\left(n\eta^{-1}\sigma_{\xi}^{-q}\sigma_{0}^{-q+2}\right)$
for any $(\x^{(i)},y^{(i)})$ such that $k_{i}^{*}\neq1$. Since $t\leq o\left(\frac{\sigma_{0}}{\eta\rho_{k}\sigma_{0}^{q-1}+\eta\sigma_{\xi}d^{-1/2}}\right)$,
by Lemma $\text{\ref{lem:vk_bd}}$, $\max_{c\in[C]}\left\langle \w_{c}(t),\v_{k_{i}^{*}}\right\rangle \leq\tO\left(\sigma_{0}\right)$.
Since $t\leq o\left(n\eta^{-1}\sigma_{\xi}^{-q}\sigma_{0}^{-q+2}\right)$,
by Lemma $\text{\ref{lem:xii_bd}}$, $y^{(i)}\left\langle \w_{c}(T),\vxi^{(i)}\right\rangle \leq\tO(\sigma_{0}\sigma_{\xi})$.
Then,
\begin{align}
y^{(i)}F(\w(t),\x^{(i)})\leq &~ C\tO(\sigma_{0}^{q})+CP\alpha^{q}\tO\left(\left(\sigma_{0}+\eta T\left(\max_{k\in[K]}\rho_{k}+\sigma_{\xi}d^{-1/2}\right)\right)^{q}\right)+C\tO(\sigma_{0}^{q}\sigma_{\xi}^{q})\nonumber \\
\leq &~ o(1).\label{eq:main_waug_eq1}
\end{align}
The first step follows from $\text{\eqref{eq:thm1-eq1}}$. The second
step follows from the upper bound assumption on $\alpha$, $\sigma_{0}\leq o(1)$
and $\sigma_{0}\sigma_{\xi}\leq o(1)$. Thus, we have show that $\oT=\tthe\left(n\eta^{-1}\sigma_{\xi}^{-q}\sigma_{0}^{-q+2}\right)$.

Finally, we show that the testing accuracy is bad on the testing dataset.
For any $(\x,y)\sim\D$ with the main feature $\v_{k^{*}}$ such that
$k^{*}\neq1$ and dominant noise $\vxi$, since $\max_{c\in[C]}\left|\left\langle \w_{c}(t),\v_{k}\right\rangle \right|\leq\tO\left(\sigma_{0}\right)$
for any $t\leq\oT$, following $\text{\eqref{eq:thm1-eq1}}$,
\begin{align*}
yF(\w(t),\x)\leq & ~ C\tO(\sigma_{0}^{q})+CP\alpha^{q}\tO\left(\left(\sigma_{0}+\eta\oT\left(\max_{k\in[K]}\rho_{k}+\sigma_{\xi}d^{-1/2}\right)\right)^{q}\right)+y\sum_{c\in[C]}\psi\left(\left\langle \w_{c}(t),\vxi\right\rangle \right)\\
\leq & ~ C\tO(\sigma_{0}^{q})+C\tO\left(\sigma_{\xi}^{q}\sigma_{0}^{q}\right)\\
 & +y\sum_{c\in[C]}\psi\left(\left\langle \w_{c}(0),\vxi\right\rangle \right)+\left|y\sum_{c\in[C]}\psi\left(\left\langle \w_{c}(t),\vxi\right\rangle \right)-y\sum_{c\in[C]}\psi\left(\left\langle \w_{c}(0),\vxi\right\rangle \right)\right|\\
\leq & ~ Co(\sigma_{\xi}^{q}\sigma_{0}^{q})+y\sum_{c\in[C]}\psi\left(\left\langle \w_{c}(0),\vxi\right\rangle \right)+\left|y\sum_{c\in[C]}\psi\left(\left\langle \w_{c}(t),\vxi\right\rangle \right)-y\sum_{c\in[C]}\psi\left(\left\langle \w_{c}(0),\vxi\right\rangle \right)\right|.
\end{align*}
The second step uses the upper bound on $\alpha$. The last step follows
the assumption $\sigma_{\xi}\geq\omega(1)$. For any $c\in[C]$, by
Lemma $\text{\ref{lem:gausscorr}}$, with probability at least $1-\frac{1}{\poly d}$,
$\left|\left\langle \w_{c}(0),\vxi\right\rangle \right|\leq\tO(\sigma_{0}\sigma_{\xi})$.
Then, by Lemma \ref{lem:w_randxi}, $\left|\left\langle \w_{c}(t),\vxi\right\rangle -\left\langle \w_{c}(0),\vxi\right\rangle \right|\leq o(\sigma_{0}\sigma_{\xi})$
with probability at least $1-\frac{nK}{\poly d}$ and therefore $\left|\left\langle \w_{c}(t),\vxi\right\rangle \right|\leq\tO(\sigma_{0}\sigma_{\xi})$
and
\begin{align*}
\left|y\sum_{c\in[C]}\psi\left(\left\langle \w_{c}(t),\vxi\right\rangle \right)-y\sum_{c\in[C]}\psi\left(\left\langle \w_{c}(0),\vxi\right\rangle \right)\right| & \leq\sum_{c\in[C]}q\tO(\sigma_{0}^{q-1}\sigma_{\xi}^{q-1})\left|\left\langle \w_{c}(t),\vxi\right\rangle -\left\langle \w_{c}(0),\vxi\right\rangle \right|\\
 & \leq Co(\sigma_{\xi}^{q}\sigma_{0}^{q})
\end{align*}
For $t=0$, $\left\langle \w_{c}(0),\vxi\right\rangle \sim\N(0, \sigma_0^2 \left\Vert \xi\right\Vert ^{2})$
and $\left\{ \left\langle \w_{c}(0),\vxi\right\rangle :c\in[C]\right\} $
are independent. By Lemma~\ref{cor:gl2norm}, $\left\Vert \xi\right\Vert ^{2} = \Theta(\sigma_\xi^2 )$.
Then, by Lemma \ref{lem:guassianq}, with probability at least $\frac{1}{2}-O(\frac{1}{\sqrt{C}})$,
\[
yF(\w(t),\x)\leq Co(\sigma_{\xi}^{q}\sigma_{0}^{q})+y\sum_{c\in[C]}\psi\left(\left\langle \w_{c}(0),\vxi\right\rangle \right)<0.
\]
\end{proof}
\withaug*
\begin{proof}
By Lemma $\text{\ref{lem:init-sub}}$, $\Ginit$ holds with probability
at least $1-O\left(\frac{n^2K^3\log d}{\poly d}\right)$.

We first show that $\oTaug=\tO\left(K\eta^{-1}\sigma_{0}^{-q+2}\right)$.
For the augmented dataset, we have $\rho_{k}^{\text{(aug)}}=\frac{1}{K}$
for all $k\in[K]$ and the size of the dataset is $Kn.$ For any $k\in[K]$,
if Assumption $\text{\ref{assu:assume_main}}$ holds, $\omega(\sigma_{\xi}^{q})\leq n$,
so $\eta^{-1}\rho_{k}^{\text{(aug)}}{}^{-1}\sigma_{0}^{-q+2}\leq o\left(Kn\eta^{-1}\sigma_{\xi}^{-q}\sigma_{0}^{-q+2}\right)$.
Then, for any $i\in[Kn]$ with $k_{i}^{*}=k$, by Lemma $\text{\ref{lem:xii_bd}}$
$\max_{c\in[C]}y^{(i)}\left\langle \w_{c}(T),\vxi^{(i)}\right\rangle \leq\tO\left(\sigma_{0}\sigma_{\xi}\right)$
for all $0\leq t\leq\tO\left(K\eta^{-1}\sigma_{0}^{-q+2}\right)$.
Then, under the assumption $\sigma_{\xi}^{q}K\leq o\left(d^{1/2}\right)$,
by Lemma $\text{\ref{lem:vphase1}}$, there exists some $t_{k}=\tthe\left(\frac{1}{\eta\rho_{k}^{\text{(aug)}}\sigma_{0}^{q-2}}\right)$
such that $\max_{c\in[C]}\left\langle \w_{c}(t_{k}),\v_{k}\right\rangle \geq\Omega\left(C^{-1/q}\right)$.
By Lemma $\text{\ref{lem:wxidecrease}}$, for any $t_{k}\leq t'\leq\tthe\left(K\eta^{-1}\sigma_{0}^{-q+2}\right)$,
$\max_{c\in[C]}\left\langle \w_{c}(t'),\v_{k}\right\rangle \geq\Omega\left(C^{-1/q}\right)$.
Then, there exists some $T=\tthe\left(K\eta^{-1}\sigma_{0}^{-q+2}\right)$
such that for all $(\x^{(i)},y^{(i)})\in\Dtraug$,
\begin{align}
y^{(i)}F(\w(T),\x^{(i)})= & ~ y^{(i)}\sum_{c\in[C]}\sum_{p\in[P]}\psi\left(\left\langle \w_{c}(T),\x_{p}^{(i)}\right\rangle \right)\nonumber \\
= &  ~ y^{(i)}\sum_{c\in[C]}\psi\left(\left\langle \w_{c}(T),y^{(i)}\v_{k_{i}^{*}}\right\rangle \right)+y^{(i)}\sum_{c\in[C]}\sum_{k\in[K]}\sum_{p\in\pbpki}\psi\left(\left\langle \w_{c}(T),-\api y^{(i)}\v_{k}\right\rangle \right)\nonumber \\
 & +y^{(i)}\sum_{c\in[C]}\psi\left(\left\langle \w_{c}(T),\vxi^{(i)}\right\rangle \right)\label{eq:thm2-eq1}\\
\geq & ~ \Omega\left(\frac{1}{C}\right)-C\tO(\sigma_{0}^{q})-CP\alpha^{q}\tO\left(\left(\sigma_{0}+\eta T\left(\max_{k\in[K]}\rho_{k}+\sigma_{\xi}d^{-1/2}\right)\right)^{q}\right)-C\tO(\sigma_{0}^{q}\sigma_{\xi}^{q})\nonumber \\
\geq & ~ \tom\left(1\right).\nonumber
\end{align}
The third step follows from $\max_{c\in[C]}\left\langle \w_{c}(T),\v_{k}\right\rangle \geq\Omega\left(C^{-1/q}\right)$,
Lemma $\text{\ref{lem:v_bd_all}}$ and Lemma $\text{\ref{lem:wxidecrease}}$.
The last step follows from the upper bound assumption on $\alpha$,
$\sigma_{0}\leq o(1)$ and $\sigma_{0}\sigma_{\xi}\leq o(1)$.

Next, when $t=o\left(\frac{1}{\eta\rho_{k}^{\text{(aug)}}\sigma_{0}^{q-2}}\right)$,
by Lemma $\text{\ref{lem:vk_bd}}$ , $\left\langle \w_{c}(t),\v_{k_{i}^{*}}\right\rangle \leq\tO(\sigma_{0})$.
By Lemma $\text{\ref{lem:xii_bd}}$, $y^{(i)}\left\langle \w_{c}(T),\vxi^{(i)}\right\rangle \leq\tO(\sigma_{0}\sigma_{\xi})$.
Then,
\begin{align*}
y^{(i)}F(\w(t),\x^{(i)})\leq & C\tO(\sigma_{0}^{q})+CP\alpha^{q}\tO\left(\left(\sigma_{0}+\eta T\left(\max_{k\in[K]}\rho_{k}+\sigma_{\xi}d^{-1/2}\right)\right)^{q}\right)+C\tO(\sigma_{0}^{q}\sigma_{\xi}^{q})\\
\leq & o(1).
\end{align*}
The second step follows from the upper bound assumption on $\alpha$,
$\sigma_{0}\leq o(1)$ and $\sigma_{0}\sigma_{\xi}\leq o(1)$. Thus,
we have shown that $T^{\text{(aug)}}=\tthe\left(K\eta^{-1}\sigma_{0}^{-q+2}\right)$.

Finally, we show that the testing accuracy is perfect at $T^{\text{(aug)}}=\tthe\left(K\eta^{-1}\sigma_{0}^{-q+2}\right)$.
For any sample $(\x,y)$ in the testing set with dominant noise $\vxi$,
if $\Ginit$ hold, by $\text{\eqref{eq:thm2-eq1}},$
\begin{align*}
yF(\w(T^{\text{(aug)}}),\x)\geq & ~ \Omega\left(\frac{1}{C}\right)-C\tO(\sigma_{0}^{q})-CP\alpha^{q}\tO\left(\left(\sigma_{0}+\eta T^{\text{(aug)}}\left(\max_{k\in[K]}\rho_{k}+\sigma_{\xi}d^{-1/2}\right)\right)^{q}\right)\\
 & +y\sum_{c\in[C]}\psi\left(\left\langle \w_{c}(T^{\text{(aug)}}),\vxi\right\rangle \right)\\
\geq & ~ \Omega\left(\frac{1}{C}\right)+y\sum_{c\in[C]}\psi\left(\left\langle \w_{c}(0),\vxi\right\rangle \right)-\left|y\sum_{c\in[C]}\psi\left(\left\langle \w_{c}(T^{\text{(aug)}}),\vxi\right\rangle \right)-y\sum_{c\in[C]}\psi\left(\left\langle \w_{c}(0),\vxi\right\rangle \right)\right|.
\end{align*}
For any $c\in[C]$, by Lemma $\text{\ref{lem:gausscorr}}$, with probability
at least $1-\frac{1}{\poly d}$, $\left|\left\langle \w_{c}(0),\vxi\right\rangle \right|\leq\tO(\sigma_{0}\sigma_{\xi})$.
Then, by Lemma \ref{lem:w_randxi}, $\left|\left\langle \w_{c}(T^{\text{(aug)}}),\vxi\right\rangle -\left\langle \w_{c}(0),\vxi\right\rangle \right|\leq o(\sigma_{0}\sigma_{\xi})$
with probability at least $1-\frac{nK}{\poly d}$ and therefore $\left|\left\langle \w_{c}(T^{\text{(aug)}}),\vxi\right\rangle \right|\leq\tO(\sigma_{0}\sigma_{\xi})$
and
\begin{align*}
\left|y\sum_{c\in[C]}\psi\left(\left\langle \w_{c}(T^{\text{(aug)}}),\vxi\right\rangle \right)-y\sum_{c\in[C]}\psi\left(\left\langle \w_{c}(0),\vxi\right\rangle \right)\right| & \leq\sum_{c\in[C]}q\tO(\sigma_{0}^{q-1}\sigma_{\xi}^{q-1})\left|\left\langle \w_{c}(T^{\text{(aug)}}),\vxi\right\rangle -\left\langle \w_{c}(0),\vxi\right\rangle \right|\\
 & \leq Co(\sigma_{\xi}^{q}\sigma_{0}^{q}).
\end{align*}
Thus, with probability at least $1-\frac{nK}{\poly d}$, $yF(\w(T^{\text{(aug)}}),\x)\geq\tom(1)$.
\end{proof}

\section{Deferred Proof of Lemmas in Appendix~\ref{app:main_result}}
\label{app:lem_proof}
In this section, we present the proof of lemmas necessary for proving
our main result. 

\vbdall*
\begin{proof}
For any $k\in[K]$, $c\in[C]$ and $0\leq t<T$,
\begin{align}
 & \left\langle \w_{c}(t+1),\v_{k}\right\rangle \nonumber \\
= & \left\langle \w_{c}(t),\v_{k}\right\rangle +\frac{\eta}{n}\sum_{i:k_{i}^{*}=k}\left(\frac{1}{1+e^{y^{(i)}F(\w(t),\x^{(i)})}}\psi'\left(\left\langle \w_{c}(t),\v_{k}\right\rangle \right)\left\Vert \v_{k}\right\Vert _{2}^{2}\right)\nonumber \\
 & -\frac{\eta}{n}\sum_{i=1}^{n}\left(\frac{1}{1+e^{y^{(i)}F(\w(t),\x^{(i)})}}\sum_{p\in\pbpki}\api\psi'\left(\left\langle \w_{c}(t),\api\v_{k}\right\rangle \right)\left\Vert \v_{k}\right\Vert _{2}^{2}\right)\nonumber \\
 & +\frac{\eta}{n}\sum_{i=1}^{n}\left(\frac{1}{1+e^{y^{(i)}F(\w(t),\x^{(i)})}}\psi'\left(\left\langle \w_{c}(t),\vxi^{(i)}\right\rangle \right)y^{(i)}\left\langle \vxi^{(i)},\v_{k}\right\rangle \right)\label{eq:v_bd_all_eq1}
\end{align}
We bound each term separately. Since $\frac{1}{1+e^{y^{(i)}F(\w(t),\x^{(i)})}}\leq1$
for all $i\in[n]$, $\left\Vert \v_{k}\right\Vert _{2}^{2}=1$, and
$\psi'\left(\left\langle \w_{c}(t),\v_{k}\right\rangle \right)\leq1$
for all $k\in[K]$, we have
\[
\frac{\eta}{n}\sum_{i:k_{i}^{*}=k}\left(\frac{1}{1+e^{y^{(i)}F(\w(t),\x^{(i)})}}\psi'\left(\left\langle \w_{c}(t),\v_{k}\right\rangle \right)\left\Vert \v_{k}\right\Vert _{2}^{2}\right)\leq O(\eta\rho_{k}).
\]
The feature noise term
\[
-\frac{\eta}{n}\sum_{i=1}^{n}\left(\frac{1}{1+e^{y^{(i)}F(\w(t),\x^{(i)})}}\sum_{p\in\pbpki}\api\psi'\left(\left\langle \w_{c}(t),\api\v_{k}\right\rangle \right)\left\Vert \v_{k}\right\Vert _{2}^{2}\right)\leq0.
\]
When $\Ginit$ holds, $\left\langle \vxi^{(i)},\v_{k}\right\rangle \leq\tO(\sigma_{\xi}d^{-1/2})$
for all $i\in[n]$. Since $\frac{1}{1+e^{y^{(i)}F(\w(t),\x^{(i)})}}\leq1$
and $\psi'\left(\left\langle \w_{c}(t),\vxi^{(i)}\right\rangle \right)\leq1$, 

\[
\frac{\eta}{n}\sum_{i=1}^{n}\left(\frac{1}{1+e^{y^{(i)}F(\w(t),\x^{(i)})}}\psi'\left(\left\langle \w_{c}(t),\vxi^{(i)}\right\rangle \right)y^{(i)}\left\langle \vxi^{(i)},\v_{k}\right\rangle \right)\leq\tO\left(\eta\sigma_{\xi}d^{-1/2}\right).
\]
Then, for all $0\leq t<T$,
\begin{align*}
\left\langle \w_{c}(t+1),\v_{k}\right\rangle  & \leq\left\langle \w_{c}(t),\v_{k}\right\rangle +\eta\tO\left(\rho_{k}+\sigma_{\xi}d^{-1/2}\right),
\end{align*}
which implies for any $0\leq t'\leq t\leq T$,
\[
\left\langle \w_{c}(t),\v_{k}\right\rangle \leq\left\langle \w_{c}(t'),\v_{k}\right\rangle +\eta(t-t')\tO\left(\rho_{k}+\sigma_{\xi}d^{-1/2}\right).
\]

Next, we lower bound $\left\langle \w_{c}(t),\v_{k}\right\rangle $
using induction. When $\Ginit$ holds, $\left\langle \w_{c}(0),\v_{k}\right\rangle \geq-\tO\left(\sigma_{0}+\eta T\sigma_{\xi}d^{-1/2}\right)$. Assume for all $0\leq t'\leq t$, 	\begin{align*}
	\min_{c\in[C]}\left\langle \text{\ensuremath{\w}}_{c}(t),\v_{k}\right\rangle  & \geq\min_{c\in[C]}\left\langle \w_{c}(t'),\v_{k}\right\rangle -\eta(t-t')\tO\left(\sigma_{\xi}d^{-1/2}\right)\\
	& \geq-\tO\left(\sigma_{0}+\eta T\sigma_{\xi}d^{-1/2}\right)
\end{align*}
for induction. We have 
\begin{align*}
\frac{\eta}{n}\sum_{i:k_{i}^{*}=k}\left(\frac{1}{1+e^{y^{(i)}F(\w(t),\x^{(i)})}}\psi'\left(\left\langle \w_{c}(t),\v_{k}\right\rangle \right)\left\Vert \v_{k}\right\Vert _{2}^{2}\right) & \geq0.
\end{align*}
We have shown that $\left\langle \w_{c}(t),\v_{k}\right\rangle \leq\tO\left(\sigma_{0}+\eta T\rho_{k}+\eta T\sigma_{\xi}d^{-1/2}\right)$
for all $c\in[C]$ and $k\in[K]$. By the induction hypothesis, 
\begin{align*}
 & \frac{\eta}{n}\sum_{i=1}^{n}\left(\frac{1}{1+e^{y^{(i)}F(\w(t),\x^{(i)})}}\sum_{p\in\pbpki}\api\psi'\left(\left\langle \w_{c}(t),\api\v_{k}\right\rangle \right)\left\Vert \v_{k}\right\Vert _{2}^{2}\right)\\
\leq & ~ \eta\alpha^{q}P\tO\left(\left(\sigma_{0}+\eta T\rho_{k}+\eta T\sigma_{\xi}d^{-1/2}\right)^{q-1}\right)\\
\leq & ~ \tO\left(\eta\sigma_{\xi}d^{-1/2}\right).
\end{align*}
The last inequality follows from $\alpha\leq\tO\left(\sigma_{\xi}^{\frac{1}{q}}d^{-\frac{1}{2q}}P^{-\frac{1}{q}}/\left(\sigma_{0}+\eta T\rho_{k}+\eta T\sigma_{\xi}d^{-1/2}\right)^{\frac{q-1}{q}}\right)$.
When $\Ginit$ holds,
\[
-\frac{\eta}{n}\sum_{i=1}^{n}\left(\frac{1}{1+e^{y^{(i)}F(\w(t),\x^{(i)})}}\psi'\left(\left\langle \w_{c}(t),\vxi^{(i)}\right\rangle \right)y^{(i)}\left\langle \vxi^{(i)},\v_{k}\right\rangle \right)\leq\tO\left(\eta\sigma_{\xi}d^{-1/2}\right).
\]
Then, plugging into $\text{\eqref{eq:v_bd_all_eq1}},$ for any $0\leq t'\leq t+1\leq T,$
\[
-\left\langle \w_{c}(t+1),\v_{k}\right\rangle \leq-\left\langle \w_{c}(t'),\v_{k}\right\rangle +\eta(t+1-t')\tO\left(\sigma_{\xi}d^{-1/2}\right).
\]
Thus, we have completed the induction and therefore
\[
\min_{c\in[C]}\left\langle \w_{c}(t),\v_{k}\right\rangle \geq\min_{c\in[C]}\left\langle \w_{c}(t'),\v_{k}\right\rangle -\eta(t-t')\tO\left(\sigma_{\xi}d^{-1/2}\right).
\]

Finally, for $t'=0$, when $\Ginit$ holds, $\left|\left\langle \w_{c}(0),\v_{k}\right\rangle \right|\leq\tO(\sigma_{0}).$
\end{proof}

\xibdall*

\begin{proof}
For any $c\in[C]$ and $i\in[n]$, we have 
\begin{align}
 & y^{(i)}\left\langle \w_{c}(t+1),\vxi^{(i)}\right\rangle \nonumber \\
= & ~ y^{(i)}\left\langle \w_{c}(t),\vxi^{(i)}\right\rangle +\frac{\eta}{n}\frac{1}{1+e^{y^{(i)}F(\w(t),\x^{(i)})}}\psi'\left(\left\langle \w_{c}(t),\vxi^{(i)}\right\rangle \right)\left\Vert \vxi^{(i)}\right\Vert _{2}^{2}\nonumber \\
 & +\frac{\eta}{n}\sum_{j:j\neq i}\left(\frac{y^{(i)}y^{(j)}}{1+e^{y^{(j)}F(\w(t),\x^{(j)})}}\psi'\left(\left\langle \w_{c}(t),\vxi^{(j)}\right\rangle \right)\left\langle \vxi^{(j)},\vxi^{(i)}\right\rangle \right)\label{eq:wxigbd_eq1}\\
 & +\frac{\eta}{n}\sum_{j=1}^{n}\left(\frac{y^{(i)}}{1+e^{y^{(j)}F(\w(t),\x^{(j)})}}\psi'\left(\left\langle \w_{c}(t),\v_{k_{j}^{*}}\right\rangle \right)\left\langle \v_{k_{j}^{*}},\vxi^{(i)}\right\rangle \right)\label{eq:wxigbd_eq2}\\
 & -\frac{\eta}{n}\sum_{j=1}^{n}\left(\frac{y^{(i)}}{1+e^{y^{(j)}F(\w(t),\x^{(j)})}}\sum_{k\in[K]}\sum_{p\in\pbpkj}\psi'\left(\left\langle \w_{c}(t),\apj\v_{k}\right\rangle \right)\left\langle \apj\v_{k},\vxi^{(i)}\right\rangle \right)\label{eq:wxigbd_eq3}
\end{align}
We have $\frac{\eta}{n}\left(\frac{1}{1+e^{y^{(i)}F(\w(t),\x^{(i)})}}\psi'\left(\left\langle \w_{c}(t),\vxi^{(i)}\right\rangle \right)\left\Vert \vxi^{(i)}\right\Vert _{2}^{2}\right)$
positive for any $i\in[n]$. Since $\frac{1}{1+e^{y^{(j)}F(\w(t),\x^{(j)})}}\leq1$
and $\psi'\left(\left\langle \w_{c}(t),\vxi^{(j)}\right\rangle \right)\leq1$
for all $j\in[n]$ , if $\Ginit$ holds,

\[
\text{\eqref{eq:wxigbd_eq1}}\geq-\eta\tO\left(\sigma_{\xi}^{2}d^{-1/2}\right),\text{\; \text{\eqref{eq:wxigbd_eq2}}\ensuremath{\geq-}\ensuremath{\eta\tO\left(\sigma_{\xi}d^{-1/2}\right)}.}
\]

Also,
\begin{align*}
\text{\eqref{eq:wxigbd_eq3}} & \geq-\eta\tO\left(\alpha^{q}P\sigma_{\xi}d^{-1/2}\max_{k\in[K]}\left|\left\langle \w_{c}(t),\v_{k}\right\rangle \right|^{q-1}\right)\\
 & \geq-\eta\tO\left(\alpha^{q}P\sigma_{\xi}d^{-1/2}\left(\sigma_{0}+\eta T\left(\rho_{k}+\sigma_{\xi}d^{-1/2}\right)\right)^{q-1}\right)\\
 & \geq-\eta\tO\left(\sigma_{\xi}d^{-1/2}\right)
\end{align*}
The second inequality follows from Lemma $\text{\ref{lem:v_bd_all}}$.
The third inequality follows from the upper bound on $\alpha$. Then, 

\[
y^{(i)}\left\langle \w_{c}(t+1),\vxi^{(i)}\right\rangle -y^{(i)}\left\langle \w_{c}(t),\vxi^{(i)}\right\rangle \geq-\eta\tO\left(\sigma_{\xi}^{2}d^{-1/2}+\sigma_{\xi}d^{-1/2}\right),
\]
which gives 
\[
\min_{c\in[C]}y^{(i)}\left\langle \w_{c}(t),\vxi^{(i)}\right\rangle \geq\min_{c\in[C]}y^{(i)}\left\langle \w_{c}(t'),\vxi^{(i)}\right\rangle -\eta(t-t')\tO\left(\sigma_{\xi}^{2}d^{-1/2}+\sigma_{\xi}d^{-1/2}\right).
\]
\end{proof}

\wxidecrease*
\begin{proof}
	By Lemma $\text{\ref{lem:xi_bd_all}}$, for any $(\x^{(i)},y^{(i)})$,
	
	\[
	\max_{c\in[C]}y^{(i)}\left\langle \w_{c}(t),\vxi^{(i)}\right\rangle \geq\max_{c\in[C]}y^{(i)}\left\langle \w_{c}(t'),\vxi^{(i)}\right\rangle -\eta(t-t')\tO\left(\sigma_{\xi}^{2}d^{-1/2}+\sigma_{\xi}d^{-1/2}\right).
	\]
		Then, when $t-t'\leq\tO\left(n\eta^{-1}\sigma_{\xi}^{-q}\sigma_{0}^{-q+2}\right)$
	and $n\leq o\left(\min\left\{ \sigma_{0}^{q-1}\sigma_{\xi}^{q}d^{1/2},\sigma_{0}^{q-1}\sigma_{\xi}^{q-1}d^{1/2}\right\} \right)$,
	or when $t-t'\leq\tO\left(K\eta^{-1}\sigma_{0}^{-q+2}\right)$ and
	$K\leq o\left(\min\left\{ \sigma_{0}^{q-1}\sigma_{\xi}^{-1}d^{1/2},\sigma_{0}^{q-1}d^{1/2}\right\} \right)$,
	$\eta(t-t')\tO\left(\sigma_{\xi}^{2}d^{-1/2}+\sigma_{\xi}d^{-1/2}\right)\leq o\left(\sigma_{0}\sigma_{\xi}\right)$.
	
	By Lemma $\text{\ref{lem:v_bd_all}}$,
	
	\[
	\max_{c\in[C]}\left\langle \w_{c}(t),\v_{k}\right\rangle \geq\max_{c\in[C]}\left\langle \w_{c}(t'),\v_{k}\right\rangle -\eta(t-t')\tO\left(\sigma_{\xi}d^{-1/2}\right).
	\]
		Then, when $t-t'\leq\tO\left(n\eta^{-1}\sigma_{\xi}^{-q}\sigma_{0}^{-q+2}\right)$
	and $n\leq o\left(\sigma_{0}^{q-1}\sigma_{\xi}^{q-1}d^{1/2}\right)$,
	or when $t-t'\leq\tO\left(K\eta^{-1}\sigma_{0}^{-q+2}\right)$ and
	$K\leq o(\sigma_{0}^{q-1}\sigma_{\xi}^{-1}d^{1/2})$, 
	\[
	\eta(t-t')\tO\left(\sigma_{\xi}d^{-1/2}\right)\leq o\left(\sigma_{0}\right),
	\]
	which completes the proof.
\end{proof}

\vphase*
\begin{proof}
By the upper bound on $\alpha$ and Lemma $\text{\ref{lem:v_bd_all}}$,
for any $i\in[n]$ and $c\in[C]$, 
\begin{align*}
\sum_{k'\in[K]}\sum_{p\in\pbpkip}\psi\left(\left\langle \w_{c}(t),-y^{(i)}\api\v_{k'}\right\rangle \right) & \leq\sum_{k'\in[K]}\sum_{p\in\pbpki}\left|\left\langle \w_{c}(t),\api\v_{k'}\right\rangle \right|^{q}\\
 & \leq\tO\left(\alpha^{q}P\left(\sigma_{0}+\eta T\left(\max_{k'}\rho_{k'}+\sigma_{\xi}d^{-1/2}\right)\right)^{q}\right)\\
 & \leq o(1).
\end{align*}
Then, since $\max_{c\in[C]}\left\langle \w_{c}(t),\v_{k}\right\rangle \leq O(C^{-1/q})$,
and $\max_{c\in[C]}y\left\langle \w_{c}(t),\vxi^{(i)}\right\rangle \leq o(1)$
for all $i\in[n]$, we have for all $i$ such that $k_{i}^{*}=k$,
$y^{(i)}F(\w(t),\x^{(i)})\leq O(1)$ and $\frac{1}{1+e^{y^{(i)}F(\w(t),\x^{(i)})}}\geq\Omega(1)$. 

Now, we compute the update $\left\langle \w_{c}(t+1),\v_{k}\right\rangle -\left\langle \w_{c}(t),\v_{k}\right\rangle $,
\begin{align}
 & \left\langle \w_{c}(t+1),\v_{k}\right\rangle \nonumber \\
= &~ \left\langle \w_{c}(t),\v_{k}\right\rangle +\frac{\eta}{n}\sum_{i:k_{i}^{*}=k}\left(\frac{1}{1+e^{y^{(i)}F(\w(t),\x^{(i)})}}\psi'\left(\left\langle \w_{c}(t),\v_{k}\right\rangle \right)\left\Vert \v_{k}\right\Vert _{2}^{2}\right)\nonumber \\
 & -\frac{\eta}{n}\sum_{i=1}^{n}\left(\frac{1}{1+e^{y^{(i)}F(\w(t),\x^{(i)})}}\sum_{p\in\pbpki}\api\psi'\left(\left\langle \w_{c}(t),\api\v_{k}\right\rangle \right)\left\Vert \v_{k}\right\Vert _{2}^{2}\right)\label{eq:vphase1eq1}\\
 & +\frac{\eta}{n}\sum_{i=1}^{n}\left(\frac{1}{1+e^{y^{(i)}F(\w(t),\x^{(i)})}}\psi'\left(\left\langle \w_{c}^{(t)},\vxi^{(i)}\right\rangle \right)y^{(i)}\left\langle \vxi^{(i)},\v_{k}\right\rangle \right)\label{eq:vphase1eq2}
\end{align}

Then, when $\Ginit$ holds, since $\frac{1}{1+e^{y^{(i)}F(\w(t),\x^{(i)})}}\geq\Omega(1)$
for all $i\in[n]$ such that $k_{i}^{*}=k$,
\begin{align*}
\frac{\eta}{n}\sum_{i:k_{i}^{*}=k}\left(\frac{1}{1+e^{y^{(i)}F(\w(t),\x^{(i)})}}\psi'\left(\left\langle \w_{c}(t),\v_{k}\right\rangle \right)\left\Vert \v_{k}\right\Vert _{2}^{2}\right) & =\Theta\left(\eta\rho_{k}\left|\left\langle \w_{c}(t),\v_{k}\right\rangle \right|^{q-1}\right).
\end{align*}

We can bound the term $\text{\eqref{eq:vphase1eq1}}$ as
\begin{align*}
\left|\eqref{eq:vphase1eq1}\right|\leq & ~\tO\left(\frac{\eta}{N}\sum_{i=1}^{n}\sum_{p\in\pbpki}\api\psi'\left(\left\langle \w_{c}(t),\api\v_{k}\right\rangle \right)\right)\\
\leq & ~\tO\left(\eta\rho_{k}\alpha^{q}P\left|\left\langle \w_{c}(t),\v_{k}\right\rangle \right|^{q-1}\right)\leq o(\eta\rho_{k}\left|\left\langle \w_{c}(t),\v_{k}\right\rangle \right|^{q-1}).
\end{align*}

For the term $\text{\eqref{eq:vphase1eq2}}$, if $\Ginit$ holds,
\begin{align*}
\text{\ensuremath{\left|\eqref{eq:vphase1eq2}\right|}}\leq &~ \frac{\eta}{n}\sum_{i=1}^{n}\left(\frac{1}{1+e^{y^{(i)}F(\w(t),\x^{(i)})}}\psi'\left(\left\langle \w_{c}^{(t)},\vxi^{(i)}\right\rangle \right)\left|\left\langle \vxi^{(i)},\v_{k}\right\rangle \right|\right)\\
\leq &~ \tO\left(\frac{\eta}{n}\sum_{i=1}^{n}\left|\left\langle \w_{c}^{(t)},\vxi^{(i)}\right\rangle \right|^{q-1}\left\langle \vxi^{(i)},\v_{k}\right\rangle \right)\\
\leq & ~\tO\left(\eta\sigma_{0}^{q-1}\sigma_{\xi}^{q}d^{-1/2}\right).
\end{align*}

For $t=0$, if $\Ginit$ holds, $\max_{c\in[C]}\left\langle \w_{c}(0),\v_{k}\right\rangle \geq\tom(\sigma_{0})$.
Then, if $\sigma_{0}^{q-1}\sigma_{\xi}^{q}d^{-1/2}\leq o(\rho_{k}\sigma_{0}^{q-1})$,
we have 
\begin{equation}
\max_{c\in[C]}\left\langle \w_{c}(t+1),\v_{k}\right\rangle =\max_{c\in[C]}\left\langle \w_{c}(t),\v_{k}\right\rangle +\Theta\left(\eta\rho_{k}\psi'\left(\max_{c\in[C]}\left\langle \w_{c}(t),\v_{k}\right\rangle \right)\right),\label{eq:vphase1eq3}
\end{equation}
 which shows $\max_{c\in[C]}\left\langle \w_{c}(t),\v_{k}\right\rangle $
is increasing. Then, $\text{\eqref{eq:vphase1eq3}}$ holds for all
$t$.

Starting from some $\left\langle \w_{c}(t'),\v_{k}\right\rangle $,
the number of iterations it takes to reach $\max_{c\in[C]}\left\langle \w_{c}(t),\v_{k}\right\rangle \geq2\max_{c\in[C]}\left\langle \w_{c}(t'),\v_{k}\right\rangle $
is at most $O\left(\frac{\max_{c\in[C]}\left\langle \w_{c}(t'),\v_{k}\right\rangle }{\eta\rho_{k}\left(\max_{c\in[C]}\left\langle \w_{c}(t'),\v_{k}\right\rangle \right)^{q-1}}\right).$
Then, starting from $\Theta(\sigma_{0})$, it takes at most
\[
\tO\left(\sum_{i=0}^{\infty}\frac{2^{i}\sigma_{0}}{\eta\rho_{k}(2^{i}\sigma_{0})^{q-1}}\right)\leq\tilde{O}\left(\frac{1}{\eta\rho_{k}\sigma_{0}^{q-2}}\right)
\]
time steps to reach $\max_{c\in[C]}\left\langle \w_{c}(t),\v_{k}\right\rangle \geq\Omega\left(C^{-1/q}\right)$. 
\end{proof}

\xiphase*
\begin{proof}
By the upper bound on $\alpha$ and Lemma $\text{\ref{lem:v_bd_all}}$,
for any $c\in[C],$
\begin{align*}
\left|\sum_{k'\in[K]}\sum_{p\in\pbpkip}\psi\left(\left\langle \w_{c}(t),\api\v_{k'}\right\rangle \right)\right| & \leq\sum_{k'\in[K]}\sum_{p\in\pbpki}\left|\left\langle \w_{c}(t),\api\v_{k'}\right\rangle \right|^{q}\\
 & \leq\tO\left(\alpha^{q}P\left(\sigma_{0}+\eta T\left(\max_{k'}\rho_{k'}+\sigma_{\xi}d^{-1/2}\right)\right)^{q}\right)\\
 & \leq o(1).
\end{align*}
For $i$, when $\max_{c\in[C]}\left\langle \w_{c}(t),v_{k_{i}^{*}}\right\rangle \leq O(C^{-1/q})$,
$\max_{c\in[C]}y^{(i)}\left\langle \w_{c}(t),\vxi^{(i)}\right\rangle \leq O(C^{-1/q})$
and 
\[
\left|\sum_{k\in[K]}\sum_{p\in\pbpki}\psi\left(\left\langle \w_{c}(t),-y\api\v_{k}\right\rangle \right)\right|\leq o(1),
\]
we have $y^{(i)}F(\w(t),\x^{(i)})\leq O(1)$ and therefore $\frac{1}{1+e^{y^{(i)}F(\w(t),\x^{(i)})}}\geq\Omega(1)$.
Then, 
\begin{align}
 & y^{(i)}\left\langle \w_{c}(t+1),\vxi^{(i)}\right\rangle \nonumber \\
= & ~ y^{(i)}\left\langle \w_{c}(t),\vxi^{(i)}\right\rangle +\frac{\eta}{n}\left(\frac{1}{1+e^{y^{(i)}F(\w(t),\x^{(i)})}}\psi'\left(\left\langle \w_{c}(t),\vxi^{(i)}\right\rangle \right)\left\Vert \vxi^{(i)}\right\Vert _{2}^{2}\right)\nonumber \\
 & +\frac{\eta}{n}\sum_{j:j\neq i}\left(\frac{y^{(i)}y^{(j)}}{1+e^{y^{(j)}F(\w(t),\x^{(j)})}}\psi'\left(\left\langle \w_{c}(t),\vxi^{(j)}\right\rangle \right)\left\langle \vxi^{(j)},\vxi^{(i)}\right\rangle \right)\label{eq:xi_phase2eq1}\\
 & +\frac{\eta}{n}\sum_{j=1}^{n}\left(\frac{y^{(i)}}{1+e^{y^{(j)}F(\w_{c}(t),\x^{(j)})}}\psi'\left(\left\langle \w_{c}(t),\v_{k_{j}^{*}}\right\rangle \right)\left\langle \v_{k_{j}^{*}},\vxi^{(i)}\right\rangle \right)\label{eq:xi_phase2_eq2}\\
 & -\frac{\eta}{n}\sum_{j=1}^{n}\left(\frac{y^{(i)}}{1+e^{y^{(j)}F(\w_{c}(t),\x^{(j)})}}\sum_{k\in[K]}\sum_{p\in\pbpkj}\psi'\left(\left\langle \w_{c}(t),\apj\v_{k}\right\rangle \right)\left\langle \apj\v_{k},\vxi^{(i)}\right\rangle \right).\label{eq:xi_phase2_eq3}
\end{align}

If $\frac{1}{1+e^{y^{(i)}F(\w(t),\x^{(i)})}}\geq\Omega(1)$, and $\tom\left(\sigma_{0}\sigma_{\xi}\right)\leq\max_{c\in[C]}y^{(i)}\left\langle \w_{c}(t),\vxi^{(i)}\right\rangle $,
\[
\frac{\eta}{n}\left(\frac{1}{1+e^{y^{(i)}F(\w(t),\x^{(i)})}}\psi'\left(\left\langle \w_{c}(t),\vxi^{(i)}\right\rangle \right)\left\Vert \vxi^{(i)}\right\Vert _{2}^{2}\right)\geq \tom \left(\frac{\eta}{n}\left(\sigma_{0}\sigma_{\xi}\right)^{q-1}\sigma_{\xi}^{2}\right).
\]

When $\Ginit$ holds, since $\frac{1}{1+e^{y^{(i)}F(\w(t),\x^{(i)})}}\leq1$,
$\psi'\left(\left\langle \w_{c}(t),\vxi^{(j)}\right\rangle \right)\leq1$,
and $\psi'\left(\left\langle \w_{c}(t),\v_{k_{j}^{*}}\right\rangle \right)\leq1$,
$\left|\text{\eqref{eq:xi_phase2eq1}}\right|\leq\tO\left(\eta\sigma_{\xi}^{2}d^{-1/2}\right)$
and $\left|\eqref{eq:xi_phase2_eq2}\right|\leq\tO\left(\eta\sigma_{\xi}d^{-1/2}\right)$.

By Lemma $\text{\ref{lem:v_bd_all}}$ and the upper bound on $\alpha$,
\begin{align*}
\left|\eqref{eq:xi_phase2_eq3}\right| & \leq\tO\left(\eta\alpha^{q}P\max_{k\in[K]}\left|\left\langle \w_{c}(t),\v_{k}\right\rangle \right|^{q-1}\sigma_{\xi}d^{-1/2}\right)\\
 & \leq\tO\left(\eta\alpha^{q}P\left(\sigma_{0}+\eta T\left(\rho_{k}+\sigma_{\xi}\right)\right)^{q-1}\sigma_{\xi}d^{-1/2}\right)\\
 & \leq\tO\left(\eta\sigma_{\xi}d^{-1/2}\right).
\end{align*}

When $\Ginit$ holds, $\tom\left(\sigma_{0}\sigma_{\xi}\right)\leq y^{(i)}\max_{c\in[C]}\left\langle \w_{c}(0),\vxi^{(i)}\right\rangle $.
Then, when $n\leq o\left(\min\left\{ \sigma_{0}^{q-1}\sigma_{\xi}^{q}d^{1/2},\sigma_{0}^{q-1}\sigma_{\xi}^{q-1}d^{1/2}\right\} \right)$,
for $t=0$,
\begin{equation}
\max_{c\in[C]}y^{(i)}\left\langle \w_{c}(t+1),\vxi^{(i)}\right\rangle =\max_{c\in[C]}y^{(i)}\left\langle \w_{c}(t),\vxi^{(i)}\right\rangle +\frac{\eta}{n}\tthe\left(\sigma_{\xi}^{2}\psi'\left(\max_{c\in[C]}y^{(i)}\left\langle \w_{c}(t),\vxi^{(i)}\right\rangle \right)\right),\label{eq:xi_phase2_eq4}
\end{equation}
which shows $\max_{c\in[C]}y^{(i)}\left\langle \w_{c}(t),\vxi^{(i)}\right\rangle $
is increasing. Then, $\text{\eqref{eq:xi_phase2_eq4}}$ holds for
all $0\leq t\leq T$.

Starting from $\max_{c\in[C]}y^{(i)}\left\langle \w_{c}(t'),\vxi^{(i)}\right\rangle $,
the number of iterations it takes to reach $\max_{c\in[C]}y^{(i)}\left\langle \w_{c}(t),\vxi^{(i)}\right\rangle \geq2\max_{c\in[C]}y^{(i)}\left\langle \w_{c}(t'),\vxi^{(i)}\right\rangle $
is at most $O\left(\frac{n\max_{c\in[C]}y^{(i)}\left\langle \w_{c}(t'),\vxi^{(i)}\right\rangle }{\eta\sigma_{\xi}^{2}\left(\max_{c\in[C]}y^{(i)}\left\langle \w_{c}(t'),\vxi^{(i)}\right\rangle \right)^{q-1}}\right)$.
Then, starting from $\max_{c\in[C]}y^{(i)}\left\langle \w_{c}(0),\vxi^{(i)}\right\rangle \geq\tom\left(\sigma_{0}\sigma_{\xi}\right)$,
it takes at most 
\[
T'\leq\tO\left(\sum_{i=0}^{\infty}\frac{n2^{i}\sigma_{0}\sigma_{\xi}}{\eta\sigma_{\xi}^{2}(2^{i}\sigma_{0}\sigma_{\xi})^{q-1}}\right)\leq\tilde{O}\left(\frac{n}{\eta\sigma_{\xi}^{2}(\sigma_{0}\sigma_{\xi})^{q-2}}\right)
\]
time steps to reach $\max_{c\in[C]}y^{(i)}\left\langle \w_{c}(T'),\vxi^{(i)}\right\rangle \geq \Omega(C^{-1/q})$. 
\end{proof}

\vkbd*
\begin{proof}
For every $k\in[K]$,
\begin{align}
 & \left\langle \w_{c}(t+1),\v_{k}\right\rangle \nonumber \\
= & \left\langle \w_{c}(t),\v_{k}\right\rangle +\frac{\eta}{n}\sum_{i:k_{i}^{*}=k}\left(\frac{1}{1+e^{y^{(i)}F(\w(t),\x^{(i)})}}\psi'\left(\left\langle \w_{c}(t),\v_{k}\right\rangle \right)\left\Vert \v_{k}\right\Vert _{2}^{2}\right)\nonumber \\
 & -\frac{\eta}{n}\sum_{i=1}^{n}\left(\frac{1}{1+e^{y^{(i)}F(\w(t),\x^{(i)})}}\sum_{p\in\pbpki}\api\psi'\left(\left\langle \w_{c}(t),\api\v_{k}\right\rangle \right)\left\Vert \v_{k}\right\Vert _{2}^{2}\right)\label{eq:vk_bd_eq1}\\
 & +\frac{\eta}{n}\sum_{i=1}^{n}\left(\frac{1}{1+e^{y^{(i)}F(\w(t),\x^{(i)})}}\psi'\left(\left\langle \w_{c}^{(t)},\vxi^{(i)}\right\rangle \right)y^{(i)}\left\langle \vxi^{(i)},\v_{k}\right\rangle \right)\label{eq:vk_bd_eq2}
\end{align}
Then, since $\frac{1}{1+e^{y^{(i)}F(\w(t),\x^{(i)})}}\leq1$ and $\psi'\left(\left\langle \w_{c}(t),\v_{k}\right\rangle \right)\leq O\left(\left|\left\langle \w_{c}(t),\v_{k}\right\rangle \right|^{q-1}\right)$,
\[
\frac{\eta}{n}\sum_{i:k_{i}^{*}=k}\left(\frac{1}{1+e^{y^{(i)}F(\w(t),\x^{(i)})}}\psi'\left(\left\langle \w_{c}(t),\v_{k}\right\rangle \right)\left\Vert \v_{k}\right\Vert _{2}^{2}\right)\leq O\left(\eta\rho_{k}\left|\left\langle \w_{c}(t),\v_{k}\right\rangle \right|^{q-1}\right).
\]
The second term $\text{\eqref{eq:vk_bd_eq1}\ensuremath{\leq0}}.$
For $\eqref{eq:vk_bd_eq2}$, since $\left|\frac{y^{(i)}}{1+e^{y^{(i)}F(\w(t),\x^{(i)})}}\right|\leq1$
and $\psi'\left(\left\langle \w_{c}^{(t)},\vxi^{(i)}\right\rangle \right)\leq1$,
if $\Ginit$ holds, $\eqref{eq:vk_bd_eq2}\leq\tO\left(\sigma_{\xi}d^{-1/2}\right).$

Finally, if $\Ginit$ holds, $\left\langle \w_{c}(0),\v_{k}\right\rangle \leq\tO\left(\sigma_{0}\right)$,
so it takes at least $t\geq\tom\left(\frac{\sigma_{0}}{\eta\rho_{k}\sigma_{0}^{q-1}+\eta\sigma_{\xi}d^{-1/2}}\right)$
time steps to reach $\left\langle \w_{c}(t),\v_{k}\right\rangle \geq2\left\langle \w_{c}(0),\v_{k}\right\rangle $.
\end{proof}

\xiibd*
\begin{proof}
We prove using induction. At $t=0$, when $\Ginit$ holds, $\max_{i\in[n],c\in[C]}y^{(i)}\left\langle \w_{c}(0),\vxi^{(i)}\right\rangle \leq\tO(\sigma_{0}\sigma_{\xi}).$
Assume $\max_{i\in[n],c\in[C]}y^{(i)}\left\langle \w_{c}(t'),\vxi^{(i)}\right\rangle \leq\tO(\sigma_{0}\sigma_{\xi})$
for any $0\leq t'\leq t$ for induction. For any $c\in[C]$,
\begin{align}
 & y^{(i)}\left\langle \w_{c}(t+1),\vxi^{(i)}\right\rangle \nonumber \\
= & ~ y^{(i)}\left\langle \w_{c}(t),\vxi^{(i)}\right\rangle +\frac{\eta}{n}\left(\frac{1}{1+e^{y^{(i)}F(\w(t),\x^{(i)})}}\psi'\left(\left\langle \w_{c}(t),\vxi^{(i)}\right\rangle \right)\left\Vert \vxi^{(i)}\right\Vert _{2}^{2}\right)\nonumber \\
 & +\frac{\eta}{n}\sum_{j:j\neq i}\left(\frac{y^{(i)}y^{(j)}}{1+e^{y^{(j)}F(\w(t),\x^{(j)})}}\psi'\left(\left\langle \w_{c}(t),\vxi^{(j)}\right\rangle \right)\left\langle \vxi^{(j)},\vxi^{(i)}\right\rangle \right)\label{eq:xi_phase2eq1-1}\\
 & +\frac{\eta}{n}\sum_{j=1}^{n}\left(\frac{y^{(i)}}{1+e^{y^{(j)}F(\w_{c}(t),\x^{(j)})}}\psi'\left(\left\langle \w_{c}(t),\v_{k_{j}^{*}}\right\rangle \right)\left\langle \v_{k_{j}^{*}},\vxi^{(i)}\right\rangle \right)\label{eq:xi_phase2_eq2-1}\\
 & -\frac{\eta}{n}\sum_{j=1}^{n}\left(\frac{y^{(i)}}{1+e^{y^{(j)}F(\w_{c}(t),\x^{(j)})}}\sum_{k\in[K]}\sum_{p\in\pbpkj}\psi'\left(\left\langle \w_{c}(t),\apj\v_{k}\right\rangle \right)\left\langle \apj\v_{k},\vxi^{(i)}\right\rangle \right).\label{eq:xi_phase2_eq3-1}
\end{align}
Then, when $\Ginit$ holds, by $\frac{1}{1+e^{y^{(i)}F(\w(t),\x^{(i)})}}\leq1$
and $\psi'(\cdot)\leq1$, 
\begin{align*}
 & ~y^{(i)}\left\langle \w_{c}(t+1),\vxi^{(i)}\right\rangle \\
\leq & ~ y^{(i)}\left\langle \w_{c}(t),\vxi^{(i)}\right\rangle +\eta\tO(n^{-1}\sigma_{0}^{q-1}\sigma_{\xi}^{q+1}+\sigma_{\xi}^{2}d^{-1/2}+\sigma_{\xi}d^{-1/2}+\alpha^{q}P\sigma_{\xi}d^{-1/2}\left(\sigma_{0}+\eta T\left(\rho_{k}+\sigma_{\xi}d^{-1/2}\right)\right)^{q-1})\\
\leq &~ y^{(i)}\left\langle \w_{c}(0),\vxi^{(i)}\right\rangle +\eta t\tO(n^{-1}\sigma_{0}^{q-1}\sigma_{\xi}^{q+1}+\sigma_{\xi}^{2}d^{-1/2}+\sigma_{\xi}d^{-1/2}).
\end{align*}
The last step uses the upper bound on $\alpha$ and the induction
hypothesis. Since $n\leq o\left(\min\left\{ \sigma_{0}^{q-1}\sigma_{\xi}^{q}d^{1/2},\sigma_{0}^{q-1}\sigma_{\xi}^{q-1}d^{1/2}\right\} \right)$
and $t\leq o(n\eta^{-1}\sigma_{\xi}^{-q}\sigma_{0}^{-q+2})$, $\max_{i\in[n],c\in[C]}y^{(i)}\left\langle \w_{c}(t),\vxi^{(i)}\right\rangle \leq\tO(\sigma_{0}\sigma_{\xi}).$

\end{proof}

\wrandxi*
\begin{proof}
For any $0\leq t<T$, 
\begin{align*}
 & \left\langle \w_{c}(t+1),\vxi\right\rangle \\
= & \left\langle \w_{c}(t),\vxi\right\rangle +\frac{\eta}{n}\sum_{i=1}^{n}\frac{y^{(i)}}{1+e^{y^{(i)}F(\w(t),\x^{(i)})}}\psi'\left(\left\langle \w_{c}(t),\vxi^{(i)}\right\rangle \right)\left\langle \vxi^{(i)},\vxi\right\rangle \\
 & +\frac{\eta}{n}\sum_{i=1}^{n}\left(\frac{1}{1+e^{y^{(i)}F(\w(t),\x^{(i)})}}\psi'\left(\left\langle \w_{c}(t),\v_{k_{i}^{*}}\right\rangle \right)\left\langle \v_{k_{i}^{*}},\vxi\right\rangle \right)\\
 & -\frac{\eta}{n}\sum_{i=1}^{n}\left(\frac{1}{1+e^{y^{(i)}F(\w(t),\x^{(i)})}}\sum_{k\in[K]}\sum_{p\in\pbpki}\psi'\left(\left\langle \w_{c}(t),\api\v_{k}\right\rangle \right)\left\langle \api\v_{k},\vxi\right\rangle \right).
\end{align*}
By Lemma \ref{lem:gausscorr}, with probability at least $1-\frac{nK}{\poly d}$,
for all $i\in[n]$, $\left\langle \vxi^{(i)},\vxi\right\rangle \leq\tO(\sigma_{\xi}^{2}d^{-1/2})$
and for all $k\in[K]$, $\left\langle \v_{k},\vxi\right\rangle \leq\tO(\sigma_{\xi}d^{-1/2})$.
Then, by $\frac{1}{1+e^{y^{(i)}F(\w(t),\x^{(i)})}}\leq1$ , $\psi'\left(\left\langle \w_{c}(t),\vxi^{(i)}\right\rangle \right)\leq1$,
$\psi'\left(\left\langle \w_{c}(t),\v_{k_{i}^{*}}\right\rangle \right)\leq1$
and $\psi'\left(\left\langle \w_{c}(t),\api\v_{k}\right\rangle \right)\leq1$,
\begin{align*}
 & \left|\left\langle \w_{c}(t+1),\vxi\right\rangle -\left\langle \w_{c}(t),\vxi\right\rangle \right|\\
\leq &~ \tO(\eta\sigma_{\xi}^{2}d^{-1/2})+\tO(\eta\sigma_{\xi}d^{-1/2})+\tO\left(\eta\alpha^{q}P\sigma_{\xi}d^{-1/2}\left(\sigma_{0}+\eta T\left(\max_{k'}\rho_{k'}+\sigma_{\xi}d^{-1/2}\right)\right)^{q-1}\right)\\
\leq &~ \eta\tO(\sigma_{\xi}^{2}d^{-1/2}+\sigma_{\xi}d^{-1/2}).
\end{align*}
The second step uses Lemma \ref{lem:v_bd_all}. The third step uses
the upper bound on $\alpha$. Summing over $0\leq t'\leq t$, 
\[
\left|\left\langle \w_{c}(t),\vxi\right\rangle -\left\langle \w_{c}(0),\vxi\right\rangle \right|\leq\eta T\tO(\sigma_{\xi}^{2}d^{-1/2}+\sigma_{\xi}d^{-1/2}).
\]
When $n\leq o\left(\min\left\{ \sigma_{0}^{q-1}\sigma_{\xi}^{q}d^{1/2},\sigma_{0}^{q-1}\sigma_{\xi}^{q-1}d^{1/2}\right\} \right)$,
$K\leq o\left(\min\left\{ \sigma_{0}^{q-1}\sigma_{\xi}^{-1}d^{1/2},\sigma_{0}^{q-1}d^{1/2}\right\} \right)$,
and\\
 $T\leq\tO\left(\max\left\{ n\eta^{-1}\sigma_{\xi}^{-q}\sigma_{0}^{-q+2},K\eta^{-1}\sigma_{0}^{-q+2}\right\} \right)$,
\[
\left|\left\langle \w_{c}(t),\vxi\right\rangle -\left\langle \w_{c}(0),\vxi\right\rangle \right|\leq o(\sigma_{0}\sigma_{\xi}).
\]
\end{proof}

\end{document}